\documentclass{article}

\usepackage[utf8]{inputenc}
\usepackage[T1]{fontenc}
\usepackage{microtype}

\usepackage{amsmath}
\usepackage{amssymb}
\usepackage{amsfonts}
\usepackage{mathtools}
\usepackage{amsthm}
\usepackage{nicefrac}

\usepackage{natbib}

\usepackage{hyperref}
\usepackage[all]{hypcap}
\usepackage{xcolor}
\usepackage{colortbl}
\usepackage{rotating}
\definecolor{teal}{HTML}{008080}
\hypersetup{
  colorlinks,
  citecolor=teal,
  linkcolor=red!80!black,
  urlcolor=blue
}

\usepackage[capitalize,noabbrev,nameinlink]{cleveref}

\usepackage{algorithm}
\usepackage{algpseudocode}

\usepackage{graphicx}
\usepackage{subcaption}
\usepackage{wrapfig}
\usepackage{float}
\usepackage{pdfpages}
\usepackage{booktabs}
\usepackage{pdflscape}

\usepackage{verbatim}
\usepackage{comment}
\usepackage{xspace}

\usepackage[acronym,nowarn]{glossaries}
\glsdisablehyper
\makeglossaries

\usepackage{tikz,siunitx}
\usetikzlibrary{shapes.geometric,shapes.symbols,arrows.meta,positioning,decorations.pathreplacing,calc,fit,backgrounds,patterns}

\usepackage{tcolorbox}
\definecolor{mylightgray}{gray}{0.95}
\definecolor{boxcolor}{HTML}{faf9f5}
\newtcolorbox{mybox}{colback=mylightgray,colframe=mylightgray,top=0.8pt,bottom=0.8pt,right=1.8pt,left=1.8pt}

\usepackage[textsize=tiny]{todonotes}

\usepackage{needspace}

\newacronym{AU}{AU}{Aleatoric Uncertainty}
\newacronym{AUC}{AUC}{area under the curve}
\newacronym{BDL}{BDL}{Bayesian Deep Learning}
\newacronym[firstplural=Bayesian neural networks]{BNN}{BNN}{Bayesian neural network}
\newacronym{CNN}{CNN}{convolutional neural network}
\newacronym{DE}{DE}{Deep Ensembles}
\newacronym[firstplural=Deep Gaussian Processes]{DGP}{DGP}{Deep Gaussian Process}
\newacronym[firstplural=deep neural networks]{DNN}{DNN}{deep neural network}
\newacronym{ECE}{ECE}{expected calibration error}
\newacronym{ELBO}{ELBO}{evidence lower bound}
\newacronym{ELL}{ELL}{expected log likelihood}
\newacronym{ESS}{ESS}{Effective Sample Size}
\newacronym{EU}{EU}{Epistemic Uncertainty}
\newacronym{GLM}{GLM}{Generalized linear model}
\newacronym{GP}{GP}{Gaussian process}
\newacronym{GGN}{GGN}{Generalized Gauss-Newton}
\newacronym{HMC}{HMC}{Hamiltonian Monte Carlo}
\newacronym{ICL}{ICL}{In-Context Learning}
\newacronym{IVON}{IVON}{Improved Variational Online Newton}
\newacronym{KL}{KL}{Kullback-Leibler divergence}
\newacronym{LLA}{LLA}{Linearized Laplace Approximation}
\newacronym{LLC}{LLC}{Local Learning Coefficient}
\newacronym{LL}{LL}{log-likelihood}
\newacronym{LLM}{LLM}{Large Language Model}
\newacronym{MAP}{MAP}{Maximum a posteriori}
\newacronym{MC}{MC}{Monte Carlo}
\newacronym{MCD}{MCD}{Monte Carlo Dropout}
\newacronym{MCMC}{MCMC}{Markov chain Monte Carlo}
\newacronym{MI}{MI}{Mutual Information}
\newacronym{MLP}{MLP}{multilayer perceptron}
\newacronym{MNLL}{MNLL}{mean negative loglikelihood}
\newacronym{NLL}{NLL}{negative loglikelihood}
\newacronym{NMLL}{NMLL}{Negative Marginal Log Likelihood}
\newacronym{NN}{NN}{neural network}
\newacronym{NTK}{NTK}{Neural Tangent Kernel}
\newacronym{OOD}{OOD}{out-of-distribution}
\newacronym{PSD}{PSD}{Positive Semi-Definite}
\newacronym{RLCT}{RLCT}{Real Log Canonical Threshold}
\newacronym{RELU}{ReLU}{rectified linear unit}
\newacronym{RMSE}{RMSE}{root mean square error}
\newacronym{SAM}{SAM}{Sharpness Aware Minimization}
\newacronym{SGLD}{SGLD}{Stochastic Gradient Langevin Dynamics}
\newacronym{SGHMC}{SGHMC}{Stochastic Gradient Hamiltonian Monte Carlo}
\newacronym{SGD}{SGD}{Stochastic Gradient Descent}
\newacronym{SLT}{SLT}{Singular Learning Theory}
\newacronym{TU}{TU}{Total Uncertainty}
\newacronym{UQ}{UQ}{Uncertainty Quantification}
\newacronym{VI}{VI}{Variational Inference}
\newacronym{WD}{WD}{Wasserstein distance}
\newacronym{O-information}{$\Omega$-info}{O-information }

\usepackage{math_commands}

\newcommand{\model}{\textsc{Alice}\xspace}

\DeclareMathOperator{\MI}{I}                 %
\newcommand{\Ex}{\mathbb{E}}                 %
\newcommand{\cU}{\mathcal{U}}                %
\DeclareMathAlphabet{\mathbbold}{U}{bbold}{m}{n}
\newcommand{\indone}{\mathbbold{1}}
\newcommand{\indzero}{\mathbbold{0}}
\newcommand{\Norm}[1]{\left\lVert#1\right\rVert}
\newcommand{\vel}{v}                         %
\newcommand{\ctx}{C}                         %
\newcommand{\netpar}{\theta}                 %
\newcommand{\vmodel}{\vel_{\netpar}}         %
\newcommand{\headout}{m_{\netpar}}           %

\newcommand{\Normald}[2]{\mathcal{N}\!\left(#1,\,#2\right)}

\theoremstyle{plain}
\newtheorem{theorem}{Theorem}
\newtheorem{lemma}{Lemma}
\newtheorem{proposition}{Proposition}

\theoremstyle{definition}

\theoremstyle{remark}
\newtheorem{remark}{Remark}

\usepackage{techreport}

\title{\textsc{Alice}: In-context, Zero-shot, Mutual Information Estimation}

\author[*]{Giulio Franzese}
\author[*]{Simone Rossi}
\author[*]{Pietro Michiardi}

\affiliation[]{EURECOM, Sophia Antipolis, France}

\contribution[*]{Equal contribution}

\abstract{Estimating mutual information (MI) from samples is a central objective in a variety of scientific fields. Modern neural estimators are accurate in the large-data regime, but they fall short when data is scarce, and each must be fit anew for every distribution under study. Current estimators are moreover tied to specific data types. These constraints limit their adoption in many applications where per-distribution training is impractical and sample sizes are small.
We present \model{}, a foundation model that removes per-distribution training, while achieving competitive estimation accuracy. Trained exclusively on a broad family of synthetic distributions, \model{} acts as an in-context estimator of rectified-flow velocity fields: conditioned on samples of an unseen distribution, it estimates that distribution's velocity field without any explicit training. MI is then obtained through a fixed identity that integrates the squared difference between the joint and conditional fields.
We validate \model{} on a standard, challenging benchmark and apply it in three domains, biology, genetics, and neuroscience, whose data the model has never seen.
For the first time, we show that a single model closes the gap with neural estimators trained separately for each distribution, while natively supporting different data dimensionality and sample cardinality, enabling zero-shot \acrshort{MI} analysis across scientific domains.
}

\date{\today}
\correspondence{\{giulio.franzese, simone.rossi, pietro.michiardi\} at eurecom.fr}
\code{\raisebox{-0.15em}{\includegraphics[height=1em]{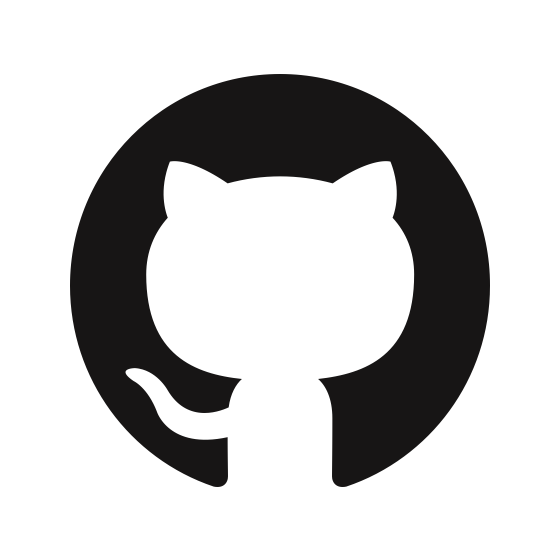}}~\href{https://github.com/eurecom-probai/alice}{\texttt{github.com/eurecom-probai/alice}}}
\metadata[Models]{\raisebox{-0.2em}{\includegraphics[height=1.1em]{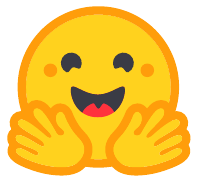}}~\href{https://huggingface.co/collections/eurecom-probai/alice-foundation-models-for-information-estimation}{\texttt{hf.co/collections/eurecom-probai/alice-foundation-models-for-information-estimation}}}

\hypersetup{
  pdftitle={ALICE: In-context, Zero-shot, Mutual Information Estimation},
  pdfauthor={Giulio Franzese, Simone Rossi, Pietro Michiardi}
}

\begin{document}
\maketitle

\newif\ifarchdetail
\archdetailfalse

\section{Introduction}
\label{sec:introduction}

\gls{MI} quantifies the non-linear statistical dependence between two random variables \citep{shannon1948mathematical, mackay2003information} and is widely used in machine learning~\citep{stratos2018mutual, belghazi2018mine, oord2019representation, hjelm2019learning}, in biology~\citep{nurse2008life,tostevin2009mutual,waltermann2011information,brennan2012information} and neuroscience~\citep{borst1999information,ince2017statistical,nieh2021geometry}, to name a few.
For random variables $X\in\bbR^{d_x}$ and $Y\in\bbR^{d_y}$, we write $Z=(X,Y)\in\bbR^d$ with $d=d_x+d_y$, and denote their joint distribution and marginals by $p_{XY}$, $p_X$, and $p_Y$.
Their mutual information is the KL divergence
\begin{equation}\label{eq:mi_def}
  \MI(X;Y)=\KL{p_{XY}}{p_X\otimes p_Y},
\end{equation}
where $p_X\otimes p_Y$ is the product of the marginals, with density $p_X(x)\,p_Y(y)$.
Estimating \gls{MI} from 
finite samples 
is a difficult problem: the estimand depends on the full joint density, \gls{MI} is unbounded and dominated by rare high-density events, and guarantees are fragile in high dimension \citep{paninski2003estimation,poole2019variational,mcallester2020formal,czyz2023beyond}.

Existing sample-based estimators share a structural limitation: each one is fit anew for every distribution.
Variational bounds such as MINE \citep{belghazi2018mine}, InfoNCE \citep{oord2019representation}, NWJ \citep{nguyen2010nwj}, and SMILE \citep{song2020smile}, together with diffusion based estimators such as MINDE \citep{franzese2024minde}, InfoBridge \citep{kholkin2026infobridge}, FMMI \citep{butakov2026fmmi}, require training from scratch for each distribution, with associated computational and tuning costs and a risk of failure.
InfoAtlas \citep{hu2026infoatlas} is an amortized alternative which uses a hypernetwork to sidestep per-distribution training, but it suffers from a non-negligible penalty in terms of accuracy.
\Cref{sec:related} discusses these estimators and other related work.

The diffusion-based estimators above build on score-based and flow-matching generative models, which represent a distribution by a time-indexed field attached to a noising process that maps clean samples to Gaussian noise \citep{song2021score,lipman2023flow}.
In this work, we use the rectified-flow velocity as this field.
Let $p$ be a density on $\bbR^d$, let $Z_0\sim p$ be a clean sample, let $\epsilon\sim\Normald{0}{I}$ be standard Gaussian noise independent of $Z_0$, and let $t\in[0,1]$.
The rectified-flow interpolant $Z_t=(1-t)Z_0+t\epsilon$ connects data at $t=0$ to noise at $t=1$.
The per-sample flow-matching target is the direction $Z_0-\epsilon$, and the associated velocity field is its conditional mean at a noised point,
\begin{equation}
  \vel_t(z)=\Ex_p\!\left[Z_0-\epsilon\mid Z_t=z\right].
\end{equation}

The key link to \gls{MI} is that, along a common rectified-flow path, the KL divergence between two distributions is a time integral of squared velocity differences \citep{guo2005immse,franzese2024minde,wang2026relative}.
For \Cref{eq:mi_def}, an equivalent form of this identity compares the joint velocity with the two block-conditional velocities, obtained by noising one block while holding the other clean, so \gls{MI} estimation reduces to evaluating three velocity fields.

This formulation suggests an amortized estimator.
We propose \model{}, a single Transformer network~\citep{vaswani2017attention}, trained once on synthetic distributions to predict their rectified-flow velocity fields from samples, in the spirit of amortized in-context predictors for tabular data \citep{hollmann2023tabpfn} and function classes \citep{garg2022what}.
At inference, a finite context of samples from an unseen joint distribution determines the field represented by the model, and a query specifies the point at which it is evaluated.
Three masked queries provide the fields required 
to estimate \gls{MI} with a fixed set of forward passes.
The training corpus is entirely synthetic: a family of parametric distributions that is simple to define, cheap to sample, and easy to extend.
\model{} generalizes to distributions and data types that the corpus does not contain (\Cref{sec:applications}).

The absence of per-distribution training is key in the low-data regime: existing neural estimators achieve high accuracy only with hundreds of thousands of training samples per distribution and degrade sharply below that (\Cref{sec:exp:czyz}), while datasets in biology or neuroscience, for example, often provide a few thousand pairs at most.
In contrast, \model{} covers context sizes from only a few hundred samples to tens of thousands and is competitive across the whole range.

Our contributions are as follows.
We present \textbf{\model{}} (\Cref{sec:method}), a foundation model for in-context estimation of velocity fields that can be used at any joint width and context length.
\model{} is the first zero-shot \gls{MI} estimator whose accuracy matches that of estimators trained per distribution.
We \textbf{validate} \model{} (\Cref{sec:experiments}) on the ``Beyond Normal'' benchmark \citep{czyz2023beyond}.
In the zero-shot setting, \model{} is competitive with trained neural estimators at their full budget and, with one thousand samples, is the most accurate estimator by a factor of at least two.
We also report three scientific \textbf{applications} (\Cref{sec:applications}) on data absent from the training corpus.
\looseness=-1
In these applications, \model{} reproduces findings obtained with dedicated estimators and extends them, since the cost of a few forward passes per estimate allows analyses that per-distribution training makes impractical.

\section{\texorpdfstring{\model}{ALICE}}
\label{sec:method}
\model{} is a \gls{MI} estimator that amortizes velocity-field estimation across joint distributions.
It is trained once, exclusively on a synthetic corpus of joint distributions,  with a masked flow-matching objective.
At inference, it conditions on samples from an unseen joint distribution and evaluates the joint and block-conditional velocity fields under three noising patterns; a fixed velocity identity combines their aligned block-wise differences into the estimate.
\Cref{fig:schematic} summarizes
these ideas.

This section develops the construction in four steps.
We first derive the velocity-form identity and its Monte Carlo estimator in \Cref{sec:mi-from-kl}.
We then define the context-conditioned velocity field in \Cref{sec:method:posterior}, describe the size-independent architecture in \Cref{sec:method:arch}, and explain the synthetic training corpus and masked pretraining objective in \Cref{sec:method:pretraining}.

\usetikzlibrary{svg.path}
\definecolor{rgRed}{HTML}{ED1C24}
\definecolor{rgRedFill}{HTML}{EF5B61}
\definecolor{rgOrange}{HTML}{FF7F00}
\definecolor{rgOrangeFill}{HTML}{FFA64D}
\definecolor{rgGreen}{HTML}{4DAF4A}
\definecolor{rgGreenFill}{HTML}{86C982}
\definecolor{rgBlue}{HTML}{377EB8}
\definecolor{rgBlueFill}{HTML}{73A4CC}
\definecolor{rgPurple}{HTML}{984EA3}
\definecolor{rgPurpleFill}{HTML}{B17CBB}
\definecolor{rgGray}{HTML}{7F7F7F}
\def\archvTwoCorrelatedGaussian{0.901/4.287,1.406/4.476,1.281/4.352,1.2/4.388,1.439/4.509,1.247/4.385,1.08/4.367,1.517/4.414,1.579/4.591,1.329/4.436,1.181/4.334,1.378/4.466,1.142/4.353,1.474/4.594,1.419/4.373,1.436/4.505,1.558/4.574,0.949/4.434,1.013/4.413,1.347/4.434,1.069/4.354,1.213/4.517,1.343/4.537,1.018/4.32,1.446/4.563,1.397/4.459,1.04/4.462,1.113/4.358,1.575/4.555,1.124/4.436,1.211/4.496,1.228/4.373,1.284/4.436,0.978/4.295,0.991/4.388,1.501/4.503,1.15/4.417,1.249/4.462,1.353/4.493,1.382/4.419}
\def\archvTwoGaussianMixture{2.343/4.31/rgOrange,2.287/4.384/rgOrange,2.311/4.388/rgOrange,2.299/4.319/rgOrange,2.229/4.361/rgOrange,2.31/4.404/rgOrange,2.235/4.373/rgOrange,2.211/4.39/rgOrange,2.315/4.425/rgOrange,2.262/4.349/rgOrange,2.199/4.425/rgOrange,2.179/4.38/rgOrange,2.097/4.432/rgOrange,2.278/4.367/rgOrange,2.254/4.42/rgOrange,2.264/4.403/rgOrange,2.273/4.343/rgOrange,2.116/4.436/rgOrange,2.223/4.443/rgOrange,2.377/4.374/rgOrange,2.56/4.445/rgOrange,2.768/4.47/rgOrange,2.596/4.507/rgOrange,2.535/4.444/rgOrange,2.567/4.467/rgOrange,2.503/4.442/rgOrange,2.632/4.426/rgOrange,2.526/4.487/rgOrange,2.527/4.46/rgOrange,2.69/4.473/rgOrange,2.784/4.465/rgOrange,2.482/4.485/rgOrange,2.639/4.502/rgOrange,2.49/4.464/rgOrange,2.719/4.479/rgOrange,2.722/4.44/rgOrange,2.662/4.47/rgOrange,2.761/4.429/rgOrange,2.729/4.479/rgOrange,2.626/4.442/rgOrange}
\def\archvTwoBanana{1.352/3.735,1.201/3.678,0.948/3.83,1.276/3.698,1.092/3.725,1.469/3.792,1.575/3.861,1.421/3.746,1.14/3.72,0.931/3.823,1.408/3.735,1.137/3.735,1.276/3.702,1.102/3.755,0.993/3.802,1.12/3.753,1.045/3.757,1.581/3.879,1.446/3.768,1.206/3.715,1.269/3.684,1.428/3.763,1.013/3.776,1.274/3.697,1.125/3.736,1.562/3.815,1.186/3.698,1.395/3.725,1.44/3.742,1.427/3.755,1.421/3.75,1.074/3.734,1.079/3.746,0.992/3.805,1.245/3.666,1.167/3.727,1.025/3.785,1.14/3.716,1.43/3.73,1.417/3.785}
\def\archvTwoSpiral{2.478/3.819/rgBlue,2.503/3.809/rgBlue,2.489/3.822/rgBlue,2.488/3.813/rgBlue,2.45/3.822/rgBlue,2.449/3.83/rgBlue,2.433/3.85/rgBlue,2.413/3.829/rgBlue,2.383/3.833/rgBlue,2.34/3.831/rgBlue,2.352/3.818/rgBlue,2.356/3.802/rgBlue,2.356/3.789/rgBlue,2.328/3.771/rgBlue,2.355/3.762/rgBlue,2.403/3.734/rgBlue,2.408/3.729/rgBlue,2.446/3.733/rgBlue,2.505/3.714/rgBlue,2.524/3.738/rgBlue,2.549/3.743/rgBlue,2.601/3.762/rgBlue,2.627/3.754/rgBlue,2.622/3.792/rgBlue,2.638/3.812/rgBlue,2.651/3.831/rgBlue,2.615/3.869/rgBlue,2.553/3.884/rgBlue,2.532/3.894/rgBlue,2.513/3.905/rgBlue,2.435/3.907/rgBlue,2.351/3.91/rgBlue,2.323/3.897/rgBlue,2.266/3.878/rgBlue,2.232/3.856/rgBlue,2.193/3.83/rgBlue,2.186/3.793/rgBlue,2.18/3.759/rgBlue,2.205/3.734/rgBlue,2.251/3.702/rgBlue}

\begin{figure}[t]
  \centering
  \resizebox{0.92\linewidth}{!}{%
    \begin{tikzpicture}[
        x=1cm,
        y=1cm,
        font=\sffamily\scriptsize,
        text=black!90,
        >={Latex[length=1.5mm,width=1.1mm]},
        panel/.style={draw=rgGray, line width=0.9pt, rounded corners=7pt, fill=white},
        module/.style={draw=rgGray, line width=0.8pt, rounded corners=4pt, fill=white,
        align=center, inner sep=2pt, font=\sffamily\tiny},
        model/.style={module, draw=rgPurple, fill=rgPurpleFill!28, minimum height=1.04cm},
        integration/.style={draw=rgBlue!65, line width=0.75pt, rounded corners=6pt,
        fill=rgBlueFill!6},
        flowline/.style={line width=0.8pt, black!80},
        flow/.style={flowline, ->, shorten >=1.8pt, shorten <=1.2pt},
        maskline/.style={line width=1.0pt, black!80, rounded corners=2pt},
        maskflow/.style={maskline, ->, shorten >=1.8pt, shorten <=1.2pt},
        wire/.style={line width=0.55pt, black!65},
        lab/.style={inner sep=1pt, align=center, font=\sffamily\tiny},
      ]
      \node[panel, minimum width=12.42cm, minimum height=3.cm] at (6.45,4.20) {};
      \node[panel, minimum width=12.42cm, minimum height=3.4cm] at (6.45,0.85) {};

      \begin{scope}[shift={(0cm,0.2cm)}]

        \node[font=\sffamily\small\bfseries, anchor=west] at (0.72,5.16) {(a) Pretraining};
        \begin{scope}[shift={(0cm,-0.2cm)}]
          \node[module, minimum width=2.68cm, minimum height=1.50cm] (corpus) at (1.88,4.13) {};
          \node[module, minimum width=0.94cm, minimum height=0.54cm] (distA) at (1.25,4.43) {};
          \node[module, minimum width=0.94cm, minimum height=0.54cm] (distB) at (2.45,4.43) {};
          \node[module, minimum width=0.94cm, minimum height=0.54cm] (distC) at (1.25,3.80) {};
          \node[draw=rgBlue, fill=rgBlueFill!25, line width=1.3pt, rounded corners=4pt,
          minimum width=1.02cm, minimum height=0.62cm] (selected) at (2.45,3.80) {};
          \foreach \xx/\yy in \archvTwoCorrelatedGaussian {
            \fill[rgRed] (\xx,\yy) circle (0.8pt);
          }
          \foreach \xx/\yy/\cc in \archvTwoGaussianMixture {
            \fill[\cc] (\xx,\yy) circle (0.8pt);
          }
          \foreach \xx/\yy in \archvTwoBanana {
            \fill[rgGreen] (\xx,\yy) circle (0.8pt);
          }
          \foreach \xx/\yy/\cc in \archvTwoSpiral {
            \fill[\cc] (\xx,\yy) circle (0.8pt);
          }
          \node[lab] at (1.88,5.04) {corpus $\mathcal T$};
        \end{scope}

        \begin{scope}[shift={(0.4cm,0cm)}]
          \node[module, minimum width=1.36cm, minimum height=0.58cm] (ca) at (4.42,4.62) {};
          \foreach \xx in {3.88,4.10,4.32,4.54,4.76} {
            \fill[rgBlueFill, draw=rgBlue, line width=0.4pt]
            (\xx,4.52) rectangle ({\xx+0.20},4.72);
          }
          \node[lab] at (4.42,4.21) {clean context $C$};
        \end{scope}

        \begin{scope}[shift={(0.4cm,-0.2cm)}]
          \node[module, minimum width=1.36cm, minimum height=0.82cm] (qa) at (4.42,3.64) {};
          \foreach \xx/\op in {3.88/0.38,4.10/0.62,4.32/0.46,4.54/0.76,4.76/0.54} {
            \fill[rgBlueFill, fill opacity=\op, draw=rgBlue, line width=0.4pt]
            (\xx,3.74) rectangle ({\xx+0.20},3.94);
          }
          \foreach \xx in {3.88,4.10,4.32,4.54,4.76} {
            \fill[rgPurpleFill, draw=rgPurple, line width=0.4pt]
            (\xx,3.44) rectangle ({\xx+0.14},3.64);
            \fill[rgOrangeFill, draw=rgOrange, line width=0.4pt]
            ({\xx+0.14},3.44) rectangle ({\xx+0.20},3.64);
            \draw[wire] (\xx,3.44) rectangle ({\xx+0.20},3.64);
            \draw[wire] ({\xx+0.14},3.44) -- ({\xx+0.14},3.64);
          }
          \node[lab, fill=white] at (4.42,3.07) {queries at $t$; $\phi(t),m$};
        \end{scope}
        \draw[flow, rounded corners=2pt] (selected.east) -- (3.52,3.60) |- (ca.west);
        \draw[flow, rounded corners=2pt] (selected.east) -- (3.4,3.60) |- (qa.west);
        \node[lab, fill=white] at (3.6,4.08) {draw $p$};

        \begin{scope}[shift={(0.2cm,0cm)}]
          \node[model, draw=rgOrange, fill=rgOrangeFill!22, minimum width=1.46cm] (trainmodel) at (7.00,4.18)
          {ALICE};
          \node[anchor=south east, inner sep=1pt] at ([xshift=-2pt,yshift=2pt]trainmodel.south east)
          {\includegraphics[width=0.28cm]{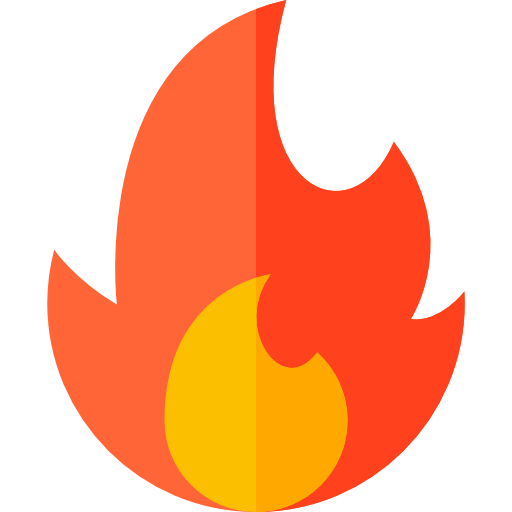}};
        \end{scope}
        \draw[flow, rounded corners=2pt] (ca.east) --+ (0.3, 0)  |- (trainmodel.west);
        \draw[flow, rounded corners=2pt] (qa.east) --+ (0.2, 0) |- (trainmodel.west);

        \begin{scope}[shift={(0.3cm,0cm)}]
          \node[module, minimum width=1.36cm, minimum height=0.58cm] (target) at (8.95,4.18) {};
          \foreach \xx/\shade in {8.41/black!18,8.63/black!44,8.85/black!28,
          9.07/black!62,9.29/black!36} {
            \fill[\shade, draw=black!70, line width=0.4pt]
            (\xx,4.08) rectangle ({\xx+0.20},4.28);
          }
          \node[lab] at (8.95,4.65) {velocity prediction};
        \end{scope}
        \draw[flow] (trainmodel.east) -- (target.west);

        \begin{scope}[shift={(0.3cm,0cm)}]
          \node[module, minimum width=0.78cm, minimum height=0.48cm] (loss) at (10.70,4.18)
          {$\mathcal L(\theta)$};
          \node[lab, anchor=west, fill=white] at (10.8,4.60) {velocity loss};
          \draw[flow, rounded corners=3pt]
          (target.east) -- (loss.west);
        \end{scope}
        \draw[flow, rounded corners=3pt]
        (loss.north) --+ (0,0.65) -| (trainmodel.north) ;

      \end{scope}

      \begin{scope}[shift={(0cm,-0.3cm)}]
        \node[integration, minimum width=6.32cm, minimum height=2.78cm]
        (integrationbox) at (7.34,1.16) {};
        \node[font=\sffamily\small\bfseries, anchor=west] at (0.90,2.55) {(b) Estimation};
        \begin{scope}[shift={(0.15cm,0cm)}]
          \node[module, minimum width=1.30cm, minimum height=0.78cm] (unseen) at (1.20,1.56) {};
          \foreach \xx/\yy/\cc in {0.76/1.78/rgRed,0.93/1.68/rgOrange,1.11/1.80/rgGreen,
            1.29/1.61/rgBlue,1.47/1.76/rgPurple,0.84/1.49/rgOrange,
          1.03/1.42/rgRed,1.22/1.50/rgGreen,1.42/1.39/rgBlue} {
            \fill[\cc] (\xx,\yy) circle (1.0pt);
          }
          \node[lab, text width=1.80cm] at (1.20,0.77) {samples of an unseen $p_{XY}$};
        \end{scope}

        \begin{scope}[shift={(0.15cm,0cm)}]
          \node[module, minimum width=1.20cm, minimum height=0.58cm] (cb) at (3.10,2.02) {};
          \foreach \xx in {2.56,2.78,3.00,3.22,3.44} {
            \fill[rgRedFill, draw=rgRed, line width=0.4pt]
            (\xx,1.92) rectangle ({\xx+0.20},2.12);
          }
          \node[lab] at (3.10,1.61) {context $C$};
        \end{scope}

        \begin{scope}[shift={(0.15cm,0cm)}]
          \node[module, minimum width=1.00cm, minimum height=0.50cm] (qb) at (3.10,1.08) {};
          \foreach \xx in {2.79,3.21} {
            \fill[rgRedFill, draw=rgRed, line width=0.4pt]
            (\xx,0.98) rectangle ({\xx+0.20},1.18);
          }
          \node[lab] at (2.89,0.70) {$x_0$};
          \node[lab] at (3.31,0.70) {$y_0$};
        \end{scope}
        \draw[flow, rounded corners=3pt] (unseen.east) -- (2.25,1.56) |- (cb.west);
        \draw[flow, rounded corners=3pt] (unseen.east) -- (2.25,1.56) |- (qb.west);

        \begin{scope}[shift={(.2cm,0cm)}]
          \begin{scope}[shift={(0.4cm,0cm)}]
            \node[draw=rgGray, line width=0.8pt, rounded corners=4pt, fill=white,
            minimum width=3.0cm, minimum height=1.90cm] (maskbox) at (5.45,1.36) {};
            \node[lab, fill=white] at (5.45,2.38) {noise \& mask};
          \end{scope}

          \begin{scope}[shift={(0.4cm,0cm)}]
            \node[module, minimum width=1.00cm, minimum height=0.50cm] (qfull) at (5.45,1.96) {};
            \fill[rgBlueFill, pattern=north east lines, pattern color=rgBlue,
            draw=rgBlue, line width=0.4pt]
            ([xshift=-3.1mm,yshift=-1mm]qfull.center)
            rectangle ([xshift=-1.1mm,yshift=1mm]qfull.center);
            \fill[rgOrangeFill, pattern=north east lines, pattern color=rgOrange,
            draw=rgOrange, line width=0.4pt]
            ([xshift=1.1mm,yshift=-1mm]qfull.center)
            rectangle ([xshift=3.1mm,yshift=1mm]qfull.center);
            \node[lab, anchor=east, fill=white] at ([yshift=1.25mm]qfull.west) {$m_{XY}$};
          \end{scope}

          \begin{scope}[shift={(0.4cm,0cm)}]
            \node[module, minimum width=1.00cm, minimum height=0.50cm] (qx) at (5.45,1.36) {};
            \fill[rgBlueFill, pattern=north east lines, pattern color=rgBlue,
            draw=rgBlue, line width=0.4pt]
            ([xshift=-3.1mm,yshift=-1mm]qx.center)
            rectangle ([xshift=-1.1mm,yshift=1mm]qx.center);
            \fill[rgOrangeFill, draw=rgOrange, line width=0.4pt]
            ([xshift=1.1mm,yshift=-1mm]qx.center)
            rectangle ([xshift=3.1mm,yshift=1mm]qx.center);
            \node[lab, anchor=east, fill=white]
            at ([yshift=1.25mm]qx.west) {$m_X$};
          \end{scope}

          \begin{scope}[shift={(0.4cm,0cm)}]
            \node[module, minimum width=1.00cm, minimum height=0.50cm] (qy) at (5.45,0.76) {};
            \fill[rgBlueFill, draw=rgBlue, line width=0.4pt]
            ([xshift=-3.1mm,yshift=-1mm]qy.center)
            rectangle ([xshift=-1.1mm,yshift=1mm]qy.center);
            \fill[rgOrangeFill, pattern=north east lines, pattern color=rgOrange,
            draw=rgOrange, line width=0.4pt]
            ([xshift=1.1mm,yshift=-1mm]qy.center)
            rectangle ([xshift=3.1mm,yshift=1mm]qy.center);
            \node[lab, anchor=east, fill=white]
            at ([yshift=1.25mm]qy.west) {$m_Y$};
          \end{scope}

          \begin{scope}[shift={(0cm,0cm)}]
            \draw[maskline] (qb.east) -- ([yshift=-0.28cm]maskbox.west);
            \draw[maskflow] ([yshift=-0.28cm]maskbox.west) --+(0.3,0) |- (qfull.west);
            \draw[maskflow] ([yshift=-0.28cm]maskbox.west) --+(0.3,0) |- (qx.west);
            \draw[maskflow] ([yshift=-0.28cm]maskbox.west) --+(0.3,0) |- (qy.west);
            \draw[maskflow] ([xshift=-0.6cm, yshift=0.28cm]maskbox.west) -- ([yshift=0.28cm]maskbox.west)
            node[midway, left, fill=white] {$t$};
          \end{scope}

          \begin{scope}[shift={(0cm,0cm)}]
            \draw[maskline] (qfull.east) --+(0.2,0) |- ([yshift=+0.cm]maskbox.east);
            \draw[maskline] (qx.east) -- (maskbox.east);
            \draw[maskline] (qy.east) --+ (0.2, 0) |- ([yshift=-0.cm]maskbox.east);
          \end{scope}

        \end{scope}

        \begin{scope}[shift={(0.65cm,-.2cm)}]
          \node[model, draw=rgBlue, fill=rgBlueFill!22, minimum width=1.30cm, minimum height=1.10cm] (frozenmodel) at (8.00,1.56)
          {ALICE};
          \node[anchor=south east, inner sep=1pt] at ([xshift=-2pt,yshift=2pt]frozenmodel.south east)
          {\includegraphics[width=0.27cm]{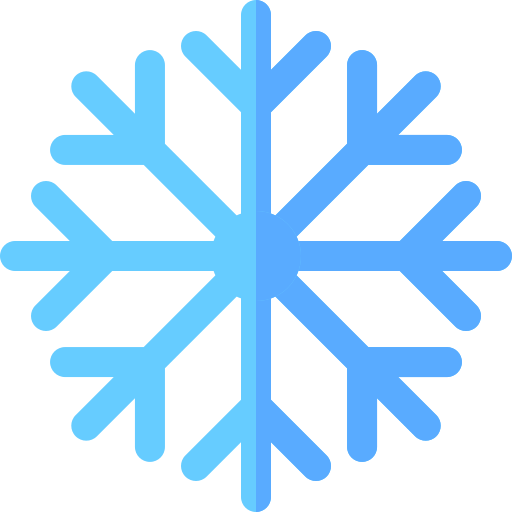}};
        \end{scope}
        \draw[flow, rounded corners=3pt]
        (cb.east) -- (4.08,2.02) -- (4.08,2.62) -- (8.00,2.62) -| (frozenmodel.north);
        \begin{scope}[shift={(0cm,0cm)}]
          \draw[maskflow] ([xshift=-0.3cm]maskbox.east) -- (frozenmodel.west);
        \end{scope}

        \begin{scope}[shift={(0.7cm,-0.2cm)}]
          \node[module, minimum width=0.50cm, minimum height=0.98cm] (vels) at (9.25,1.56) {};
          \foreach \yy/\cc/\ff in {1.82/rgRed/rgRedFill,1.56/rgOrange/rgOrangeFill,
          1.30/rgGreen/rgGreenFill} {
            \fill[\ff, draw=\cc, line width=0.4pt]
            (9.15,\yy-0.10) rectangle (9.35,\yy+0.10);
          }
          \node[lab] at (9.25,2.34) {three velocities};
        \end{scope}
        \draw[flow] (frozenmodel.east) -- (vels.west);

        \node[lab, anchor=east, text=rgBlue] at (9.52,0.08) {integrate over $t$};
        \begin{scope}[shift={(9.65,0.31)}]
          \path[fill=rgBlue, draw=rgBlueFill!6, line width=0.85pt, line join=round]
          svg[xscale=0.042,yscale=-0.042] {
            M152.924,300.748c84.319,0,152.912-68.6,152.912-152.918
            c0-39.476-15.312-77.231-42.346-105.564c0,0,3.938-8.857,8.814-19.783
            c4.864-10.926-2.138-18.636-15.648-17.228l-79.125,8.289
            c-13.511,1.411-17.999,11.467-10.021,22.461l46.741,64.393
            c7.986,10.992,17.834,12.31,22.008,2.937l7.56-16.964
            c12.172,18.012,18.976,39.329,18.976,61.459
            c0,60.594-49.288,109.875-109.87,109.875
            c-60.591,0-109.882-49.287-109.882-109.875
            c0-19.086,4.96-37.878,14.357-54.337
            c5.891-10.325,2.3-23.467-8.025-29.357
            c-10.328-5.896-23.464-2.3-29.36,8.031
            C6.923,95.107,0,121.27,0,147.829
            C0,232.148,68.602,300.748,152.924,300.748z
          };
        \end{scope}

        \begin{scope}[shift={(0.23cm,-0.2cm)}]
          \node[module, draw=rgGreen, fill=rgGreenFill!25, text width=1.42cm,
          minimum width=1.70cm, minimum height=0.92cm] (estimate) at (11.42,1.56)
          {MI estimation\\$\widehat I(X;Y)$};
        \end{scope}
        \draw[flow] (vels.east) -- (estimate.west);

      \end{scope}
    \end{tikzpicture}%
  }
  \caption{\textsc{Alice} overview.
    \textbf{(a)} Pretraining: a distribution $p$ drawn from the corpus $\mathcal T$ provides a clean context and queries noised according to indicator $m$; the model regresses the velocity target $z_0-\epsilon$.
  \textbf{(b)} Estimation: samples of an unseen distribution provide a clean context and query $z_0=(x_0,y_0)$; one shared Gaussian perturbation $\epsilon=(\epsilon_X,\epsilon_Y)$ at time $t$ produces the three masked queries 
  $m_{XY}$, $m_X$, and $m_Y$, 
  whose block-wise velocity difference $g$, weighted by $(1-t)/t$, averages to the estimate.}
  \label{fig:schematic}
\end{figure}
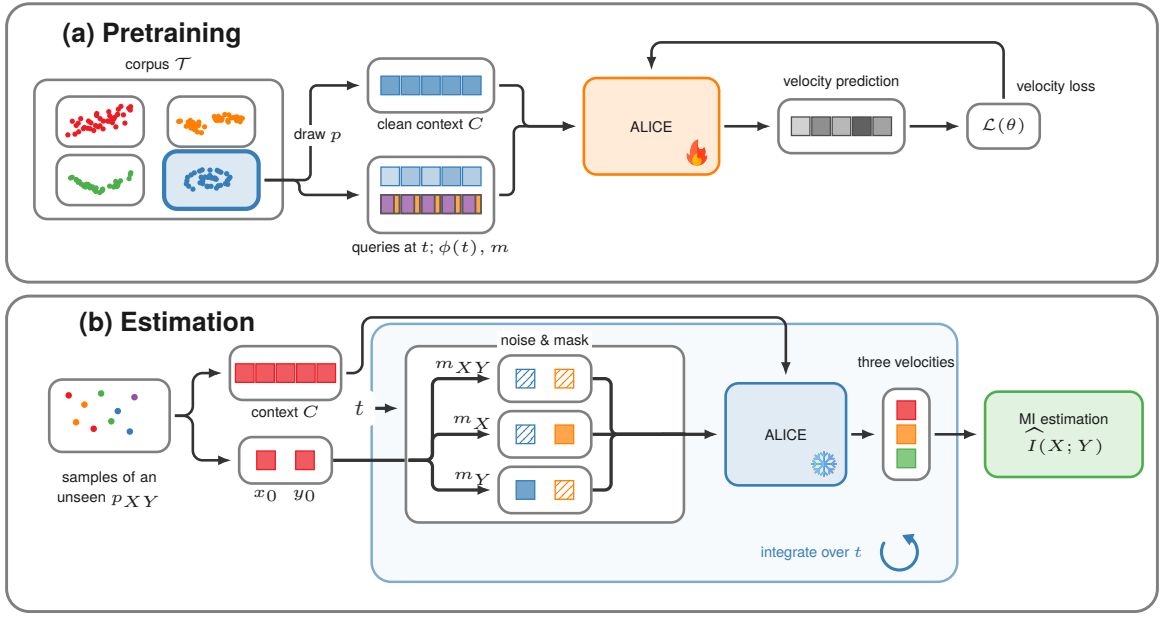

\subsection{Mutual information estimation}
\label{sec:mi-from-kl}

\gls{MI} is the KL divergence between the joint law $p_{XY}$ and the product of its marginals $p_X\otimes p_Y$.
For two densities following a common rectified-flow interpolant, this KL divergence is a time integral of squared differences between their velocity fields \citep{guo2005immse,franzese2024minde,wang2026relative,butakov2026fmmi}.
We now derive the main velocity-form identity for \gls{MI}, while we defer the full derivation and discussion to \Cref{app:mi-variants}.

Let $z_0=(x_0,y_0)\sim p_{XY}$ be a clean joint sample, and let $\epsilon=(\epsilon_X,\epsilon_Y)$ be an independent standard Gaussian perturbation.
The two components of $z_0$ define the $X$ and $Y$ blocks, each of which may contain multiple coordinates.
For $t\in[0,1]$, diffuse the two blocks as $x_t=(1-t)x_0+t\epsilon_X$ and $y_t=(1-t)y_0+t\epsilon_Y$.
Let $z_t=(x_t,y_t)$ denote the jointly noised point.
For a concatenated vector $u=(u_X,u_Y)$, the selections $u|_X$ and $u|_Y$ retain the coordinates in the corresponding blocks.

Let $\vel_t(z)$ denote the joint velocity field.
Let $\vel_t(x\mid y_0)$ denote the conditional velocity field of $p_{X\mid Y=y_0}$ evaluated at $x$.
Let $\vel_t(y\mid x_0)$ denote the conditional velocity field of $p_{Y\mid X=x_0}$ evaluated at $y$.
The resulting identity, proved in \Cref{app:mi-variants} (\Cref{thm:mi-velocity}) and closest in mechanism to the decompositions of \citet{franzese2024minde} and \citet{wang2026relative}, is:
\begin{equation}
  \begin{aligned}
    \MI(X;Y)
    =\int_0^{1}\frac{1-t}{t}\;\Ex_{x_0,y_0,\epsilon}\!\Bigg[
      \Norm{\vel_t(z_t)|_X-\vel_t(x_t\mid y_0)}^2
      +\Norm{\vel_t(z_t)|_Y-\vel_t(y_t\mid x_0)}^2
    \Bigg]\d{t}.
  \end{aligned}
  \label{eq:mi-velocity}
\end{equation}

In principle, \Cref{eq:mi-velocity} involves three velocity fields: the joint field and two block-conditional fields.
In practice, one can amortize these fields with a single model, represented by the parametric velocity field $v_\theta(z,t,m)$ for $z\in\bbR^d$ \citep{franzese2024minde}.
\looseness=-1
The mask $m\in\{0,1\}^d$ identifies the coordinates that are diffused and predicted and those held clean as evidence.
The joint evaluation diffuses both blocks, and each conditional evaluation diffuses one block while holding the other clean.

We estimate the integral in \Cref{eq:mi-velocity} with Monte Carlo.
For each Monte Carlo draw $i$, sample a clean joint point ${z_0^{(i)}=(x_0^{(i)},y_0^{(i)})}$, a time $t_i\sim\cU[0,1]$, and one independent standard Gaussian perturbation $\smash[t]{\epsilon^{(i)}}$. %
Evaluating the definitions above at $t_i$ gives the full query point $z_{t_i}^{(i)}=(x_{t_i}^{(i)},y_{t_i}^{(i)})$.
Define $m_{XY}=(\indone_X,\indone_Y)$, $m_X=(\indone_X,\indzero_Y)$, and $m_Y=(\indzero_X,\indone_Y)$, where $\indone_X$ and $\indzero_X$ are the all-one and all-zero vectors on the $X$ block, with the analogous convention for $Y$.
Then, we have that 
\[
  \widehat v^{(i)}=v_\theta(z_{t_i}^{(i)},t_i,m_{XY}),\qquad
  \widehat v_X^{(i)}=v_\theta((x_{t_i}^{(i)},y_0^{(i)}),t_i,m_X),\qquad
  \widehat v_Y^{(i)}=v_\theta((x_0^{(i)},y_{t_i}^{(i)}),t_i,m_Y).
\]
Using the same $\smash[t]{\epsilon^{(i)}}$ in all three model evaluations, the Monte Carlo estimator is
\begin{equation}
  \widehat\MI(X;Y)=\frac{1}{N_{\mathrm{MC}}}\sum_{i=1}^{N_{\mathrm{MC}}}\frac{1-t_i}{t_i}
  \left[\Norm{(\widehat v^{(i)}-\widehat v_X^{(i)})|_X}^2
  +\Norm{(\widehat v^{(i)}-\widehat v_Y^{(i)})|_Y}^2\right],
  \label{eq:estimator}
\end{equation}
where $N_{\mathrm{MC}}$ is the number of samples 
\looseness=-1
(see \Cref{fig:schematic}--b and \Cref{alg:mi} in \Cref{app:algorithm} for details).

Thus, our estimator requires a model that conditions on a clean sample context and accepts the query point, time, and noising indicator.
The construct that meets these requirements is described next.

\subsection{The in-context velocity field}
\label{sec:method:posterior}

A conventional flow-matching model associates one set of parameters with one distribution $p$ and approximates the map $(z_t,t)\mapsto\vel_t(z_t)$.
Evaluating \Cref{eq:estimator} would require training one separate model for each distribution before its three velocity fields could be queried.
We instead define, train, and use a single context-conditioned velocity model across a family of distributions.
We represent this model as the map from a clean context \emph{and} a query to a velocity, $(\ctx,z_t,t,m)\mapsto v_\theta(z_t,t,m;\ctx)$, where $\ctx=\{z^{(k)}\}_{k=1}^{n}$ is a collection of $n$ clean samples from the unseen distribution.
At inference, the context determines \emph{which} velocity field the in-context learning represents, while the query gives the argument \emph{where} that field is evaluated.
From a statistical learning perspective, the context size contributes to the \emph{bias} of the estimator, while the query contributes to its \emph{variance}.
We implement this context-conditioned map with an attention-based transformer.
A growing literature gives theoretical analyses and empirical demonstrations that transformers can implement learning procedures in their forward pass from in-context data \citep{garg2022what,aky_urek2023what,von_oswald2023transformers,bai2023transformers,xie2022explanation,zhang2025what,xie2025initialization}.
In a setting close to ours, \citet{smart2025incontext} show that a one-layer attention model can solve certain in-context denoising problems optimally.
This motivates our approach, in which pretraining over a family of distributions teaches one shared attention model to infer the distribution-specific velocity computation from the context.

\subsection{Architecture}
\label{sec:method:arch}

The architecture must process the context as a matrix $\ctx\in\bbR^{n\times d}$, with one row per sample and one column per coordinate, for any context size $n$ and joint width $d$ with one set of parameters.
Our method builds on recent work on in-context learning and set transformers, including the scalar tokenization of Chronos \citep{ansari2024chronos}, the any-variate attention of Moirai \citep{woo2024moirai}, and the amortized in-context inference of TabPFN \citep{hollmann2023tabpfn} and TabICL \citep{qu2025tab}.
The rows of the context are evidence about which distribution's velocity field to represent.
The columns carry the dependence between coordinates, which the velocity of one coordinate needs from the values of the others.
\Cref{app:alice-details} complements the high-level description we discuss next.

\noindent \textbf{Size-independent representation.}
Each scalar $z^{(k)}[i]$, coordinate $i$ of context sample $k$, becomes one token, with one input projection and one scalar output head shared over all samples and coordinates.
Context tokens contain a clean value and a type indicator; query tokens additionally contain the time features $\phi(t)$ and the noising-indicator entry $m[i]$.
Since the context rows form a set and coordinate order is arbitrary, the model uses no positional encodings along either axis, and attention over the rows is applied separately to each column.
A validity mask makes padded rows invisible, so the same parameters accept any context size $n$ and are invariant to the order of the rows.

\noindent \textbf{Induced latents.}
Self-attention over the $n$ context tokens in every block would make compute and memory quadratic in $n$.
To keep the cost linear in the context size, we introduce a bottleneck of $K$ induced tokens that summarize the context for each coordinate, drawing inspiration from inducing variables in sparse Gaussian processes \citep{snelson2005sparse,titsias2009variational}, their use in deep Gaussian processes \citep{damianou2013deep,salimbeni2017doubly}, and induced set attention and latent-array architectures \citep{lee2019set,jaegle2021perceiver}.
The $K$ induced tokens are the rows of one learned matrix $U\in\bbR^{K\times D}$, where $D$ is the width of the token representations.
For each coordinate, \model{} updates a copy of $U$ through cross-attention over the context tokens, producing $K$ context-specific latent vectors.
For $L$ blocks, this changes the context-dependent attention cost from $\mathcal O(Ldn^2)$ to $\mathcal O(dnK+LdK^2)$, which grows linearly in $n$ for fixed $K$.

\noindent \textbf{Context-derived relation graph.}
The velocity of one coordinate can depend on the values of other coordinates, and this dependence changes with the distribution.
Separate marginal summaries cannot identify it: independently shuffling one context column preserves its marginal samples while changing which values occur together in a joint observation.
We therefore construct a weighted graph with one node per coordinate, computed once from the clean context, whose edges control information exchange between coordinate representations; this follows the pattern of inferring interactions from observations to guide message passing \citep{kipf2018neural}.
The edge between coordinates $i$ and $j$ is derived from the covariance, over the context, of learned nonlinear features of $z^{(k)}[i]$ and $z^{(k)}[j]$, a principle also used in kernel dependence measures \citep{gretton2005measuring}.
Nonlinear features expose relations such as $z[j]\approx z[i]^2$ that linear correlation misses, and centering the features makes the population descriptor vanish under independence.
Each attention head turns this descriptor into a signed, gated edge and uses the edges to mix the coordinate representations in every graph layer: within each context row before latent compression, and between latent and query representations.
\ifarchdetail
\begin{wrapfigure}{r}{0.33\textwidth}
  \raggedright
  \hspace*{-18pt}%
  \resizebox{1.1\linewidth}{!}{%
    \definecolor{rgRed}{HTML}{ED1C24}
\definecolor{rgRedFill}{HTML}{EF5B61}
\definecolor{rgOrange}{HTML}{FF7F00}
\definecolor{rgOrangeFill}{HTML}{FFA64D}
\definecolor{rgGreen}{HTML}{4DAF4A}
\definecolor{rgGreenFill}{HTML}{86C982}
\definecolor{rgBlue}{HTML}{377EB8}
\definecolor{rgBlueFill}{HTML}{73A4CC}
\definecolor{rgPurple}{HTML}{984EA3}
\definecolor{rgPurpleFill}{HTML}{B17CBB}
\definecolor{rgGray}{HTML}{7F7F7F}

\begin{tikzpicture}[
    x=1cm,
    y=1cm,
    font=\sffamily\scriptsize,
    text=black!90,
    >={Latex[length=2mm,width=1.6mm]},
    panel/.style={draw=rgGray, line width=0.9pt, rounded corners=7pt, fill=white},
    module/.style={draw=rgGray, line width=0.8pt, rounded corners=4pt, fill=white},
    flow/.style={->, line width=0.8pt, black!80},
    wire/.style={line width=0.55pt, black!70},
    label/.style={inner sep=1pt, align=center},
  ]
  \node[panel, minimum width=6.15cm, minimum height=0.68cm] at (3.10,9.84) {};
  \fill[rgRedFill, draw=rgRed, line width=0.75pt] (0.28,9.67) rectangle (0.66,10.01);
  \node[anchor=west] at (0.76,9.84) {coordinate};
  \fill[rgRedFill, draw=rgRed, line width=0.65pt] (2.40,9.68) rectangle (2.63,10.00);
  \fill[rgGreenFill, draw=rgGreen, line width=0.65pt] (2.63,9.68) rectangle (2.86,10.00);
  \node[anchor=west] at (2.96,9.84) {pair $j\!\to\!i$};
  \draw[flow] (4.62,9.84) -- (5.04,9.84);
  \node[anchor=west] at (5.14,9.84) {flow};

  \node[panel, minimum width=6.15cm, minimum height=2.95cm] at (3.10,7.73) {};
  \node[font=\sffamily\small\bfseries] at (3.10,8.99) {Context-derived relation graph};

  \node[label] at (1.18,8.48) {clean context $\ctx$\\paired sample rows};
  \foreach \yy in {8.05,7.75,7.45,7.15} {
    \fill[rgRedFill, draw=rgRed, line width=0.45pt] (0.43,\yy-0.10) rectangle (0.80,\yy+0.10);
    \fill[rgOrangeFill, draw=rgOrange, line width=0.45pt] (0.80,\yy-0.10) rectangle (1.17,\yy+0.10);
    \fill[rgGreenFill, draw=rgGreen, line width=0.45pt] (1.17,\yy-0.10) rectangle (1.54,\yy+0.10);
    \fill[rgBlueFill, draw=rgBlue, line width=0.45pt] (1.54,\yy-0.10) rectangle (1.91,\yy+0.10);
  }
  \node at (1.17,6.84) {$\vdots$};
  \draw[flow] (2.08,7.60) -- (2.48,7.60);

  \node[module, minimum width=1.52cm, minimum height=2.08cm] (features) at (3.28,7.60) {};
  \node[label] at (3.28,8.35) {shared maps\\$\ell,r$};
  \fill[rgRedFill, draw=rgRed, line width=0.55pt] (2.80,7.86) rectangle (3.11,8.13);
  \fill[white, opacity=0.35] (2.87,7.92) rectangle (3.04,8.07);
  \fill[rgGreenFill, draw=rgGreen, line width=0.55pt] (3.15,7.86) rectangle (3.46,8.13);
  \fill[white, opacity=0.35] (3.22,7.92) rectangle (3.39,8.07);
  \draw[flow] (3.13,7.78) -- (3.13,7.55);
  \fill[rgRedFill, draw=rgRed, line width=0.55pt] (2.84,7.26) rectangle (3.13,7.51);
  \fill[white, opacity=0.35] (2.90,7.31) rectangle (3.07,7.46);
  \fill[rgGreenFill, draw=rgGreen, line width=0.55pt] (3.13,7.26) rectangle (3.42,7.51);
  \fill[white, opacity=0.35] (3.19,7.31) rectangle (3.36,7.46);
  \node[label, font=\sffamily\tiny] at (3.28,6.91) {row average};
  \draw[flow] (4.08,7.60) -- (4.45,7.60);

  \node[label] at (5.30,8.48) {relation graph\\$A^{(h)}$};
  \coordinate (g1) at (5.16,8.02);
  \coordinate (g2) at (5.72,7.69);
  \coordinate (g3) at (5.56,7.05);
  \coordinate (g4) at (4.91,7.08);
  \draw[rgPurple, line width=1.5pt] (g1) -- (g2);
  \draw[rgGreen, line width=2.1pt] (g2) -- (g3);
  \draw[rgBlue, line width=1.0pt] (g3) -- (g4);
  \draw[rgOrange, line width=1.7pt] (g4) -- (g1);
  \draw[black!35, line width=0.55pt] (g1) -- (g3);
  \draw[black!35, line width=0.55pt] (g2) -- (g4);
  \fill[rgRedFill, draw=rgRed, line width=0.8pt] (g1) circle (0.15);
  \fill[rgOrangeFill, draw=rgOrange, line width=0.8pt] (g2) circle (0.15);
  \fill[rgGreenFill, draw=rgGreen, line width=0.8pt] (g3) circle (0.15);
  \fill[rgBlueFill, draw=rgBlue, line width=0.8pt] (g4) circle (0.15);

  \begin{scope}[yshift=0.6cm]
    \node[panel, minimum width=6.15cm, minimum height=3.68cm] at (3.10,3.19) {};
    \node[font=\sffamily\small\bfseries] at (3.10,4.67) {Graph attention};

    \node[label] at (0.91,4.19) {$\mathbf{r}_j^{(h)}$};
    \fill[rgRedFill, draw=rgRed, line width=0.70pt] (0.43,3.71) rectangle (1.39,4.03);
    \fill[rgOrangeFill, draw=rgOrange, line width=0.70pt] (0.43,3.26) rectangle (1.39,3.58);
    \fill[rgGreenFill, draw=rgGreen, line width=0.70pt] (0.43,2.81) rectangle (1.39,3.13);
    \fill[rgBlueFill, draw=rgBlue, line width=0.70pt] (0.43,2.36) rectangle (1.39,2.68);
    \fill[rgPurpleFill, draw=rgPurple, line width=0.70pt] (0.43,1.91) rectangle (1.39,2.23);

    \node[label] at (2.46,4.19) {edges $A_{ij}^{(h)}$};
    \foreach \yy/\cc/\ff in {3.87/rgRed/rgRedFill,3.42/rgOrange/rgOrangeFill,2.97/rgGreen/rgGreenFill,2.52/rgBlue/rgBlueFill,2.07/rgPurple/rgPurpleFill} {
      \fill[\ff, draw=\cc, line width=0.55pt] (2.05,\yy-0.11) rectangle (2.34,\yy+0.11);
      \fill[rgGreenFill, draw=rgGreen, line width=0.55pt] (2.38,\yy-0.11) rectangle (2.67,\yy+0.11);
    }
    \foreach \yy in {3.87,3.42,2.97,2.52,2.07} {
      \draw[wire] (1.45,\yy) -- (1.96,\yy);
    }
    \node[label, font=\sffamily\tiny] at (2.46,1.67) {green target $i$};
    \draw[flow] (2.82,2.97) -- (3.15,2.97);

    \node[module, minimum width=1.32cm, minimum height=2.64cm] (aggregate) at (3.87,2.97) {};
    \node[label, font=\sffamily\tiny] at (3.87,4.00) {weighted\\messages};
    \foreach \yy/\cc/\ff in {3.72/rgRed/rgRedFill,3.43/rgOrange/rgOrangeFill,3.14/rgBlue/rgBlueFill} {
      \fill[\ff, draw=\cc, line width=0.50pt] (3.48,\yy-0.10) rectangle (3.75,\yy+0.10);
      \fill[white, opacity=0.35] (3.54,\yy-0.06) rectangle (3.69,\yy+0.06);
      \fill[rgGreenFill, draw=rgGreen, line width=0.50pt] (3.78,\yy-0.10) rectangle (4.05,\yy+0.10);
      \fill[white, opacity=0.35] (3.84,\yy-0.06) rectangle (3.99,\yy+0.06);
    }
    \draw[flow] (3.87,3.00) -- (3.87,2.80);
    \node[font=\normalsize] at (3.87,2.57) {$\sum$};
    \draw[flow] (3.87,2.32) -- (3.87,2.16);
    \node[label, font=\sffamily\tiny] at (3.87,1.91) {normalize};
    \draw[flow] (4.57,2.97) -- (4.88,2.97);

    \node[label] at (5.46,4.15) {message $u_i^{(h)}$};
    \fill[rgGreenFill, draw=rgGreen, line width=0.85pt] (4.97,2.79) rectangle (5.95,3.15);
    \fill[white, opacity=0.35] (5.08,2.86) rectangle (5.84,3.08);
    \node at (5.46,2.97) {$u_i^{(h)}$};
    \draw[flow] (5.46,2.74) -- (5.46,2.55);
    \fill[rgGreenFill, draw=rgGreen, line width=0.9pt] (4.82,2.12) rectangle (6.10,2.52);
    \fill[white, opacity=0.30] (4.96,2.20) rectangle (5.96,2.44);
    \node[label, font=\sffamily\tiny] at (5.46,1.73) {residual update};
  \end{scope}

  \draw[flow, rounded corners=3pt]
  (5.56,6.88) -- (5.56,5.96) -- (3.10,5.96) -- (3.10,5.59);
  \node[label, fill=white, font=\tiny] at (3.6,6.02) {$A^{(h)}$};
\end{tikzpicture}%
  }
  \caption{\model{} relation-graph construction and coordinate mixing.
    Node and block colors identify coordinates, and each two-color block denotes a source-target pair $j\to i$.
    Top: shared feature maps $\ell,r$ aggregate paired clean-context rows into signed, gated relation weights $A_{ij}^{(h)}$.
  Bottom: for the green target $i$, each source representation $\mathbf{r}_j^{(h)}$ is weighted by its incoming edge, and the center module sums and normalizes these contributions to produce $u_i^{(h)}$ for the residual update.}
  \label{fig:method-relation-graph}
\end{wrapfigure}
Let $\smash[t]{\hat z^{(k)}[i]}$ be coordinate $i$ of clean context sample $k$, standardized over the context.
Two learned maps $\ell,r:\bbR\to\bbR^R$, each shared across coordinates and context samples, take this single scalar as input and output $R$ nonlinear features.
Subtracting each feature's context mean gives $\smash[t]{\tilde\ell_i^{(k)}}$ and $\smash[t]{\tilde r_i^{(k)}}$.
The descriptor of coordinates $i$ and $j$ is then
\begin{equation}
  c_{ij}=\frac{1}{2n}\sum_{k=1}^{n}
  \left(\tilde\ell_i^{(k)}\odot\tilde r_j^{(k)}
  +\tilde r_i^{(k)}\odot\tilde\ell_j^{(k)}\right)\in\bbR^R.
  \label{eq:method-relation}
\end{equation}
Each component of $c_{ij}$ is an average of two empirical feature covariances, with the common index $k$ preserving the joint observations and symmetrization giving $c_{ij}=c_{ji}$.
Centering makes the population descriptor vanish under independence, although finite contexts introduce sampling fluctuations.

For graph head $h$, learned projection vectors $\smash[t]{w_s^{(h)}},\smash[t]{w_g^{(h)}}\in\bbR^R$ and scalar biases $\smash[t]{b_s^{(h)}},\smash[t]{b_g^{(h)}}$ convert the descriptor into a signed, gated edge.
With $\sigma$ the sigmoid and $\smash[t]{\mathbf{r}_j^{(h)}}$ the projected representation of coordinate $j$ at the same context, latent, or query position, the edge and aggregated message are
\begin{equation}
  A_{ij}^{(h)}=\tanh\!\left((w_s^{(h)})^\top c_{ij}+b_s^{(h)}\right)
  \sigma\!\left((w_g^{(h)})^\top c_{ij}+b_g^{(h)}\right),
  \qquad
  u_i^{(h)}=\frac{\sum_{j\ne i}A_{ij}^{(h)}\mathbf{r}_j^{(h)}}
  {\max\!\left(1,\sum_{j\ne i}|A_{ij}^{(h)}|\right)}.
  \label{eq:method-graph}
\end{equation}
The signed factor allows additive or subtractive contributions, while the gate controls their magnitude.
Normalizing by absolute edge mass bounds the aggregate contribution, and the lower bound of one preserves small updates when all edges are weak.
Initializing gates near zero suppresses cross-coordinate messages at the start of training; exact disconnection under independence is not enforced.
The feature maps and edge projections are learned through the velocity objective, so the edges represent pairwise associations useful for prediction without imposing a conditional-independence interpretation.
Sharing these maps across coordinates keeps the parameter count independent of $d$ and makes the graph equivariant to coordinate permutations.
\fi

\noindent \textbf{Attention pattern.}
We use separate attention operations for context, latent, and query representations.
Latent representations are updated by cross-attention from context tokens and by latent self-attention.
Query representations attend to the latent and context representations and do not attend to one another, so each query is processed independently conditional on the same context-derived states.
The shared attention and graph operations, together with the absence of positional encodings, preserve invariance to permutations of context samples and equivariance to permutations of coordinates.

\noindent \textbf{Caching.}
The relation graph, the context tokens after the input projection, and the induced latents depend only on $\ctx$, so they are computed once and reused across all queries for that distribution.
Although this is not strictly useful during training, it is essential for inference on large contexts, where recomputing these states for every query would otherwise be costly.

\subsection{Pretraining}
\label{sec:method:pretraining}

We train \model{} to infer the velocity field of an unseen distribution from its context.

\noindent \textbf{Pretraining corpus $\mathcal{T}$.}
Each episode is a synthetic joint distribution over $z=(x,y)\in\bbR^d$.
The corpus combines base distributions (Gaussian, Student's t) with copula mixtures, latent warps, nonparametric regressions, and manifolds to vary dependence structure, conditional behavior, and support geometry.
Copula mixtures follow the dependence-diversity construction of InfoAtlas \citep{hu2026infoatlas} and use additive-coupling bijections \citep{dinh2017realnvp} to enrich the sampled dependencies.
Latent warps transform Gaussian mixtures through shifts, folds, and rotations, nonparametric regressions generate responses from random Fourier-feature functions with noise, and manifolds generate near-singular supports.
\Cref{app:training:corpus} gives the full construction.

\noindent \textbf{Training objective.}
The identity in \Cref{eq:mi-velocity} reduces \gls{MI} estimation to differences between a joint velocity field and masked conditional fields, so pretraining focuses on predicting these fields.
The objective uses samples from each synthetic distribution and requires no \gls{MI} labels, which lets one model amortize velocity-field estimation across the corpus $\cT$.
At each step, we sample a distribution $p\sim\cT$, a context $\ctx$ of $n$ independent samples from $p$, a further clean sample $z_0\sim p$, a time $t$, Gaussian noise $\epsilon$, and a noising indicator $m$.
The indicator selects the coordinates that follow the interpolant in the query point $z_t=m\odot\bigl((1-t)z_0+t\epsilon\bigr)+(1-m)\odot z_0$, while the remaining coordinates stay clean as evidence (\Cref{fig:schematic}--a).
We then regress the model output on the flow-matching direction over the selected coordinates:
\begin{equation}
  \cL(\netpar)=\Ex_{p\sim\cT,\,\ctx,\,z_0\sim p,\,t\sim\cU[0,1],\,\epsilon,\,m}
  \left[\Norm{m}_1^{-1}\Norm{m\odot\bigl(\vmodel(z_t,t,m;\ctx)-(z_0-\epsilon)\bigr)}^2\right].
  \label{eq:objective}
\end{equation}
The factor $\Norm{m}_1^{-1}$ makes the loss a mean over noised coordinates, and the target
contains no $1/t$ factor, so the training loss has no singularity as $t\to0$.
The noising indicator is sampled to cover the three fields required by \Cref{eq:mi-velocity}, all coordinates noised for the joint field and one block noised while the other remains clean for each block-conditional field, together with random coordinate subsets for general partial observation; since $m$ is an input, one network represents all of these fields.

\needspace{8\baselineskip}
\section{Validation}
\label{sec:experiments}
\label{sec:exp:czyz}
\begin{figure}[t]
  \centering
  \includegraphics[width=\textwidth]{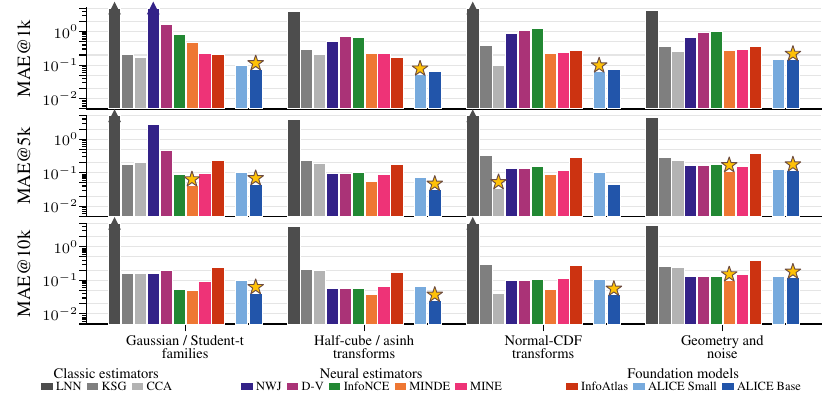}
  \caption{Category-wise MAE on the $40$-task \citet{czyz2023beyond} benchmark at matched data budgets.
    Rows show budgets of $1$k, $5$k, and $10$k samples, and columns group tasks by base family or transformation.
  Upward triangles mark values above the plotted range. Stars mark the lowest MAE.}%
  \label{fig:czyz-category-mae}
\end{figure}
We evaluate \model{} (Small and Base variants, see \Cref{app:alice-details}) on the $40$ tasks of the suite by \citet{czyz2023beyond}, which spans joint widths from $2$ to $100$ and provides a closed-form ground-truth \gls{MI}.

\noindent \textbf{Protocol.}
Every estimator receives the same data budget of $N\in\{1\text{k},5\text{k},10\text{k}\}$ samples per task.
\model{} splits the budget into $64$ query samples and a context of the remaining $N-64$ clean samples, and estimates \gls{MI} in-context; estimates are in nats, averaged over eight independent context draws, and \Cref{app:czyz} gives the full inference settings.
We compare against the following estimators: InfoAtlas \citep{hu2026infoatlas}, MINDE \citep{franzese2024minde}, MINE \citep{belghazi2018mine}, InfoNCE \citep{oord2019representation}, D-V \citep{donsker1975asymptotic}, NWJ \citep{nguyen2010nwj}, KSG \citep{kraskov2004estimating}, LNN \citep{gao2015efficient}, and CCA \citep{hotelling1936relations}.
InfoAtlas is the amortized baseline and is evaluated zero-shot on the same contexts as \model{}.
Neural estimators are trained (and tuned) separately on the $N$ samples of each distribution, and classic estimators are fit directly on them.

\noindent \textbf{Results.}
We report the mean absolute error (MAE) between estimate and ground truth over the $40$ tasks, in nats; \Cref{fig:czyz-category-mae} shows it per distribution group and \Cref{tab:czyz-summary} in \Cref{app:czyz} over the whole suite.
\model{} Base has the lowest MAE at every budget: $0.092$ nats at $1$k samples, $0.063$ at $5$k, and $0.060$ at $10$k, against $0.195$ for the best competitor at $1$k (CCA) and $0.070$ and $0.065$ for MINDE at $5$k and $10$k.
Within groups, \model{} has the lowest MAE in all four groups at $1$k and in the three Gaussian-based groups at $10$k.
\model{} Small ($19$M parameters, against $85$M for Base) has an MAE of $0.10$ nats at every budget: second-lowest at $1$k, and below MINE, D-V, NWJ, and the classic estimators at $5$k and $10$k.
The two sizes differ on the wide tasks alone: on the $7$ tasks of joint width $50$ and $100$ the MAE of Base falls from $0.153$ nats at $1$k to $0.086$ at $10$k while that of Small rises from $0.197$ to $0.241$; on the $33$ narrower tasks they are within $0.02$ nats of each other (\Cref{tab:czyz-width}).
InfoAtlas has an MAE between $0.25$ and $0.28$ nats at every budget.
\begin{figure}[t]
  \centering
  \includegraphics[width=0.42\textwidth]{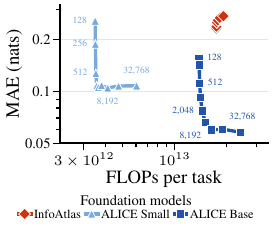}
  \caption{MAE against inference FLOPs per task and seed on the $40$-task Czy\.{z} benchmark; labels show \model{} total budgets ($64$ queries).
  }
  \label{fig:czyz-flops}
\end{figure}

\noindent \textbf{Small-data regime.}
The trained estimators need the step from $1$k to $5$k/$10$k samples to become practically usable: the MAE of MINDE falls from $0.353$ to $0.070$ nats, that of InfoNCE from $0.887$ to $0.123$, and that of D-V from $1.229$ to $0.279$, while NWJ diverges at $1$k and is still at $1.261$ nats at $5$k.
At a $1$k budget, \model{} Base outperforms every competitor by a factor of at least two, and is more accurate than MINE, InfoNCE, D-V, and NWJ at $5$k.

\noindent \textbf{Inference compute.}
\Cref{fig:czyz-flops} compares operator-level FLOPs per task for \model{} and InfoAtlas.
Here, \model{} uses a budget $N$ from $128$ to $32768$, split into $N-64$ context samples and $64$ query samples, with $64$ time draws per query; InfoAtlas uses contexts from $128$ to $8192$ samples.
For each task, we average the \gls{MI} estimates over $32$ seeds for \model{} and eight seeds for InfoAtlas, then compute MAE across the $40$ tasks.
With a budget of $8192$, Small and Base achieve MAEs of $0.105$ and $0.060$ nats for $4.1\cdot10^{12}$ and $1.6\cdot10^{13}$ FLOPs per task, respectively, compared with $0.274$ nats for $1.9\cdot10^{13}$ FLOPs for InfoAtlas.
Increasing the budget to $32768$ gives Base an MAE of $0.058$ nats for $2.4\cdot10^{13}$ FLOPs per task.

\section{Applications}
\label{sec:applications}
In this section we showcase \model{} on three scientific applications: biology, genetics and neuroscience.
In these applications, datasets include discrete distributions, sequences of tokens, and time-series of real numbers: not only \model{}'s training corpus never encountered such distribution types, the model itself has never been trained on such data.

\subsection{Analysis of multivariate single-cell signaling responses}
\label{sec:exp:sc-signalling}
Cellular signaling can be naturally described in information-theoretic terms: an extracellular stimulus $X$ is transmitted through a stochastic biochemical network to a cellular response $Y$; $\MI(X;Y)$ measures how reliably a cell can infer the stimulus, and the channel capacity is the maximum of $\MI(X;Y)$ over input distributions \citep{nurse2008life,brennan2012information,jetka2018information}. We use \model{} to analyze the NF-$\mathcal{K}$B pathway, which responds to the inflammatory cytokine TNF-$\alpha$~\citep{jetka2019information}: $15{,}632$ cells stimulated with one of $m=11$ TNF-$\alpha$ concentrations ($0$ to $100$ ng/ml) and imaged for $2$ h at $3$-min resolution (40 frames in total), the response being the nuclear-to-cytoplasmic NF-$\mathcal{K}$B ratio.

\noindent \textbf{Protocol.} \gls{MI} decomposes as $\MI(X;Y)=\sum_i p_i D_i$, where $D_i$ is the divergence of the response distribution at dose $i$ from the mixture over doses. We estimate every $D_i$ from the velocity difference between a joint context and a label-shuffled context (see \Cref{app:mi-variants:cond} for a detailed formulation); \model{} returns both fields, with no training. From the $D_i$ we obtain the \gls{MI} at uniform input, the capacity by Blahut--Arimoto ascent, and the probability of correct discrimination (PCD) of every pair of doses, bracketed by the Jensen--Shannon divergence
(see \Cref{app:slemi-details} for details).
This is a small-data regime: after the filtering of the reference analysis, a dose has between $536$ and $1{,}307$ cell samples, and each $D_i$ is estimated from a context of $1{,}024$ cells.

\begin{figure}[t]
  \centering
  \begin{tikzpicture}
    \node[inner sep=0pt] (nfkbfigure) {\includegraphics[width=\textwidth]{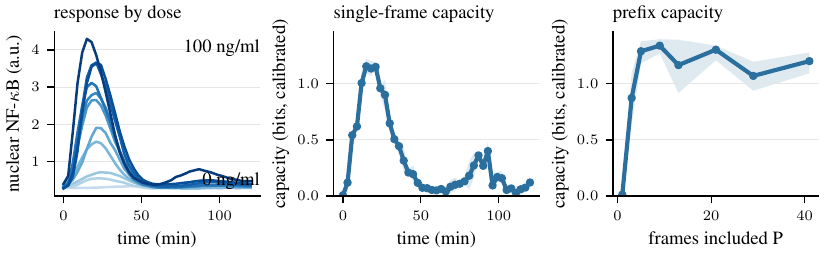}};
    \node[anchor=north] at ([xshift=0.190\textwidth,yshift=2pt]nfkbfigure.south west) {(a)};
    \node[anchor=north] at ([xshift=0.529\textwidth,yshift=2pt]nfkbfigure.south west) {(b)};
    \node[anchor=north] at ([xshift=0.868\textwidth,yshift=2pt]nfkbfigure.south west) {(c)};
  \end{tikzpicture}\\[-10pt]
  \caption{In-context analysis of the NF-$\mathcal{K}$B dose channel.
    (a) Median nuclear NF-$\mathcal{K}$B response per TNF-$\alpha$ dose, with the interquartile band at the two extreme doses.
    (b) Capacity of each single frame and (c) of the prefix of frames from minute $0$, with the envelope over seeds.
  The pairwise discrimination matrices are in \Cref{app:slemi-details}.}
  \label{fig:nfkb}
\end{figure}

\noindent \textbf{Results.} \Cref{fig:nfkb} shows the three findings of the analysis.
First, the information carried by a single frame follows the NF-$\mathcal{K}$B translocation: the capacity of one frame rises with the first nuclear peak, reaches about $1.2$ bits at minutes $15$ to $21$, and decays, with a smaller second rise at the second peak (panel (b)).
Second, the trajectory carries more than any single frame: the capacity of a prefix of frames saturates by frame $12$ (panel (c)), so the dose is encoded in the timing and amplitude of the first response peak, and the prefix of frames reaches a capacity of about $1.3$ bits against $1.2$ for the best single frame.
Third, dynamics separate the high doses: from a single frame, pairs of doses at or above $0.5$ ng/ml are close to indistinguishable (mean PCD $0.56$, where chance is $0.5$), and the trajectory raises their PCD to $0.67$, while the low doses are separable from a single frame already (\Cref{fig:nfkb-pcd} in \Cref{app:slemi-details}).
The three findings, the position and height of the capacity peak, and the discrimination pattern are obtained from a single model never trained on biological data, and not only corroborate those of \citet{jetka2019information}, but overcome the limiting assumptions required approximate \gls{MI} by fitting a linear classifier per analysis, which might not hold in more complex scenarios.

\subsection{Promoter Identification}
\label{sec:exp:promoter}
Regulatory motifs are short DNA patterns that control gene expression, and \gls{MI}-based methods locate them by measuring the dependence between the content of a regulatory region and the expression it drives \citep{elemento2007universal, rao2007motif}. We use \model{} to locate the \textsc{tata-box}, a core promoter motif whose preferred position in \textit{Arabidopsis thaliana} lies $26$ to $39$ bases upstream of the transcription start site (TSS) \citep{bernard2010tc}, on the promoter and non-promoter sequences of \citet{umarov2017recognition} from the \textsc{epd} database \citep{dreos2013epd}: $1{,}497$ sequences per class after balancing, each of $251$ bases spanning positions $-200$ to $+50$ around the TSS, so the promoter label $X$ is uniform and every \gls{MI} value is bounded by $H(X)=\ln 2$ nats.

\noindent \textbf{Protocol.} For a window of $L\in\{4,6\}$ bases starting at position $s$, we estimate $\MI(X;Y)$ between the promoter label $X$ and the window content $Y$. Sliding the window along the sequence produces a \gls{MI} profile: windows on segments unrelated to promoter status yield near zero values, and windows overlapping an informative motif obtain high values. Each window is scored on its own, so a motif is detected even when another motif is correlated with it.
Bases are input in the velocity fields of \Cref{eq:mi-velocity} through a fixed injective embedding of one real coordinate per base, which preserves $\MI(X;Y)$ exactly, and the blocks of dimension $1$ and $L$ are handled natively by \model{}. For every window position, \model{} conditions on a context of $1{,}024$ label--window pairs and evaluates on $1{,}024$ held-out pairs. This is a small-data regime: the whole dataset holds $2{,}994$ sequences, an order of magnitude below the training sets that neural estimators require, and \model{} produces each window from $1{,}024$ of them (see \Cref{app:promoter-details} for additional details).

\begin{figure}[t]
  \centering
  \begin{minipage}[t]{0.38\textwidth}
    \centering
    \includegraphics[width=\linewidth]{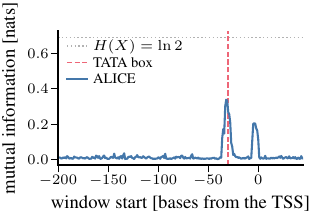}
    \captionof{figure}{\gls{MI} between the promoter label and a sliding window of $L=6$, against the offset of the window start from the TSS; the dotted line marks the $\ln 2$ ceiling of the label entropy.}
    \label{fig:tata-box-search}
  \end{minipage}\hfill
  \begin{minipage}[t]{0.60\textwidth}
    \centering
    \includegraphics[width=\linewidth]{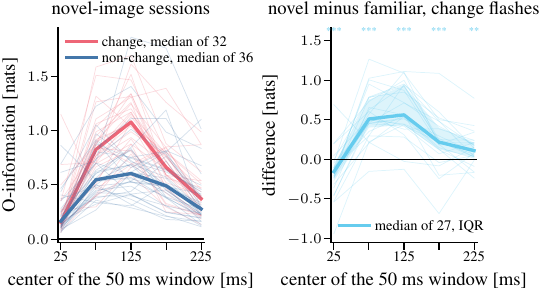}
    \captionof{figure}{
    \acrshort{O-information} of the six visual areas, one estimate per session.
    Left: novel-image sessions, one thin line per mouse and flash type, medians in bold.
    Right: novel minus familiar session for change flashes, with the median, the interquartile band, and a signed-rank test per window ( $**$: $p<0.01$, $***$: $p<0.001$).}
    \label{fig:soi-mice}
  \end{minipage}
  \vspace{-12pt}
\end{figure}

\noindent \textbf{Results.} \Cref{fig:tata-box-search} shows the profile for $L=6$. The estimate is flat and near zero over the $200$ bases upstream of the motif and over the $50$ bases downstream of the TSS, rises sharply over the \textsc{tata-box} band, with its maximum at $30$ to $32$ bases upstream of the TSS for both window lengths, and shows a second, smaller maximum on the TSS itself, which corresponds to the initiator element. The maximum lies inside the documented \textsc{tata-box} band, which serves as a positive control for localization.
The existing neural competitor for this task is \textsc{Info-SEDD} \citep{foresti2026infosedd}, a discrete-diffusion estimator trained on this dataset, which locates the \textsc{tata-box} with windows realized by masking. Its profile scans positions $-60$ to $-25$ and reports a single peak, with a bias floor substantially higher than  \model{} (see \Cref{app:promoter-details} for additional results).

\subsection{Brain Region Activity Patterns}
\label{sec:exp:soi-mice}
We use \model{} to estimate the \gls{O-information} \citep{Rosas2019QuantifyingHI} of six visual-cortex areas of mice performing a visual change-detection task, on the Visual Behavior Neuropixels recordings of the Allen Institute \citep{allen-inst}, first analyzed by \citet{venkatesh2023gaussian}. \citet{bounoua2024} estimated the \acrshort{O-information} of these recordings with a score-based estimator whose networks are trained on all sessions pooled together; we obtain the estimate from our  frozen \model{} checkpoint, with no training on neural data, for every single session. 
For $N$ random variables, the quantity $\Omega=\mathrm{TC}-\mathrm{DTC}$ is the difference between the total correlation and the dual total correlation; a positive value indicates that redundancy dominates the interactions, that is, the variables carry overlapping information, and a negative value indicates that synergy dominates. Both terms are time integrals of squared velocity differences of the form of \Cref{eq:estimator}, between the joint field and the concatenation of the $N$ marginal fields ($\mathrm{TC}$) or of the $N$ conditional fields ($\mathrm{DTC}$), so one checkpoint provides all the necessary fields (see \Cref{app:soi-mice-details} for details and validation).

\noindent \textbf{Protocol.} A mouse watches a natural image shown for $250$ ms every $750$ ms; the image repeats for several presentations (flashes) and then changes. We use the $72$ sessions selected by \citet{bounoua2024}: $36$ mice, each recorded on one day with a familiar image set and on another day with a novel one. For every flash, spikes are counted in five consecutive $50$ ms windows and averaged over the units of each of six visual areas, so one flash is one draw of six variables and each window is one system of joint width six; change flashes and non-change flashes (repeats) are analyzed separately. Each session is estimated separately: \model{} conditions on a context of $128$ flashes of the session and evaluates on the remaining flashes. Paired comparisons follow between the two flash types of a session and between the two sessions of a mouse. This is a small-data regime: a session provides about $150$ to $200$ independent flashes per flash type (\Cref{app:soi-mice-details}), too few to train an estimator per session, so \citet{bounoua2024} pool all $72$ sessions, and obtain no per-session estimate.

\noindent \textbf{Results.} \Cref{fig:soi-mice} shows how the six areas share information after a flash. In novel-image sessions the \acrshort{O-information} is positive in every window and every mouse: the areas carry overlapping information. This redundancy is low at flash onset, maximal at $100$ to $150$ ms, when the visual response has reached all six areas, and decays afterwards. A change of image produces more redundancy than a repeat: in the peak window the within-session difference is positive in $31$ of $32$ sessions, and it is absent in the first window, before the visual response reaches the cortex. The same comparison in familiar-image sessions gives no difference at the peak and a reversed sign in the late windows, and the two sessions of each mouse (\Cref{fig:soi-mice}, right) show that novelty raises the redundancy of the response in $25$ of $27$ animals for change flashes and in $28$ of $36$ for non-change flashes. We observe that a novel image drives a stimulus signal that is broadcast across the visual areas, and that this shared component fades with familiarity. The pooled result of \citet{bounoua2024}, a larger \acrshort{O-information} after a change flash, therefore holds for novel images and in one window only, and the dependence on experience is a new finding of this work: the pooled analysis merges the two days of every mouse and cannot separate them. \model{} produced the $675$ per-session systems in few forward passes of one frozen model; a trained-per-system estimator would require $675$ training runs (\Cref{app:soi-mice-details}).

\section{Conclusion and Limitations}
\label{sec:conclusion}
\label{sec:limitations}

We presented \model{}, the first foundation model for \gls{MI} estimation, whose zero-shot accuracy matches that of existing estimators trained per distribution. \model{} is a single Transformer, trained once as an in-context rectified-flow velocity field, that provides the joint and block-conditional fields of an unseen distribution from samples alone. A known identity uses such fields to estimate \gls{MI}.

\model{} was trained exclusively on synthetic data and produced \gls{MI} estimates zero-shot, on distributions and data types absent from its training corpus.
On the ``Beyond Normal'' benchmark, it has the lowest error of all estimators at matched budgets of $1$k, $5$k, and $10$k samples, although every competitor is trained and tuned on each distribution; at $1$k samples, its error is lower by a factor of at least two.
In three scientific applications, \model{} reproduces the findings of dedicated estimators in the small-data regime typical of biology and neuroscience, with a few forward passes per estimate.

We believe \model{} to be an invaluable asset for scientific discoveries across fields, that materializes as a local model that can be run ``plug-and-play'' on modest hardware.

\noindent \textbf{Limitations.}
Our implementation is research code, and it has not been thoroughly optimized.
Model size and training budget can be increased, training corpus can be augmented with higher dimensional data, maximum context size at training time can be increased, which might yield even better results in our benchmark validation.

\section*{Acknowledgments}
This project was provided with AI computing and storage resources by GENCI at IDRIS thanks to the grant AD011018178 on the supercomputer Jean Zay's H100 partition. The Authors acknowledge the support of CIRCALIS AI-HPC facility at EURECOM, with partial funding from French Region Sud.

\bibliographystyle{plainnat}
\bibliography{references,references_fminfo}

\newpage

\beginappendix
\crefalias{section}{appendix}
\crefalias{subsection}{subappendix}
\crefalias{subsubsection}{subsubappendix}
\section{Related Work}
\label{sec:related}

We here expand on the closest prior works: per-distribution neural estimators, in-context inference with Transformers, and amortized estimation from synthetic corpora.

\paragraph{Variational MI estimation.}
Neural lower bounds (MINE \citep{belghazi2018mine}, InfoNCE/CPC
\citep{oord2019representation}, NWJ \citep{nguyen2010nwj}, and the bias/variance study of
SMILE \citep{song2020smile} and \citet{poole2019variational}) optimize a bound
per distribution and are the standard against which diffusion estimators are measured.
Classic $k$NN estimators \citep{kraskov2004estimating} remain strong nonparametric
baselines, and MIENF \citep{butakov2024mienf} fits normalizing flows that separate the copula from the marginals.

\paragraph{Diffusion and information.}
MINDE \citep{franzese2024minde} expresses MI through a score-difference integral;
information-theoretic diffusion \citep{kong2023infodiffusion} and the MMSE-gap
estimator \citep{mmg2025} develop the denoiser view; the velocity-form relative
entropy of \citet{wang2026relative} provides the basic estimator we use; InfoBridge \citep{kholkin2026infobridge} replaces score matching by bridge matching and obtains an exact drift-difference identity. All connect to the I-MMSE relation \citep{guo2005immse}
and the likelihood weighting of \citet{song2021maxlik}. These are the closest
prior estimators based on \textit{diffusion models}; each trains a network per distribution, which is the step
\model{} amortizes.

\paragraph{Foundation models and in-context inference.}
Amortized in-context inference is realized by TabPFN \citep{hollmann2023tabpfn}
for tabular prediction, and that transformers learn function classes in-context
is established broadly by \citet{garg2022what}; \model{} adapts this inference mechanism to information estimation. Closest in the mechanism,
\citet{smart2025incontext} study in-context denoising with one-layer transformers
and its connection to associative memory \citep{ramsauer2021hopfield}. The induced-latent context bottleneck follows
set-attention and latent-array designs \citep{lee2019set,jaegle2021perceiver}.

\paragraph{Amortized estimation and training corpora.}
The zero-shot claim depends on a broad synthetic training distribution.
We extend the dependence-diversity design of InfoAtlas \citep{hu2026infoatlas}, random copula mixtures with coupling-flow \citep{dinh2017realnvp} augmentation, in the spirit of scaling-law-driven pretraining \citep{kaplan2020scaling}.
InfoAtlas is the closest amortized estimator, and \model{} differs from it in three respects.
First, InfoAtlas trains a hypernetwork that outputs the weights of a separate variational estimator for each distribution.
\model{} keeps a single network and conditions it on the samples through attention, so no distribution-specific parameters are produced.
Second, the coordinate-shared architecture of \Cref{sec:method:arch} is applied at any joint width, including widths absent from the corpus, while the weights a hypernetwork emits have a fixed shape and bind the estimator to the joint widths it was trained on.
Third, \model{} estimates \gls{MI} through the velocity-difference identity of \Cref{eq:mi-velocity}, which is exact for the true fields and involves no variational bound.
Variational estimators output lower bounds, and a high-confidence lower bound above $\log N$ nats cannot be certified from $N$ samples \citep{mcallester2020formal}.

\section{Mutual Information as a Velocity-Difference Integral}
\label{app:mi-variants}

This appendix proves \Cref{eq:mi-velocity} in the notation of \Cref{sec:introduction,sec:mi-from-kl}. \Cref{thm:kl-velocity} expresses the KL divergence between two densities as a weighted time integral of the squared difference of their velocity fields, and \Cref{thm:mi-velocity} turns it into the joint-versus-conditional form the estimator uses. Results of the same kind exist in the I-MMSE relation of \citet{franzese2024minde,guo2005immse,wang2026relative}.

\subsection{Velocity and score}
\label{app:mi-variants:fields}

As in \Cref{sec:introduction}, for a density $p$ on $\bbR^d$, $Z_0\sim p$, and $\epsilon\sim\Normald{0}{I}$ independent of $Z_0$, the interpolant $Z_t=(1-t)Z_0+t\epsilon$ has density $p_t$, and $\vel_t(z)=\Ex_p[Z_0-\epsilon\mid Z_t=z]$ is the velocity field of $p$. When two densities are compared we mark the density as a superscript, $\vel^p_t$ and $\vel^q_t$. Throughout, densities are assumed smooth with finite second moments and, for $t>0$, with Gaussian tails, so that integrals can be differentiated under the sign and boundary terms of integrations by parts vanish; for $t>0$ every $p_t$ is a Gaussian convolution and has these properties.

\begin{lemma}[Velocity and score]
\label{lem:tweedie}
For $t\in(0,1)$ and every $z$,
\begin{equation}
  \nabla\log p_t(z)=\frac{(1-t)\,\vel_t(z)-z}{t}.
  \label{eq:vsd}
\end{equation}
\end{lemma}
\begin{proof}
Given $Z_0$, the noised point is Gaussian, $Z_t\sim\Normald{(1-t)Z_0}{t^2I}$, so $p_t(z)=\Ex_p[\varphi_t(z-(1-t)Z_0)]$ with $\varphi_t$ the density of $\Normald{0}{t^2I}$. Differentiating under the expectation and dividing by $p_t(z)$,
\begin{equation*}
  \nabla\log p_t(z)=-\frac{1}{t^2}\,\Ex_p\!\left[z-(1-t)Z_0\mid Z_t=z\right]=-\frac{1}{t}\,\Ex_p\!\left[\epsilon\mid Z_t=z\right],
\end{equation*}
since $z-(1-t)Z_0=t\epsilon$ when $Z_t=z$. Taking conditional expectations in $z=(1-t)Z_0+t\epsilon$ gives $z=(1-t)\Ex_p[Z_0\mid Z_t=z]+t\,\Ex_p[\epsilon\mid Z_t=z]$; subtracting $(1-t)$ times the definition of $\vel_t(z)$ yields $\Ex_p[\epsilon\mid Z_t=z]=z-(1-t)\vel_t(z)$, and the claim follows.
\end{proof}

At $t=0$ the interpolant is the clean sample and $\Ex_p[\epsilon\mid Z_0]=0$ by independence, so $\vel_0(z)=z$ for every density: all velocity fields agree at the boundary.

\begin{lemma}[Continuity equation]
\label{lem:continuity}
For $t\in(0,1)$, $\partial_t p_t=\nabla\!\cdot\!\left(p_t\,\vel_t\right)$.
\end{lemma}
\begin{proof}
Along each sample path $\frac{\mathrm{d}}{\mathrm{d}t}Z_t=\epsilon-Z_0$. For a smooth compactly supported test function $\phi$, the tower property and the definition of $\vel_t$ give
\begin{equation*}
  \frac{\mathrm{d}}{\mathrm{d}t}\,\Ex_p[\phi(Z_t)]=\Ex_p\!\left[\nabla\phi(Z_t)\cdot(\epsilon-Z_0)\right]
  =-\Ex_p\!\left[\nabla\phi(Z_t)\cdot\vel_t(Z_t)\right]
  =-\int\nabla\phi\cdot\vel_t\,p_t\,\d z .
\end{equation*}
The left side equals $\int\phi\,\partial_t p_t$, and integrating the right side by parts gives $\int\phi\,\nabla\!\cdot(p_t\vel_t)$.
\end{proof}

\subsection{KL divergence in velocity form}
\label{app:mi-variants:weight}

\begin{theorem}[KL divergence as a velocity-difference integral]
\label{thm:kl-velocity}
For two densities $p$ and $q$ on $\bbR^d$ with $\KL{p}{q}<\infty$,
\begin{equation}
  \KL{p}{q}=\int_0^1\frac{1-t}{t}\,
  \Ex_{z\sim p_t}\!\left[\Norm{\vel^p_t(z)-\vel^q_t(z)}^2\right]\d{t}.
  \label{eq:kl-velocity}
\end{equation}
\end{theorem}
\begin{proof}
Let $F(t)=\KL{p_t}{q_t}=\int p_t\log(p_t/q_t)$. At $t=1$ both interpolants equal $\epsilon$, so $p_1=q_1=\Normald{0}{I}$ and $F(1)=0$; at $t=0$, $F(0)=\KL{p}{q}$. Hence $\KL{p}{q}=-\int_0^1F'(t)\,\d t$, and it remains to compute $F'$. Since $\int\partial_t p_t=0$,
\begin{equation*}
  F'(t)=\int\partial_t p_t\,\log\frac{p_t}{q_t}-\int\frac{p_t}{q_t}\,\partial_t q_t .
\end{equation*}
Substituting \Cref{lem:continuity} for $p_t$ and $q_t$ and integrating by parts,
\begin{align*}
  \int\nabla\!\cdot(p_t\vel^p_t)\log\frac{p_t}{q_t}&=-\int p_t\,\vel^p_t\cdot\nabla\log\frac{p_t}{q_t},\\
  \int\frac{p_t}{q_t}\,\nabla\!\cdot(q_t\vel^q_t)&=-\int q_t\,\vel^q_t\cdot\nabla\frac{p_t}{q_t}=-\int p_t\,\vel^q_t\cdot\nabla\log\frac{p_t}{q_t},
\end{align*}
so that
\begin{equation*}
  F'(t)=-\Ex_{z\sim p_t}\!\left[\bigl(\vel^p_t(z)-\vel^q_t(z)\bigr)\cdot\bigl(\nabla\log p_t(z)-\nabla\log q_t(z)\bigr)\right].
\end{equation*}
By \Cref{lem:tweedie}, $\nabla\log p_t-\nabla\log q_t=\frac{1-t}{t}(\vel^p_t-\vel^q_t)$, since the term $-z/t$ is common to both. Therefore $F'(t)=-\frac{1-t}{t}\,\Ex_{z\sim p_t}\Norm{\vel^p_t(z)-\vel^q_t(z)}^2$, and integrating over $[0,1]$ gives \Cref{eq:kl-velocity}.
\end{proof}

The weight $(1-t)/t$ diverges as $t\to0$, and the integral is finite because both fields converge to the identity at the boundary. Substituting \Cref{eq:vsd} instead expresses the same integral as a score-difference integral with weight $t/(1-t)$; the velocity form is the one whose integrand is bounded at every $t$, which is why our model predicts velocities.

\subsection{Mutual information}
\label{app:mi-variants:eq21}

We use the notation of \Cref{sec:mi-from-kl}: $z_0=(x_0,y_0)\sim p_{XY}$, $\epsilon=(\epsilon_X,\epsilon_Y)$, $x_t=(1-t)x_0+t\epsilon_X$, $y_t=(1-t)y_0+t\epsilon_Y$, $z_t=(x_t,y_t)$, the joint field $\vel_t(z)$ with blocks $\vel_t(z)|_X$ and $\vel_t(z)|_Y$, and the conditional fields $\vel_t(x\mid y_0)$ and $\vel_t(y\mid x_0)$ of $p_{X\mid Y=y_0}$ and $p_{Y\mid X=x_0}$. In addition, $\vel^X_t(x)=\Ex[X_0-\epsilon_X\mid X_t=x]$ and $\vel^Y_t(y)$ denote the velocity fields of the marginals $p_X$ and $p_Y$. All expectations below are over $x_0,y_0,\epsilon$.

\begin{lemma}[Field of the product of marginals]
\label{lem:product}
The velocity field of $p_X\otimes p_Y$ is $(x,y)\mapsto\bigl(\vel^X_t(x),\vel^Y_t(y)\bigr)$.
\end{lemma}
\begin{proof}
Under $p_X\otimes p_Y$ the pairs $(X_0,\epsilon_X)$ and $(Y_0,\epsilon_Y)$ are independent, so conditioning $X_0-\epsilon_X$ on $(X_t,Y_t)$ is the same as conditioning it on $X_t$ alone, and symmetrically for $Y$.
\end{proof}

\begin{proposition}[Product and conditional forms]
\label{prop:mi-forms}
\begin{align}
  \MI(X;Y)&=\int_0^1\frac{1-t}{t}\,\Ex\!\left[\Norm{\vel_t(z_t)|_X-\vel^X_t(x_t)}^2+\Norm{\vel_t(z_t)|_Y-\vel^Y_t(y_t)}^2\right]\d t,
  \label{eq:mi-product}\\
  \MI(X;Y)&=\int_0^1\frac{1-t}{t}\,\Ex\!\left[\Norm{\vel_t(x_t\mid y_0)-\vel^X_t(x_t)}^2\right]\d t
  =\int_0^1\frac{1-t}{t}\,\Ex\!\left[\Norm{\vel_t(y_t\mid x_0)-\vel^Y_t(y_t)}^2\right]\d t.
  \label{eq:mi-conditional}
\end{align}
\end{proposition}
\begin{proof}
\Cref{eq:mi-product} is \Cref{thm:kl-velocity} with $p=p_{XY}$ and $q=p_X\otimes p_Y$, using \Cref{lem:product} and splitting the squared norm into its two blocks. For \Cref{eq:mi-conditional}, $\log\frac{p_{XY}(x,y)}{p_X(x)p_Y(y)}=\log\frac{p_{X\mid Y}(x\mid y)}{p_X(x)}$ gives $\MI(X;Y)=\Ex_{y_0}\KL{p_{X\mid Y=y_0}}{p_X}$; applying \Cref{thm:kl-velocity} to each pair $(p_{X\mid Y=y_0},p_X)$ and averaging over $y_0$ gives the first expression, and the second follows by symmetry.
\end{proof}

\begin{theorem}[Joint-versus-conditional form]
\label{thm:mi-velocity}
With one perturbation $\epsilon$ shared by the three fields, \Cref{eq:mi-velocity} holds:
\begin{equation*}
  \MI(X;Y)=\int_0^1\frac{1-t}{t}\,\Ex\!\left[
    \Norm{\vel_t(z_t)|_X-\vel_t(x_t\mid y_0)}^2
    +\Norm{\vel_t(z_t)|_Y-\vel_t(y_t\mid x_0)}^2
  \right]\d t .
\end{equation*}
\end{theorem}
\begin{proof}
Fix $t$ and consider the $X$ block. The three fields are conditional expectations of the same variable $X_0-\epsilon_X$ under three conditionings:
\begin{equation*}
  \vel^X_t(x_t)=\Ex[X_0-\epsilon_X\mid x_t],\quad
  \vel_t(z_t)|_X=\Ex[X_0-\epsilon_X\mid x_t,y_t],\quad
  \vel_t(x_t\mid y_0)=\Ex[X_0-\epsilon_X\mid x_t,y_0].
\end{equation*}
Since $\epsilon_Y$ is independent of $(x_0,\epsilon_X,y_0)$, the point $y_t=(1-t)y_0+t\epsilon_Y$ carries no information about $X_0-\epsilon_X$ beyond $(x_t,y_0)$, so by the tower property
\begin{equation*}
  \vel_t(z_t)|_X=\Ex\!\left[\vel_t(x_t\mid y_0)\mid x_t,y_t\right]
  \qquad\text{and}\qquad
  \vel^X_t(x_t)=\Ex\!\left[\vel_t(z_t)|_X\mid x_t\right].
\end{equation*}
That is, $\vel_t(z_t)|_X$ is the orthogonal projection of $\vel_t(x_t\mid y_0)$ onto the functions of $(x_t,y_t)$, and $\vel^X_t(x_t)$ is the projection of both onto the functions of $x_t$, so the two increments are orthogonal and
\begin{equation*}
  \Ex\Norm{\vel_t(x_t\mid y_0)-\vel^X_t(x_t)}^2
  =\Ex\Norm{\vel_t(x_t\mid y_0)-\vel_t(z_t)|_X}^2
  +\Ex\Norm{\vel_t(z_t)|_X-\vel^X_t(x_t)}^2 .
\end{equation*}
The same identity holds for the $Y$ block. Adding the two blocks, multiplying by $(1-t)/t$, and integrating, the left sides are the two expressions of \Cref{eq:mi-conditional}, each equal to $\MI(X;Y)$, and the last terms sum to the integrand of \Cref{eq:mi-product}, also equal to $\MI(X;Y)$. The integral of the middle terms is therefore $\MI(X;Y)+\MI(X;Y)-\MI(X;Y)=\MI(X;Y)$.
\end{proof}

\begin{remark}[Shared noise]
The projection argument requires the joint and the conditional field to be evaluated at the same noised block, with the same $\epsilon_X$ in $\vel_t(z_t)|_X$ and $\vel_t(x_t\mid y_0)$ and the same $\epsilon_Y$ in the $Y$ term. With independent perturbations the increments are no longer orthogonal and the integrand no longer averages to the mutual information.
\end{remark}

\subsection{Conditional variant: a discrete input as clean evidence}
\label{app:mi-variants:cond}

An equivalent form for estimating \gls{MI} writes $\MI(X;Y)=\Ex_{y}\KL{p_{X\mid Y=y}}{p_X}$ and compares the conditional velocity of $X$ given $Y$ to the marginal velocity of $X$. When both distributions are continuous we use the \Cref{eq:mi-velocity} form, because a single indicator-conditioned field yields all required partial velocities without a separate marginal model. When the input is discrete, the conditional form becomes a finite sum and we can use the estimator in \Cref{app:slemi-details}. Let $X$ take one of $m$ values $x_1,\dots,x_m$ with weights $p=(p_1,\dots,p_m)$ and let $Y$ be a continuous response. Then
\begin{equation}
  \MI(X;Y)=\sum_{i=1}^{m} p_i\,D_i,\qquad
  D_i=\KL{P_{Y\mid x_i}}{\bar P},\qquad
  \bar P=\sum_{j=1}^{m} p_j\,P_{Y\mid x_j},
  \label{eq:mi-cond}
\end{equation}
and each divergence is the velocity-form KL of \Cref{eq:kl-velocity} applied to the response alone,
\begin{equation}
  D_i=\int_0^1\frac{1-t}{t}\;
  \Ex_{y_0\sim P_{Y\mid x_i},\,\epsilon}
  \Norm{\vel_t(y_t|x_i)-\bar\vel_t(y_t)}^2\d{t},
  \qquad y_t=(1-t)y_0+t\epsilon.
  \label{eq:cond-kl}
\end{equation}
Only $Y$ is noised: every query holds the input block at the atom $x_i$ as clean evidence through the noising indicator (input coordinates clean, response coordinates noised), only the response block of the output is read, and no velocity field is needed for the discrete coordinate.

\paragraph{Two contexts from one field.}
The two fields in \Cref{eq:cond-kl} are the same model call, with the same indicator, at the same evaluation points, bound to two different contexts. Bound to a \emph{joint} context of clean $(x,y)$ rows, the evidence $x_i$ selects the conditional $P_{Y\mid x_i}$, and the call returns its velocity $\vel^{\,x_i}_t$. Bound to a \emph{shuffled} context, the same call returns $\bar\vel_t$. The shuffled context is built row by row from two independent draws from $p$: the first draw selects an input value and the row takes a response from the pool of that value, the second draw overwrites the input column. Input and response are therefore independent in the context, so the evidence carries no information about the response, and the response marginal of the context is $\bar P$ by construction, whatever $p$ is. Drawing the shuffled context from the same pooled rows as the joint context keeps part of the finite-context sampling noise common to the two fields. The joint context is stratified uniformly over the atoms, since the conditionals do not depend on $p$; the shuffled context follows $p$ which determines $\bar P$ when $p$ changes.

\section{Estimation algorithm}
\label{app:algorithm}

\Cref{alg:mi} lists the mutual information estimation procedure of \Cref{sec:mi-from-kl}: a disjoint context/evaluation split, one cached context encoding, three velocity queries per (point, time) pair with shared noise, and the weighted average of the block-wise velocity differences of \Cref{eq:estimator}.

\begin{algorithm}[h]
\caption{\model{} mutual information estimation}
\label{alg:mi}
\begin{algorithmic}[1]
\Require $N$ joint samples $z=(x,y)$; frozen \model{} $\vmodel$; time draws per point $n_t$; indicators $m_{XY}=(\indone_X,\indone_Y)$, $m_X=(\indone_X,\indzero_Y)$, $m_Y=(\indzero_X,\indone_Y)$
\State split the samples into a disjoint \emph{context} set $\ctx$ of size $n$ and \emph{evaluation} set $\cD_{\mathrm{eval}}$; fit a coordinate-wise copula map on $\ctx$ and apply it to both sets
\State encode $\ctx$ once and cache its keys and values \Comment{reused by every query below}
\For{each evaluation point $z_0=(x_0,y_0)\in\cD_{\mathrm{eval}}$ and each of $n_t$ draws $t\sim\cU[0,1]$}
  \State draw one $\epsilon\sim\cN(0,I)$; set $z_{\mathrm{full}}=(1-t)z_0+t\epsilon$ \Comment{shared by the three queries}
  \State $z_X\leftarrow m_X\odot z_{\mathrm{full}}+(1-m_X)\odot z_0$;\quad
         $z_Y\leftarrow m_Y\odot z_{\mathrm{full}}+(1-m_Y)\odot z_0$
  \State $v_{\mathrm{full}}\leftarrow\vmodel(z_{\mathrm{full}},t,m_{XY};\ctx)$ \Comment{three queries to the cached context}
  \State $v_X\leftarrow\vmodel(z_X,t,m_X;\ctx)$;\quad $v_Y\leftarrow\vmodel(z_Y,t,m_Y;\ctx)$
  \State $g\leftarrow \dfrac{1-t}{t}\Big[\Norm{(v_{\mathrm{full}}-v_X)|_X}^2+\Norm{(v_{\mathrm{full}}-v_Y)|_Y}^2\Big]$
\EndFor
\State \Return $\widehat\MI(X;Y)=\dfrac{1}{N_{\mathrm{MC}}}\sum g$ over the $N_{\mathrm{MC}}=|\cD_{\mathrm{eval}}|\,n_t$ pairs \Comment{\Cref{eq:estimator}}
\end{algorithmic}
\end{algorithm}

\section{\model{} Details}
\label{app:alice-details}

This section specifies the \model{} architecture summarized in \Cref{sec:method:arch}: the token layout and time conditioning, the relation graph and its attention rule, the induced context bottleneck, the boundary parameterization, and the model family. \Cref{fig:arch-detail} shows one forward pass through these components.

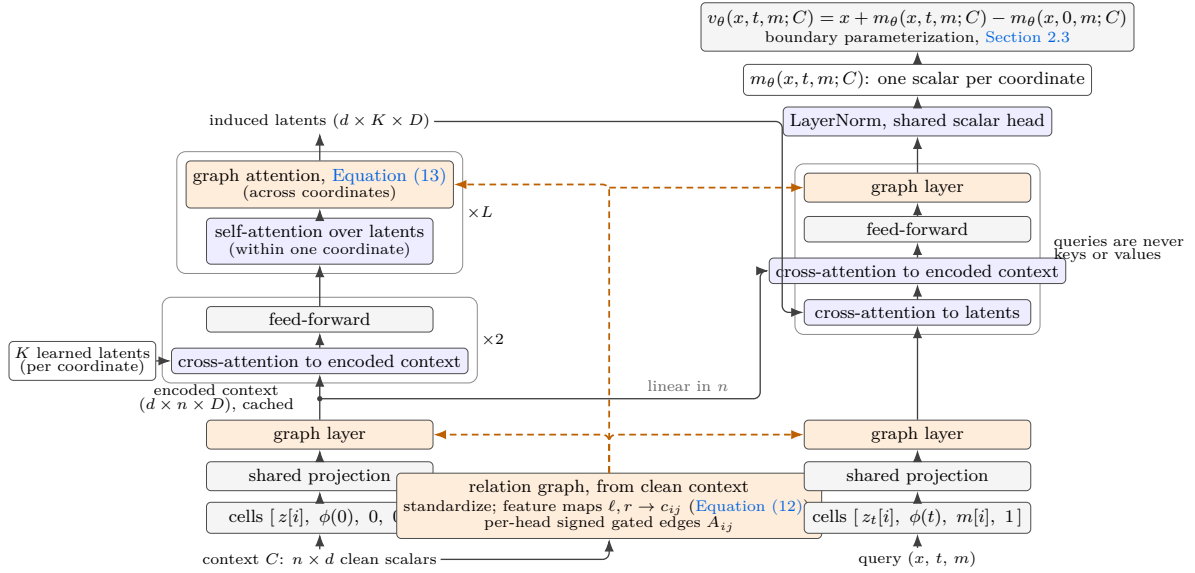
\begin{figure}[t]
  \centering
  \resizebox{0.94\linewidth}{!}{%
  \begin{tikzpicture}[
      font=\footnotesize,
      >={Latex[length=2mm]},
      box/.style={draw=black!70, rounded corners=2pt, inner sep=3pt, align=center, minimum width=36mm},
      embed/.style={box, fill=black!4},
      attn/.style={box, fill=blue!7},
      ffn/.style={box, fill=black!4},
      gph/.style={box, fill=orange!14},
      cont/.style={draw=black!45, rounded corners=4pt, inner sep=4pt},
      lab/.style={font=\scriptsize, inner sep=1.5pt, align=center},
      flow/.style={->, semithick, black!75},
      gflow/.style={->, densely dashed, thick, orange!75!black},
    ]
    \node[lab]   (C0) at (0,0)    {context $\ctx$: $n\times d$ clean scalars};
    \node[embed] (C1) at (0,0.66) {cells $[\,z[i],\ \phi(0),\ 0,\ 0\,]$};
    \node[embed] (C2) at (0,1.32) {shared projection};
    \node[gph]   (C3) at (0,1.98) {graph layer};
    \node[attn]  (R1) at (0,3.15) {cross-attention to encoded context};
    \node[ffn]   (R2) at (0,3.81) {feed-forward};
    \node[attn]  (B1) at (0,5.05) {self-attention over latents\\[-2pt]{\scriptsize (within one coordinate)}};
    \node[gph]   (B2) at (0,5.95) {graph attention, \Cref{eq:graph-attn}\\[-2pt]{\scriptsize (across coordinates)}};
    \node[lab]   (LT) at (0,6.95) {induced latents $(d\times K\times D)$};
    \begin{scope}[on background layer]
      \node[cont, fit=(R1)(R2)] (RC) {};
      \node[cont, fit=(B1)(B2)] (BC) {};
    \end{scope}
    \node[lab, anchor=west] at (RC.east) {$\times 2$};
    \node[lab, anchor=west] at (BC.east) {$\times L$};
    \draw[flow] (C0) -- (C1);
    \draw[flow] (C1) -- (C2);
    \draw[flow] (C2) -- (C3);
    \draw[flow] (C3) -- (R1);
    \draw[flow] (R1) -- (R2);
    \draw[flow] (R2) -- (B1);
    \draw[flow] (B1) -- (B2);
    \draw[flow] (B2) -- (LT);
    \fill[black!75] (0,2.55) circle (1.2pt);
    \node[lab, anchor=east] at (-0.35,2.55) {encoded context\\[-2pt]$(d\times n\times D)$, cached};
    \node[lab, draw=black!70, rounded corners=2pt, inner sep=2.5pt, anchor=east] (LQ) at (-2.6,3.15)
      {$K$ learned latents\\[-2pt](per coordinate)};
    \draw[flow] (LQ) -- (R1);
    \node[gph, minimum width=32mm] (G) at (4.6,0.85)
      {relation graph, from clean context\\[-1pt]
       {\scriptsize standardize; feature maps $\ell,r\to c_{ij}$ (\Cref{eq:relation});}\\[-2pt]
       {\scriptsize per-head signed gated edges $A_{ij}$}};
    \draw[flow, rounded corners=3pt] (C0.east) -| (G.south);
    \node[lab]   (Q0) at (9.5,0)    {query $(x,\,t,\,m)$};
    \node[embed] (Q1) at (9.5,0.66) {cells $[\,z_t[i],\ \phi(t),\ m[i],\ 1\,]$};
    \node[embed] (Q2) at (9.5,1.32) {shared projection};
    \node[gph]   (Q3) at (9.5,1.98) {graph layer};
    \node[attn]  (D1) at (9.5,3.92) {cross-attention to latents};
    \node[attn]  (D2) at (9.5,4.58) {cross-attention to encoded context};
    \node[ffn]   (D3) at (9.5,5.24) {feed-forward};
    \node[gph]   (D4) at (9.5,5.90) {graph layer};
    \begin{scope}[on background layer]
      \node[cont, fit=(D1)(D4)] (DC) {};
    \end{scope}
    \node[lab, anchor=west, align=left] at ($(DC.east)+(1.5mm,0)$) {queries are never\\[-2pt]keys or values};
    \node[attn]  (H1) at (9.5,6.95) {LayerNorm, shared scalar head};
    \node[box]   (H2) at (9.5,7.61) {$\headout(x,t,m;\ctx)$: one scalar per coordinate};
    \node[embed] (H3) at (9.5,8.45)
      {$\vmodel(x,t,m;\ctx)=x+\headout(x,t,m;\ctx)-\headout(x,0,m;\ctx)$\\[-1pt]
       {\scriptsize boundary parameterization, \Cref{sec:method:arch}}};
    \draw[flow] (Q0) -- (Q1);
    \draw[flow] (Q1) -- (Q2);
    \draw[flow] (Q2) -- (Q3);
    \draw[flow] (Q3) -- (D1);
    \draw[flow] (D1) -- (D2);
    \draw[flow] (D2) -- (D3);
    \draw[flow] (D3) -- (D4);
    \draw[flow] (D4) -- (H1);
    \draw[flow] (H1) -- (H2);
    \draw[flow] (H2) -- (H3);
    \draw[gflow, rounded corners=3pt] (G.north) |- (C3.east);
    \draw[gflow, rounded corners=3pt] (G.north) |- (Q3.west);
    \draw[gflow, rounded corners=3pt] (G.north) -- (4.6,5.95) -- (B2.east);
    \draw[gflow, rounded corners=3pt] (4.6,5.90) -- (D4.west);
    \draw[flow, rounded corners=3pt] (0,2.55) -- (7.0,2.55) -- (7.0,4.58) -- (D2.west);
    \node[lab, black!60] at (5.85,2.75) {linear in $n$};
    \draw[flow, rounded corners=3pt] (LT.east) -- (7.35,6.95) -- (7.35,3.92) -- (D1.west);
  \end{tikzpicture}%
  }
  \caption{One forward pass of \model{}. Clean context values determine the
  relation graph, whose gated edges $A_{ij}$ condition every graph layer
  (dashed). Context and query cells pass through the same shared projection
  and input graph layer; $K$ learned latents per coordinate read the encoded
  context twice through cross-attention, and $L$ blocks alternate
  self-attention among one coordinate's latents with graph attention across
  coordinates. A query decodes in one pass, cross-attending to the latents for
  the global summary and to the encoded context for local detail, and a shared
  scalar head produces one output per coordinate; two head evaluations, at
  times $t$ and $0$, form the velocity through the boundary parameterization.
  The context side, left on the figure, is encoded once per context
  and cached; every interaction with the $n$ context samples is linear in $n$.}
  \label{fig:arch-detail}
\end{figure}

\subsection{Architecture}
\label{app:alice-arch}

\noindent \textbf{Conditioning and caching.}
Queries interact with the network only through cross-attention: they are not used as keys or values. Three properties follow. 1) Context representations never depend on queries, so a context is encoded once per distribution, cached, and reused by every velocity evaluation. 2) Each query's output is a function of $(z_t,t,m;\ctx)$ alone, and does not depend on other queries that might be added or permuted: hence, many queries are scored in one pass and the boundary parameterization of \Cref{sec:method:arch} is exact. 3) The context-validity mask excludes padded context samples from every attention over the context, so a padded context is equivalent to a physically truncated one and the same weights serve any context length. Every token is embedded by a shared projection, but no channel of the embedding encodes the index of the sample or of the coordinate a token comes from: a context is an exchangeable set of samples and a sample is an unordered set of coordinates, so the network carries no positional information along either axis. The noising indicator $m$ plays no role in attention; it is used only as an input channel of the query tokens.

\noindent \textbf{Per-coordinate tokens and time conditioning.}
The joint width $d$ is not fixed a priori. Every scalar coordinate of every sample becomes one token $[\text{value},\,\phi(t),\,m[i],\,\text{type}]$, lifted to width $D$ by one shared projection; context and query tokens then pass through one shared input graph layer before any processing (\Cref{fig:arch-detail}, bottom). All parameters live in coordinate-shared maps: the token projection, the attention and feed-forward weights, the latent bank, and a scalar output head. Changing $d$ therefore changes only the number of tokens per sample, and the same model weights can be used at any joint width.
We use sinusoidal time features $\phi(t)=[t,\sin(2^k t),\cos(2^k t)]_{k<F}$, and the noising-indicator entry $m[i]\in\{0,1\}$ encodes partial observation: a $1$ marks a coordinate that follows the interpolant at time $t$ and is to be predicted, a $0$ indicates a coordinate held clean at its observed value as evidence. Context tokens carry zero time features and an all-clean indicator. The indicator channel lets a single model produce the three partially noised velocities of \Cref{eq:mi-velocity}; the boundary parameterization evaluates the network at times $t$ and $0$ with the same indicator $m$.

\noindent \textbf{The relation graph.}
The cross-coordinate mechanism must represent which coordinates depend on which, with what sign, possibly through non-monotone relations, all varying from distribution to distribution. A learned $d\times d$ interaction parameter would be tied to one dimension and one dependence pattern, and softmax attention across coordinates produces weights that are dense, nonnegative, and sum to one, so independent coordinates would still exchange information. \model{} instead measures the dependence structure from the clean context and uses the result as a weighted graph over coordinates, recomputed once per forward pass whenever the context changes; the resulting edges condition every graph layer in \Cref{fig:arch-detail} (dashed). \Cref{fig:app-relation-graph} summarizes the construction.

\begin{figure}[t]
  \centering
  \resizebox{0.55\linewidth}{!}{%
    \definecolor{rgRed}{HTML}{ED1C24}
\definecolor{rgRedFill}{HTML}{EF5B61}
\definecolor{rgOrange}{HTML}{FF7F00}
\definecolor{rgOrangeFill}{HTML}{FFA64D}
\definecolor{rgGreen}{HTML}{4DAF4A}
\definecolor{rgGreenFill}{HTML}{86C982}
\definecolor{rgBlue}{HTML}{377EB8}
\definecolor{rgBlueFill}{HTML}{73A4CC}
\definecolor{rgPurple}{HTML}{984EA3}
\definecolor{rgPurpleFill}{HTML}{B17CBB}
\definecolor{rgGray}{HTML}{7F7F7F}

\begin{tikzpicture}[
    x=1cm,
    y=1cm,
    font=\sffamily\scriptsize,
    text=black!90,
    >={Latex[length=2mm,width=1.6mm]},
    panel/.style={draw=rgGray, line width=0.9pt, rounded corners=7pt, fill=white},
    module/.style={draw=rgGray, line width=0.8pt, rounded corners=4pt, fill=white},
    flow/.style={->, line width=0.8pt, black!80},
    wire/.style={line width=0.55pt, black!70},
    label/.style={inner sep=1pt, align=center},
  ]
  \node[panel, minimum width=6.15cm, minimum height=0.68cm] at (3.10,9.84) {};
  \fill[rgRedFill, draw=rgRed, line width=0.75pt] (0.28,9.67) rectangle (0.66,10.01);
  \node[anchor=west] at (0.76,9.84) {coordinate};
  \fill[rgRedFill, draw=rgRed, line width=0.65pt] (2.40,9.68) rectangle (2.63,10.00);
  \fill[rgGreenFill, draw=rgGreen, line width=0.65pt] (2.63,9.68) rectangle (2.86,10.00);
  \node[anchor=west] at (2.96,9.84) {pair $j\!\to\!i$};
  \draw[flow] (4.62,9.84) -- (5.04,9.84);
  \node[anchor=west] at (5.14,9.84) {flow};

  \node[panel, minimum width=6.15cm, minimum height=2.95cm] at (3.10,7.73) {};
  \node[font=\sffamily\small\bfseries] at (3.10,8.99) {Context-derived relation graph};

  \node[label] at (1.18,8.48) {clean context $\ctx$\\paired sample rows};
  \foreach \yy in {8.05,7.75,7.45,7.15} {
    \fill[rgRedFill, draw=rgRed, line width=0.45pt] (0.43,\yy-0.10) rectangle (0.80,\yy+0.10);
    \fill[rgOrangeFill, draw=rgOrange, line width=0.45pt] (0.80,\yy-0.10) rectangle (1.17,\yy+0.10);
    \fill[rgGreenFill, draw=rgGreen, line width=0.45pt] (1.17,\yy-0.10) rectangle (1.54,\yy+0.10);
    \fill[rgBlueFill, draw=rgBlue, line width=0.45pt] (1.54,\yy-0.10) rectangle (1.91,\yy+0.10);
  }
  \node at (1.17,6.84) {$\vdots$};
  \draw[flow] (2.08,7.60) -- (2.48,7.60);

  \node[module, minimum width=1.52cm, minimum height=2.08cm] (features) at (3.28,7.60) {};
  \node[label] at (3.28,8.35) {shared maps\\$\ell,r$};
  \fill[rgRedFill, draw=rgRed, line width=0.55pt] (2.80,7.86) rectangle (3.11,8.13);
  \fill[white, opacity=0.35] (2.87,7.92) rectangle (3.04,8.07);
  \fill[rgGreenFill, draw=rgGreen, line width=0.55pt] (3.15,7.86) rectangle (3.46,8.13);
  \fill[white, opacity=0.35] (3.22,7.92) rectangle (3.39,8.07);
  \draw[flow] (3.13,7.78) -- (3.13,7.55);
  \fill[rgRedFill, draw=rgRed, line width=0.55pt] (2.84,7.26) rectangle (3.13,7.51);
  \fill[white, opacity=0.35] (2.90,7.31) rectangle (3.07,7.46);
  \fill[rgGreenFill, draw=rgGreen, line width=0.55pt] (3.13,7.26) rectangle (3.42,7.51);
  \fill[white, opacity=0.35] (3.19,7.31) rectangle (3.36,7.46);
  \node[label, font=\sffamily\tiny] at (3.28,6.91) {row average};
  \draw[flow] (4.08,7.60) -- (4.45,7.60);

  \node[label] at (5.30,8.48) {relation graph\\$A^{(h)}$};
  \coordinate (g1) at (5.16,8.02);
  \coordinate (g2) at (5.72,7.69);
  \coordinate (g3) at (5.56,7.05);
  \coordinate (g4) at (4.91,7.08);
  \draw[rgPurple, line width=1.5pt] (g1) -- (g2);
  \draw[rgGreen, line width=2.1pt] (g2) -- (g3);
  \draw[rgBlue, line width=1.0pt] (g3) -- (g4);
  \draw[rgOrange, line width=1.7pt] (g4) -- (g1);
  \draw[black!35, line width=0.55pt] (g1) -- (g3);
  \draw[black!35, line width=0.55pt] (g2) -- (g4);
  \fill[rgRedFill, draw=rgRed, line width=0.8pt] (g1) circle (0.15);
  \fill[rgOrangeFill, draw=rgOrange, line width=0.8pt] (g2) circle (0.15);
  \fill[rgGreenFill, draw=rgGreen, line width=0.8pt] (g3) circle (0.15);
  \fill[rgBlueFill, draw=rgBlue, line width=0.8pt] (g4) circle (0.15);

  \begin{scope}[yshift=0.6cm]
    \node[panel, minimum width=6.15cm, minimum height=3.68cm] at (3.10,3.19) {};
    \node[font=\sffamily\small\bfseries] at (3.10,4.67) {Graph attention};

    \node[label] at (0.91,4.19) {$\mathbf{r}_j^{(h)}$};
    \fill[rgRedFill, draw=rgRed, line width=0.70pt] (0.43,3.71) rectangle (1.39,4.03);
    \fill[rgOrangeFill, draw=rgOrange, line width=0.70pt] (0.43,3.26) rectangle (1.39,3.58);
    \fill[rgGreenFill, draw=rgGreen, line width=0.70pt] (0.43,2.81) rectangle (1.39,3.13);
    \fill[rgBlueFill, draw=rgBlue, line width=0.70pt] (0.43,2.36) rectangle (1.39,2.68);
    \fill[rgPurpleFill, draw=rgPurple, line width=0.70pt] (0.43,1.91) rectangle (1.39,2.23);

    \node[label] at (2.46,4.19) {edges $A_{ij}^{(h)}$};
    \foreach \yy/\cc/\ff in {3.87/rgRed/rgRedFill,3.42/rgOrange/rgOrangeFill,2.97/rgGreen/rgGreenFill,2.52/rgBlue/rgBlueFill,2.07/rgPurple/rgPurpleFill} {
      \fill[\ff, draw=\cc, line width=0.55pt] (2.05,\yy-0.11) rectangle (2.34,\yy+0.11);
      \fill[rgGreenFill, draw=rgGreen, line width=0.55pt] (2.38,\yy-0.11) rectangle (2.67,\yy+0.11);
    }
    \foreach \yy in {3.87,3.42,2.97,2.52,2.07} {
      \draw[wire] (1.45,\yy) -- (1.96,\yy);
    }
    \node[label, font=\sffamily\tiny] at (2.46,1.67) {green target $i$};
    \draw[flow] (2.82,2.97) -- (3.15,2.97);

    \node[module, minimum width=1.32cm, minimum height=2.64cm] (aggregate) at (3.87,2.97) {};
    \node[label, font=\sffamily\tiny] at (3.87,4.00) {weighted\\messages};
    \foreach \yy/\cc/\ff in {3.72/rgRed/rgRedFill,3.43/rgOrange/rgOrangeFill,3.14/rgBlue/rgBlueFill} {
      \fill[\ff, draw=\cc, line width=0.50pt] (3.48,\yy-0.10) rectangle (3.75,\yy+0.10);
      \fill[white, opacity=0.35] (3.54,\yy-0.06) rectangle (3.69,\yy+0.06);
      \fill[rgGreenFill, draw=rgGreen, line width=0.50pt] (3.78,\yy-0.10) rectangle (4.05,\yy+0.10);
      \fill[white, opacity=0.35] (3.84,\yy-0.06) rectangle (3.99,\yy+0.06);
    }
    \draw[flow] (3.87,3.00) -- (3.87,2.80);
    \node[font=\normalsize] at (3.87,2.57) {$\sum$};
    \draw[flow] (3.87,2.32) -- (3.87,2.16);
    \node[label, font=\sffamily\tiny] at (3.87,1.91) {normalize};
    \draw[flow] (4.57,2.97) -- (4.88,2.97);

    \node[label] at (5.46,4.15) {message $u_i^{(h)}$};
    \fill[rgGreenFill, draw=rgGreen, line width=0.85pt] (4.97,2.79) rectangle (5.95,3.15);
    \fill[white, opacity=0.35] (5.08,2.86) rectangle (5.84,3.08);
    \node at (5.46,2.97) {$u_i^{(h)}$};
    \draw[flow] (5.46,2.74) -- (5.46,2.55);
    \fill[rgGreenFill, draw=rgGreen, line width=0.9pt] (4.82,2.12) rectangle (6.10,2.52);
    \fill[white, opacity=0.30] (4.96,2.20) rectangle (5.96,2.44);
    \node[label, font=\sffamily\tiny] at (5.46,1.73) {residual update};
  \end{scope}

  \draw[flow, rounded corners=3pt]
  (5.56,6.88) -- (5.56,5.96) -- (3.10,5.96) -- (3.10,5.59);
  \node[label, fill=white, font=\tiny] at (3.6,6.02) {$A^{(h)}$};
\end{tikzpicture}%
  }
  \caption{\model{} relation-graph construction and coordinate mixing.
    Node and block colors identify coordinates, and each two-color block denotes a source-target pair $j\to i$.
    Top: shared feature maps $\ell,r$ aggregate paired clean-context rows into signed, gated relation weights $A_{ij}^{(h)}$.
    Bottom: for the green target $i$, each source representation $\mathbf{r}_j^{(h)}$ is weighted by its incoming edge, and the center module sums and normalizes these contributions to produce $u_i^{(h)}$ for the residual update.}
  \label{fig:app-relation-graph}
\end{figure}
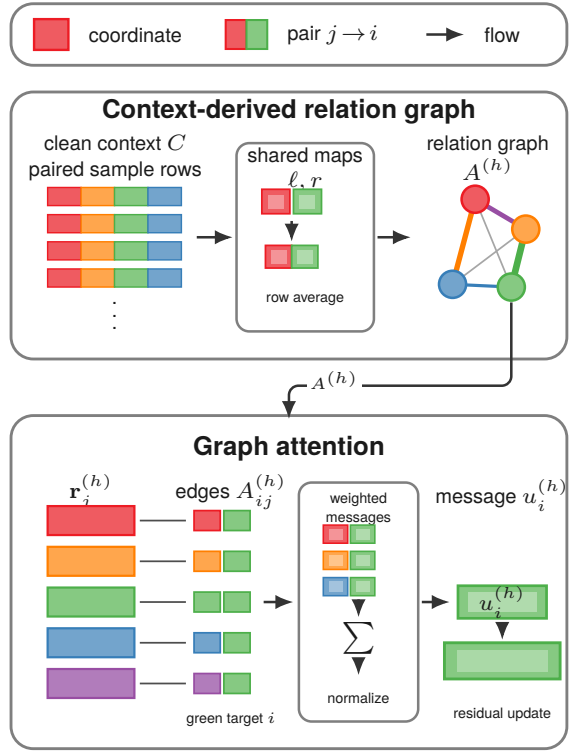

To describe pairwise dependence, we measure covariance between learned nonlinear features, a principle also used in kernel dependence measures \citep{gretton2005measuring}.
Let $\hat z^{(k)}[i]$ be coordinate $i$ of clean context sample $k$, standardized over the context. Two learned maps $\ell,r:\bbR\to\bbR^{R}$, each shared across coordinates and context samples, take this single scalar as input and output $R$ nonlinear features. Subtracting each feature's context mean gives $\tilde\ell_i^{(k)}=\ell(\hat z^{(k)}[i])-\frac1n\sum_{k'=1}^{n}\ell(\hat z^{(k')}[i])$, and likewise $\tilde r_i^{(k)}$. The descriptor of coordinates $i$ and $j$ is
\begin{equation}
  c_{ij}=\frac{1}{2n}\sum_{k=1}^{n}\Bigl(\tilde\ell_i^{(k)}\odot\tilde r_j^{(k)}
        +\tilde r_i^{(k)}\odot\tilde\ell_j^{(k)}\Bigr)\in\bbR^{R}.
  \label{eq:relation}
\end{equation}
Each component of $c_{ij}$ is an average of two empirical feature covariances, with the common index $k$ preserving the joint observations and symmetrization giving $c_{ij}=c_{ji}$. With identity feature maps, \Cref{eq:relation} reduces to the empirical correlation of the standardized coordinates; learned maps expose dependence, such as $z[j]\approx z[i]^2$, that correlation misses. Centering makes the population descriptor vanish under independence, since the expected product of centered features then factorizes, although finite contexts introduce sampling fluctuations. Permuting all context samples together leaves the descriptor invariant, while permuting one coordinate's values independently changes the empirical joint and therefore the graph.

The descriptor is shared by all attention heads. For graph head $h$, learned projection vectors $w_s^{(h)},w_g^{(h)}\in\bbR^R$ and scalar biases $b_s^{(h)},b_g^{(h)}$ convert it into a signed, gated edge. With $\sigma$ the sigmoid and $\mathbf{r}_j^{(h)}$ the projected representation of coordinate $j$ at the same context, latent, or query position, the edge and the aggregated message are
\begin{equation}
  A_{ij}^{(h)}=\tanh\!\left((w_s^{(h)})^\top c_{ij}+b_s^{(h)}\right)
  \sigma\!\left((w_g^{(h)})^\top c_{ij}+b_g^{(h)}\right),
  \qquad
  u_i^{(h)}=\frac{\sum_{j\ne i}A_{ij}^{(h)}\mathbf{r}_j^{(h)}}
  {\max\!\left(1,\sum_{j\ne i}|A_{ij}^{(h)}|\right)},
  \label{eq:graph-attn}
\end{equation}
followed by an output projection and the usual residual and feed-forward updates. The signed factor allows additive or subtractive contributions, while the gate controls their magnitude. Normalizing by absolute edge mass bounds the aggregate contribution, and the lower bound of one preserves small updates when all edges are weak. There are no self-edges. Gates are initialized nearly closed, so training starts from an independence prior and opens edges only where the context provides evidence of dependence; exact disconnection under independence is not enforced. The feature maps and edge projections are learned through the velocity objective, so the edges represent pairwise associations useful for prediction without imposing a conditional-independence interpretation. Sharing these maps across coordinates keeps the parameter count independent of $d$ and makes the graph equivariant to coordinate permutations.

\noindent \textbf{Induced context bottleneck.}
The model compresses each coordinate's $n$ context tokens into $K$ induced latents and runs its depth on the latents at a cost that is independent of $n$ (\Cref{fig:arch-detail}, left tower). In other words, a shared bank of $K$ learned vectors reads the encoded context through two cross-attentions, each linear in $n$, and the deep blocks then alternate self-attention among one coordinate's latents, refining that coordinate's summary of the context, with graph attention from \Cref{eq:graph-attn}, sharing the summaries across coordinates. A query decodes through two complementary mechanisms (\Cref{fig:arch-detail}, right tower): cross-attention to the latents provides the global summary, and a final cross-attention to the encoded context, linear in $n$, retrieves the local detail near the query that a $K$-vector summary cannot retain; a last graph layer and a shared scalar head produce one output per coordinate.
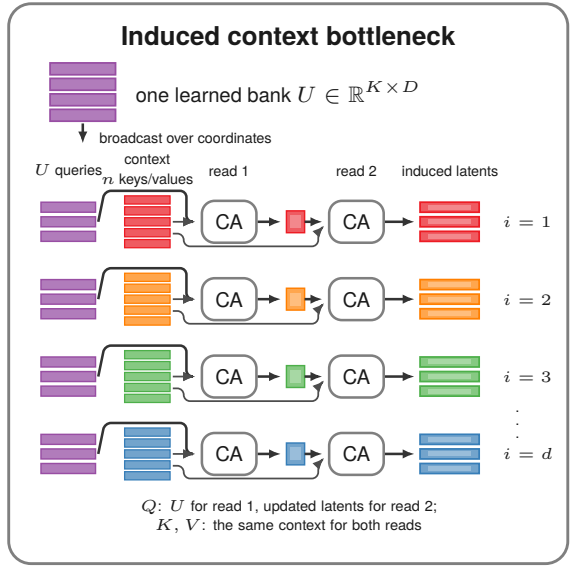
\begin{figure}[t]
  \centering
  \resizebox{0.55\linewidth}{!}{%
    \definecolor{rgRed}{HTML}{ED1C24}
\definecolor{rgRedFill}{HTML}{EF5B61}
\definecolor{rgOrange}{HTML}{FF7F00}
\definecolor{rgOrangeFill}{HTML}{FFA64D}
\definecolor{rgGreen}{HTML}{4DAF4A}
\definecolor{rgGreenFill}{HTML}{86C982}
\definecolor{rgBlue}{HTML}{377EB8}
\definecolor{rgBlueFill}{HTML}{73A4CC}
\definecolor{rgPurple}{HTML}{984EA3}
\definecolor{rgPurpleFill}{HTML}{B17CBB}
\definecolor{rgGray}{HTML}{7F7F7F}

\begin{tikzpicture}[
    x=1cm,
    y=1cm,
    font=\sffamily\scriptsize,
    text=black!90,
    >={Latex[length=1.35mm,width=1.05mm]},
    panel/.style={draw=rgGray, line width=0.9pt, rounded corners=7pt, fill=white},
    module/.style={draw=rgGray, line width=0.8pt, rounded corners=4pt, fill=white},
    flow/.style={->, line width=0.8pt, black!80, shorten >=1.6pt},
    wire/.style={->, line width=0.55pt, black!70, shorten >=1.6pt},
    label/.style={inner sep=1pt, align=center},
  ]
  \node[panel, minimum width=6.15cm, minimum height=6.15cm] at (3.10,3.10) {};
  \node[font=\sffamily\small\bfseries] at (3.10,5.83) {Induced context bottleneck};

  \node[label, anchor=west] at (1.38,5.20) {one learned bank $U\in\bbR^{K\times D}$};
  \foreach \yy in {4.95,5.12,5.29,5.46} {
    \fill[rgPurpleFill, draw=rgPurple, line width=0.55pt]
      (0.48,\yy-0.065) rectangle (1.20,\yy+0.065);
  }
  \draw[flow] (0.84,4.86) -- (0.84,4.58);
  \node[label, anchor=west, font=\sffamily\tiny] at (0.98,4.70) {broadcast over coordinates};

  \node[label, font=\sffamily\tiny] at (0.68,4.34) {$U$ queries};
  \node[label, font=\sffamily\tiny] at (1.55,4.34) {context\\$n$ keys/values};
  \node[label, font=\sffamily\tiny] at (2.45,4.34) {read 1};
  \node[label, font=\sffamily\tiny] at (3.85,4.34) {read 2};
  \node[label, font=\sffamily\tiny] at (4.88,4.34) {induced latents};

  \foreach \yy/\edge/\fillc/\idx in {
      3.78/rgRed/rgRedFill/1,
      2.93/rgOrange/rgOrangeFill/2,
      2.08/rgGreen/rgGreenFill/3,
      1.23/rgBlue/rgBlueFill/d} {
    \foreach \dy in {-0.16,0,0.16} {
      \fill[rgPurpleFill, draw=rgPurple, line width=0.45pt]
        (0.38,\yy+\dy-0.055) rectangle (0.98,\yy+\dy+0.055);
    }
    \foreach \dy in {-0.24,-0.12,0,0.12,0.24} {
      \fill[\fillc, draw=\edge, line width=0.40pt]
        (1.30,\yy+\dy-0.045) rectangle (1.80,\yy+\dy+0.045);
    }
    \node[module, minimum width=0.60cm, minimum height=0.46cm] (a\idx) at (2.45,\yy) {CA};
    \node[module, minimum width=0.60cm, minimum height=0.46cm] (b\idx) at (3.85,\yy) {CA};
    \draw[flow, rounded corners=2pt]
      (1.01,\yy) -- (1.08,\yy+0.34) -- (1.98,\yy+0.34) -- (1.98,\yy) -- (a\idx.west);
    \draw[wire] (1.83,\yy) -- (a\idx.west);
    \fill[\fillc, draw=\edge, line width=0.45pt]
      (3.08,\yy-0.12) rectangle (3.28,\yy+0.12);
    \fill[white, opacity=0.28]
      (3.12,\yy-0.07) rectangle (3.24,\yy+0.07);
    \draw[flow] (a\idx.east) -- (3.08,\yy);
    \draw[flow] (3.28,\yy) -- (b\idx.west);
    \draw[wire, rounded corners=2pt]
      (1.83,\yy-0.13) -- (1.95,\yy-0.13) -- (1.95,\yy-0.27) -- (3.43,\yy-0.27) -- (3.43,\yy-0.11) -- (b\idx.west);
    \foreach \dy in {-0.16,0,0.16} {
      \fill[\fillc, draw=\edge, line width=0.55pt]
        (4.55,\yy+\dy-0.055) rectangle (5.20,\yy+\dy+0.055);
      \fill[white, opacity=0.28]
        (4.64,\yy+\dy-0.027) rectangle (5.11,\yy+\dy+0.027);
    }
    \draw[flow] (b\idx.east) -- (4.55,\yy);
    \node[label, anchor=west, font=\sffamily\tiny] at (5.42,\yy) {$i=\idx$};
  }

  \node[label, font=\sffamily\tiny] at (5.62,1.66) {$\vdots$};
  \node[label, font=\sffamily\tiny, text width=5.35cm] at (3.10,0.54)
    {$Q$: $U$ for read 1, updated latents for read 2;\\$K,V$: the same context for both reads};
\end{tikzpicture}%
  }
  \caption{Induced context bottleneck.
    The learned bank $U$ is broadcast across coordinates.
    For each coordinate, the first cross-attention layer uses $U$ as queries and that coordinate's context tokens as keys and values; the second uses the updated latents as queries and the same context tokens as keys and values, producing $K$ context-specific latent vectors.}
  \label{fig:app-induced-latents}
\end{figure}
The remaining cost is the $d\times d$ relation graph, which is favorable in the long-context, moderate-dimension regime of \gls{MI} estimation.

\noindent \textbf{Boundary parameterization.}
The estimator multiplies squared velocity differences by $(1-t)/t$, which diverges as $t\to0$.
Since squared differences cannot be negative, any violation of the boundary condition $\vel_0(z)=z$ stated in \Cref{sec:introduction} becomes systematic positive bias where the weight is largest.
\model{} satisfies this condition by construction by adopting the parametrization described in \cite{wang2026relative,hu2025improving}.

\noindent \textbf{Model family.}
We instantiate \model{} at a range of sizes that share the number of induced latents $K$, the relation-feature width $R$, and the time-feature resolution, so that model size affects only the backbone capacity, without changing the context bottleneck or the graph statistic. Parameter counts are independent of the joint width and the context length, and a trained checkpoint is exported as a self-contained model (weights, configuration, and source), usable at any joint width without modification.

The Small and Base configurations are listed in \Cref{tab:alice-sizes}.
\model{} Base is the checkpoint reported in \Cref{sec:exp:czyz}, and \Cref{app:czyz} compares the Small and Base checkpoints on the benchmark.

\begin{table}[h]
  \centering\small
  \caption{The two \model{} variants, with model configuration values and parameter counts for each preset.}
  \label{tab:alice-sizes}
  \begin{tabular}{@{}lcc@{}}
    \toprule
    Config name & \model{} Small & \model{} Base \\
    \midrule
    Model hidden dimension & $384$ & $768$ \\
    Layers & $5$ & $6$ \\
    Attention heads & $8$ & $12$ \\
    Feed-forward dimension & $1536$ & $3072$ \\
    Time frequencies & $16$ & $16$ \\
    Dropout & $0.1$ & $0.1$ \\
    Induced latents & $128$ & $128$ \\
    Relation features & $16$ & $16$ \\
    Parameters & $19{,}296{,}113$ & $85{,}200{,}633$ \\
    \bottomrule
  \end{tabular}
\end{table}

\section{Training and Implementation Details}
\label{app:training}

This section presents the reference implementation of \model{}.

\subsection{The training corpus}
\label{app:training:corpus}

A corpus episode is one synthetic joint distribution over $z=(x,y)\in\bbR^d$, generated from a seed, with the fixed split $s=\lfloor d/2\rfloor$: coordinates $[0,s)$ are $X$ and $[s,d)$ are $Y$. Each episode is stored as a clean point pool of $2176$ samples in single precision, together with a few scalar metadata fields; normalization, noising, indicator sampling, and targets are computed at train time.

\paragraph{Composition.} The corpus covers the joint widths $\{2,3,4,5,6,8,10,12,16,20,25,32,50,100\}$ with $70{,}000$ episodes per width, drawn from four families: copula mixtures, latent warps, manifolds, and nonparametric regressions, with probabilities $0.30$, $0.25$, $0.25$, and $0.20$. A further copula-only share brings the copula fraction of the whole corpus to about $0.40$, and part of the corpus enables the two geometric modifications described below, same-sign factor covariances and the plane-rotation warp.

\paragraph{Copula mixtures.} Between $1$ and $60$ Gaussian or Student-t components with random weights and means. Each component draws a low-rank covariance $\Sigma=WW^\top+D$ of random rank, converted to a correlation and rescaled per coordinate; with probability $0.3$ it instead draws a sparse correlation with a few disjoint $X_i\!\leftrightarrow\!Y_i$ pairs whose strength is coherent within an episode; and with probability $\tfrac12$ the $X\!\leftrightarrow\!Y$ cross-block of every component is scaled down, to zero half the time, which produces weakly dependent and independent joints. Most sampled pools are then passed through an additive-coupling bijection \citep{dinh2017realnvp}, either within each block, which preserves $I(X;Y)$, or across a random coordinate partition, which leaves it unknown.

\paragraph{Latent warps.} A mixture of anisotropic Gaussians whose means lie along a random curve is standardized and pushed through a few random layers, each an additive coupling shift, an elementwise sinusoidal fold, or a rotation. The fold is non-injective, so $I(X;Y)$ is unknown by construction.

\paragraph{Nonparametric regressions.} The input is Gaussian, or a two-component mixture, and the response is a random Fourier-feature function of the input plus Gaussian noise of random scale, which provides a controlled noise floor and a smooth nonlinear conditional mean.

\paragraph{Manifolds.} The pool lies on a low-dimensional curved support, a curve or a surface winding around the origin, thickened by transverse Gaussian noise of random scale and rotated at random. At small noise the support is near-singular, which is the regime where a velocity field must resolve a thin set.

\paragraph{Plane-rotation warp.} An MI-preserving diffeomorphism rotates randomly chosen coordinate planes of a block by an angle that grows with the block norm, occasionally followed by a monotone radial stretch. Each rotation preserves the block norm, so $I(X;Y)$ is unchanged, and the warp acts on a point cloud, so it applies to every family.

\paragraph{Same-sign factor covariances.} The low-rank draw above has sign-symmetric loadings, so joints in which every coordinate pair is positively correlated, a common structure in measured data with a shared latent factor, have vanishing probability under it. Part of the corpus therefore draws equicorrelated or positive low-rank covariances instead.

\subsection{Batch construction}
\label{app:training:batch}

For each batch, the procedure (i) samples one context length shared by all its distributions (for variable-context training), (ii) samples disjoint context/query samples from each pool, and (iii) applies Gaussian-copula softrank normalization: each marginal is mapped to $\Normald{0}{1}$ through the empirical CDF \emph{fit on
the context} and applied out-of-sample to the query. Batches are dimension-homogeneous: each batch is drawn from a single joint width. Variable context length is realized either by truncating to the sampled length or by hiding context samples behind the context-validity mask; the two are equivalent (see also \Cref{sec:method:arch}).

\subsection{Objective, noising indicators, optimizer}
\label{app:training:opt}

The loss is the masked velocity MSE defined in \Cref{eq:objective}, supervised only on noised coordinates. Each query draws its own time $t\sim\cU[0,1]$ and noise $\epsilon\sim\Normald{0}{I}$, and the target $z_0-\epsilon$ is available exactly because the training loop draws $\epsilon$, $t$, and $m$ itself; no ground-truth density or \gls{MI} values are required for training at any point.

The per-query noising-indicator mixture is: all-noised with prob.\ $0.35$; an $X\!\mid\!Y$ or $Y\!\mid\!X$ block pattern with prob.\ $0.30$ (split evenly); otherwise a per-coordinate $\mathrm{Bernoulli}(0.5)$ indicator (all-zero draws fall back to all-noised). The mixture covers the three indicator patterns the estimator queries at inference (\Cref{eq:mi-velocity}) and, through the random subsets, general partial observation. The block patterns use the fixed split $s=\lfloor d/2\rfloor$. Training otherwise operates on the whole vector $z$: the $X{:}Y$ partition is used in training only through those block masks and through the corpus's block-structured couplings (decoupling and per-block flows, also at $s$), and the specific partition otherwise appears only at output time.

Each training step draws $128$ distributions from the corpus and one clean context from each. The context length is sampled uniformly per batch between $128$ and the training window (by truncation, or equivalently by masking, \Cref{app:training:batch}), so one set of weights is trained for every context length up to that window; longer contexts are extrapolation (\Cref{app:czyz}). The optimizer is AdamW with $\beta_1{=}0.9$, $\beta_2{=}0.95$, and weight decay $0.01$, with a linear warmup over $100$ steps and gradient norms clipped at $1.0$. Training is executed in \texttt{bf16} mixed precision with compiled kernels on two data-parallel replicas (DDP), with gradient accumulation setting the effective batch size. Training proceeds in two phases. The first runs $500{,}000$ steps with a context window of $1024$ samples and a cosine decay of the learning rate to zero after the warmup. The second starts from the first-phase weights with a fresh optimizer state and runs $60{,}000$ steps with a context window of $2048$ samples, holding the learning rate constant after the warmup. The peak learning rate is $3\cdot10^{-4}$ for \model{} Base and $10^{-3}$ for \model{} Small in both phases; the per-device batch is $16$ distributions with $4$ accumulation steps, except for \model{} Base in the second phase, which uses $8$ with $8$.

\subsection{Hardware for training and inference}
\label{app:hardware}

We train our \model{} variants using 2x H200 GPUs: \model{}-Small requires 2 days and 22 hours (about 8,674 optimizer steps per hour) whereas \model{}-Base requires 5 days and 10 hours (about 4,300 optimizer steps hour) of training. 
As a comparison, our understanding is that InfoAtlas \citep{hu2026infoatlas} requires 2 weeks of training on 16x H800 GPUs.

For inference, we use a single H200 GPU in all our experiments.

\section{Ground-truth benchmark: details and per-task results}
\label{app:czyz}

This Section completes \Cref{sec:exp:czyz}: it specifies the inference settings, reports the aggregate (\Cref{tab:czyz-summary}), joint-width (\Cref{tab:czyz-width}), and per-task (\Cref{tab:czyz-per-task}) results at the three matched budgets.

\paragraph{Inference settings.}
Every task provides precomputed samples and a closed-form ground-truth MI.
We use three sample budgets $N\in\{1000,5000,10000\}$ for all methods. \model{} splits each budget into $64$ query samples and a context of the remaining $N-64$ clean samples; \Cref{alg:mi} averages the velocity-difference integrand over the $64$ query samples and $n_t{=}64$ time draws per query sample.
Every reported number is a mean over eight independent context draws; where a spread is given, it is the sample standard deviation over the draws.
Training samples the context length uniformly between $128$ and the training window, $2048$ samples in the final phase (\Cref{app:training:opt}), so the contexts at the $5000$ and $10000$ budgets are extrapolation beyond the training window; the context attention is permutation invariant and uses no positional encoding, so the model accepts these longer contexts, and \Cref{tab:czyz-width} reports how each model size behaves there.
Per-coordinate monotone transforms (\texttt{normal\_cdf}, \texttt{half\_cube}, \texttt{asinh}) are absorbed by the rank-based copula normalization applied at inference; \Cref{fig:czyz-category-mae} groups them with their base tasks and \Cref{tab:czyz-per-task} lists them separately.
InfoAtlas conditions on the same contexts.
Competitor numbers are five-seed means: the neural estimators are trained on the $N$ samples of each individual task (including per-task hyperparameter tuning), and the classic estimators are fit on them.

\paragraph{Aggregate accuracy.}
\Cref{tab:czyz-summary} reports the MAE of every estimator over the suite at the three budgets.
\model{} Base has the lowest error at each budget among the estimators of \Cref{fig:czyz-category-mae}, and its lead is largest at $1$k samples, where every neural estimator is above $0.24$ nats and the best classic estimator, CCA, is at $0.195$.
\Cref{tab:czyz-per-task} reports every estimate at $10$k samples.

\begin{table}[h]
  \centering\footnotesize
  \caption{Mean absolute error in nats over the $40$ tasks of the Czy\.{z} benchmark at matched budgets of $1$k, $5$k, and $10$k samples per task.
    \model{} and InfoAtlas condition on a context of $N-64$ samples from that budget, and their cells give the mean and sample standard deviation over eight independent context draws; neural estimators are trained per distribution on that number of samples and classic estimators are fit on it, both as five-seed means.
    ``$>10$'' marks a diverged estimator.}
  \label{tab:czyz-summary}
  \begin{tabular}{@{}lccc@{}}
\toprule
Estimator & $1$k & $5$k & $10$k \\
\midrule
\multicolumn{4}{l}{\emph{Foundation models}} \\
\textbf{ALICE Base} & $0.092 \pm 0.004$ & $0.063 \pm 0.001$ & $0.060 \pm 0.001$ \\
\textbf{ALICE Small} & $0.101 \pm 0.005$ & $0.103 \pm 0.002$ & $0.104 \pm 0.001$ \\
InfoAtlas & $0.253 \pm 0.004$ & $0.271 \pm 0.005$ & $0.276 \pm 0.001$ \\
\midrule
\multicolumn{4}{l}{\emph{Neural estimators}} \\
MINDE--\textsc{c} & $0.353$ & $0.070$ & $0.065$ \\
MINE & $0.248$ & $0.117$ & $0.108$ \\
InfoNCE & $0.887$ & $0.123$ & $0.083$ \\
D-V & $1.229$ & $0.279$ & $0.144$ \\
NWJ & $>10$ & $1.261$ & $0.130$ \\
\midrule
\multicolumn{4}{l}{\emph{Classic estimators}} \\
KSG & $0.288$ & $0.240$ & $0.219$ \\
LNN & $4.864$ & $4.966$ & $4.970$ \\
CCA & $0.195$ & $0.192$ & $0.181$ \\
\bottomrule
\end{tabular}

\end{table}

\paragraph{Joint width.}
\Cref{tab:czyz-width} splits the error of both \model{} sizes and InfoAtlas between the $33$ tasks of joint width at most $10$ and the $7$ tasks of width $50$ and $100$.
\model{} Base improves with the context in both groups, and most in the wide one, from $0.153$ to $0.086$ nats.
\model{} Small improves on the narrow tasks, from $0.081$ to $0.074$ nats, and degrades on the wide ones, from $0.197$ to $0.241$, so its aggregate error is flat in $N$.
InfoAtlas is at $0.10$ to $0.12$ nats on the narrow tasks, where its per-width networks apply, and at $1.0$ nats on the wide tasks, where its sliced fallback outputs $0.02$ nats on the five sparse tasks and $0.46$ on the two dense ones against ground truths of $1.02$ to $1.62$.

\begin{table}[h]
  \centering\footnotesize
  \caption{Czy\.{z} benchmark accuracy by joint width against the sample budget $N$ for both \model{} sizes and InfoAtlas.
    ``dims $\le 10$'' aggregates the $33$ tasks of joint width at most $10$ and ``dims $50/100$'' the $7$ wider ones.
    Cells give the mean and sample standard deviation of the per-draw MAE over eight independent context draws.
    The \model{} training window is $2048$ samples, so the rows at $5000$ and $10000$ are context extrapolation.}
  \label{tab:czyz-width}
  \begin{tabular}{@{}lrccc@{}}
\toprule
Model & $N$ & All ($40$) & dims $\le 10$ ($33$) & dims $50/100$ ($7$) \\
\midrule
\textbf{ALICE Base} & $1000$ & $0.092 \pm 0.004$ & $0.079 \pm 0.004$ & $0.153 \pm 0.008$ \\
 & $5000$ & $0.063 \pm 0.001$ & $0.057 \pm 0.002$ & $0.090 \pm 0.004$ \\
 & $10000$ & $0.060 \pm 0.001$ & $0.054 \pm 0.001$ & $0.086 \pm 0.002$ \\
\midrule
\textbf{ALICE Small} & $1000$ & $0.101 \pm 0.005$ & $0.081 \pm 0.005$ & $0.197 \pm 0.017$ \\
 & $5000$ & $0.103 \pm 0.002$ & $0.077 \pm 0.001$ & $0.228 \pm 0.009$ \\
 & $10000$ & $0.104 \pm 0.001$ & $0.074 \pm 0.002$ & $0.241 \pm 0.006$ \\
\midrule
InfoAtlas & $1000$ & $0.253 \pm 0.004$ & $0.100 \pm 0.004$ & $0.978 \pm 0.008$ \\
 & $5000$ & $0.271 \pm 0.005$ & $0.117 \pm 0.006$ & $0.998 \pm 0.005$ \\
 & $10000$ & $0.276 \pm 0.001$ & $0.123 \pm 0.003$ & $1.000 \pm 0.006$ \\
\bottomrule
\end{tabular}

\end{table}

\paragraph{Error analysis.}
At $N{=}10000$, $35$ of the $40$ tasks fall within $0.1$ nats of ground truth for \model{} Base and the mean signed error is $-0.005$ nats, so the aggregate measure is not influenced by a global bias.
Two groups impact the results.
First, the spiral embeddings are under-estimated by $0.43$ and $0.44$ nats at joint widths $6$ and $10$ and by $0.18$ nats at width $50$, the largest errors \model{} experiences in the suite.
Second, the dense multinormal tasks are over-estimated, by $0.23$ nats at joint width $100$ and $0.15$ at width $50$, growing with width.
\model{} Small shares both failure modes with larger magnitudes: it under-estimates the width-$10$ spiral by $0.47$ nats and over-estimates the width-$100$ dense multinormal by $0.44$, and it also over-estimates three of the five sparse width-$50$ tasks by $0.24$ nats each.

\begin{landscape}

  \begin{table}
    \tiny
    \renewcommand{\tabcolsep}{0.0pt}

    \caption{Per-task MI estimates on the \citet{czyz2023beyond} suite against ground truth (GT), in nats, with every estimator at a budget of $10$k samples per task. Cell shading encodes the signed bias of the estimate: blue for over-estimation, red for under-estimation, with saturation growing linearly up to a bias of $0.6$ nats. Competitor rows report five-seed means. The \model{} rows use zero-shot in-context estimation with a frozen model, and the \model{} and InfoAtlas rows are means over eight independent context draws. Abbreviations: \textit{Mn} multinormal, \textit{St} Student-t, \textit{Nm} normal, \textit{Hc} half-cube, \textit{Sp} spiral.
    }
    \label{tab:czyz-per-task}
    \begin{tabular}{lrrrrrrrrrrrrrrrrrrrrrrrrrrrrrrrrrrrrrrrr}
\toprule
GT & 0.22 & 0.43 & 0.29 & 0.45 & 0.41 & 0.41 & 0.41 & 1.02 & 1.02 & 1.02 & 1.02 & 0.29 & 1.02 & 1.29 & 1.02 & 0.41 & 1.02 & 0.59 & 1.62 & 0.41 & 1.02 & 1.02 & 1.02 & 1.02 & 1.02 & 1.02 & 1.02 & 1.02 & 1.02 & 0.22 & 0.43 & 0.19 & 0.29 & 0.18 & 0.45 & 0.30 & 0.41 & 1.71 & 0.33 & 0.41 \\
\midrule
\multicolumn{41}{l}{\emph{Foundation models}} \\
\textbf{ALICE Base} & {\cellcolor[HTML]{E8EBEF}} 0.26 & {\cellcolor[HTML]{F1EEEF}} 0.42 & {\cellcolor[HTML]{F0F1F1}} 0.30 & {\cellcolor[HTML]{F2F1F1}} 0.44 & {\cellcolor[HTML]{EFF0F1}} 0.42 & {\cellcolor[HTML]{EFF0F1}} 0.42 & {\cellcolor[HTML]{EFF0F1}} 0.42 & {\cellcolor[HTML]{F0F0F1}} 1.03 & {\cellcolor[HTML]{E4E8EE}} 1.07 & {\cellcolor[HTML]{DDE3EB}} 1.09 & {\cellcolor[HTML]{E6EAEE}} 1.06 & {\cellcolor[HTML]{F2F2F2}} 0.29 & {\cellcolor[HTML]{F0F1F1}} 1.03 & {\cellcolor[HTML]{C8D5E5}} 1.44 & {\cellcolor[HTML]{E3E8ED}} 1.07 & {\cellcolor[HTML]{ECEEF0}} 0.43 & {\cellcolor[HTML]{DDE3EB}} 1.10 & {\cellcolor[HTML]{E5E9EE}} 0.64 & {\cellcolor[HTML]{B1C4DD}} 1.85 & {\cellcolor[HTML]{F0F0F1}} 0.42 & {\cellcolor[HTML]{EFF0F1}} 1.03 & {\cellcolor[HTML]{E2E7ED}} 1.08 & {\cellcolor[HTML]{DDE4EB}} 1.09 & {\cellcolor[HTML]{E9BBC8}} 0.84 & {\cellcolor[HTML]{DD7190}} 0.59 & {\cellcolor[HTML]{DD6C8C}} 0.58 & {\cellcolor[HTML]{F1ECEE}} 1.00 & {\cellcolor[HTML]{EFE1E5}} 0.96 & {\cellcolor[HTML]{F1EFEF}} 1.01 & {\cellcolor[HTML]{E8EBEF}} 0.26 & {\cellcolor[HTML]{F1EEEF}} 0.42 & {\cellcolor[HTML]{EEEFF1}} 0.21 & {\cellcolor[HTML]{EFF0F1}} 0.30 & {\cellcolor[HTML]{F0F1F1}} 0.19 & {\cellcolor[HTML]{F2F2F2}} 0.45 & {\cellcolor[HTML]{F2F1F1}} 0.29 & {\cellcolor[HTML]{DEE4EC}} 0.48 & {\cellcolor[HTML]{EEDADF}} 1.63 & {\cellcolor[HTML]{F1EAEC}} 0.31 & {\cellcolor[HTML]{F0F1F1}} 0.42 \\
\textbf{ALICE Small} & {\cellcolor[HTML]{F2F2F2}} 0.23 & {\cellcolor[HTML]{EAEDF0}} 0.46 & {\cellcolor[HTML]{F1F1F2}} 0.29 & {\cellcolor[HTML]{EEDBE1}} 0.37 & {\cellcolor[HTML]{DDE3EB}} 0.49 & {\cellcolor[HTML]{DDE3EB}} 0.49 & {\cellcolor[HTML]{DDE4EB}} 0.48 & {\cellcolor[HTML]{ABC1DC}} 1.27 & {\cellcolor[HTML]{DBE2EB}} 1.10 & {\cellcolor[HTML]{E8EBEF}} 1.06 & {\cellcolor[HTML]{CDD8E6}} 1.15 & {\cellcolor[HTML]{E6EAEE}} 0.34 & {\cellcolor[HTML]{ADC2DC}} 1.26 & {\cellcolor[HTML]{92AFD4}} 1.62 & {\cellcolor[HTML]{DCE3EB}} 1.10 & {\cellcolor[HTML]{DDE4EB}} 0.48 & {\cellcolor[HTML]{E9EBEF}} 1.05 & {\cellcolor[HTML]{CFDAE7}} 0.71 & {\cellcolor[HTML]{7399CA}} 2.06 & {\cellcolor[HTML]{DEE4EC}} 0.48 & {\cellcolor[HTML]{ACC1DC}} 1.26 & {\cellcolor[HTML]{DAE1EA}} 1.11 & {\cellcolor[HTML]{E8EBEF}} 1.06 & {\cellcolor[HTML]{ECCCD5}} 0.89 & {\cellcolor[HTML]{E28EA6}} 0.69 & {\cellcolor[HTML]{DB6587}} 0.55 & {\cellcolor[HTML]{DEE4EC}} 1.09 & {\cellcolor[HTML]{EEEFF1}} 1.03 & {\cellcolor[HTML]{F1EEEF}} 1.01 & {\cellcolor[HTML]{F1F1F2}} 0.23 & {\cellcolor[HTML]{ECEEF0}} 0.45 & {\cellcolor[HTML]{F0F0F1}} 0.20 & {\cellcolor[HTML]{F1F1F2}} 0.30 & {\cellcolor[HTML]{F2F0F1}} 0.17 & {\cellcolor[HTML]{EFDDE2}} 0.38 & {\cellcolor[HTML]{EFDDE2}} 0.23 & {\cellcolor[HTML]{CCD7E6}} 0.55 & {\cellcolor[HTML]{C9D6E5}} 1.85 & {\cellcolor[HTML]{EEEFF1}} 0.35 & {\cellcolor[HTML]{DEE4EC}} 0.48 \\
InfoAtlas & {\cellcolor[HTML]{F0E3E7}} 0.18 & {\cellcolor[HTML]{EDD5DC}} 0.33 & {\cellcolor[HTML]{EFDDE2}} 0.22 & {\cellcolor[HTML]{EDD1D9}} 0.34 & {\cellcolor[HTML]{F1EDEE}} 0.40 & {\cellcolor[HTML]{F1ECEE}} 0.39 & {\cellcolor[HTML]{F1EDEE}} 0.40 & {\cellcolor[HTML]{D53D69}} 0.02 & {\cellcolor[HTML]{F0E4E8}} 0.98 & {\cellcolor[HTML]{EFE2E6}} 0.97 & {\cellcolor[HTML]{EFE0E5}} 0.96 & {\cellcolor[HTML]{F1EDEE}} 0.28 & {\cellcolor[HTML]{D53D69}} 0.02 & {\cellcolor[HTML]{D53D69}} 0.46 & {\cellcolor[HTML]{F0E5E8}} 0.98 & {\cellcolor[HTML]{F1EAEC}} 0.39 & {\cellcolor[HTML]{EFE0E5}} 0.96 & {\cellcolor[HTML]{F0E7EA}} 0.56 & {\cellcolor[HTML]{D53D69}} 0.46 & {\cellcolor[HTML]{F1ECED}} 0.39 & {\cellcolor[HTML]{D53D69}} 0.02 & {\cellcolor[HTML]{F0E5E8}} 0.98 & {\cellcolor[HTML]{EFE0E4}} 0.96 & {\cellcolor[HTML]{D53D69}} 0.02 & {\cellcolor[HTML]{D85178}} 0.49 & {\cellcolor[HTML]{DE7795}} 0.61 & {\cellcolor[HTML]{D53D69}} 0.02 & {\cellcolor[HTML]{EBC5D0}} 0.87 & {\cellcolor[HTML]{EBC6D1}} 0.88 & {\cellcolor[HTML]{E9BCC9}} 0.04 & {\cellcolor[HTML]{E185A0}} 0.07 & {\cellcolor[HTML]{F0E6E9}} 0.15 & {\cellcolor[HTML]{EFDEE3}} 0.22 & {\cellcolor[HTML]{F0E5E8}} 0.14 & {\cellcolor[HTML]{EDD1D9}} 0.34 & {\cellcolor[HTML]{EFDDE2}} 0.23 & {\cellcolor[HTML]{E393AA}} 0.10 & {\cellcolor[HTML]{D53D69}} 1.10 & {\cellcolor[HTML]{EAC2CD}} 0.17 & {\cellcolor[HTML]{F1EDEE}} 0.40 \\
\midrule
\multicolumn{41}{l}{\emph{Neural estimators}} \\
MINDE--\textsc{c} & {\cellcolor[HTML]{F0E8EA}} 0.19 & {\cellcolor[HTML]{F0E5E8}} 0.39 & {\cellcolor[HTML]{F1E9EB}} 0.26 & {\cellcolor[HTML]{F0E6E9}} 0.41 & {\cellcolor[HTML]{F1EEEF}} 0.40 & {\cellcolor[HTML]{F1EEEF}} 0.40 & {\cellcolor[HTML]{F1EEEF}} 0.40 & {\cellcolor[HTML]{F0E2E6}} 0.97 & {\cellcolor[HTML]{F0E5E8}} 0.98 & {\cellcolor[HTML]{EFDFE4}} 0.96 & {\cellcolor[HTML]{F1EEEF}} 1.01 & {\cellcolor[HTML]{F2F1F1}} 0.29 & {\cellcolor[HTML]{EDEEF0}} 1.04 & {\cellcolor[HTML]{EED6DD}} 1.20 & {\cellcolor[HTML]{F0E8EB}} 0.99 & {\cellcolor[HTML]{EDEFF0}} 0.43 & {\cellcolor[HTML]{F0E8EB}} 0.99 & {\cellcolor[HTML]{EDEFF0}} 0.61 & {\cellcolor[HTML]{ECD0D8}} 1.51 & {\cellcolor[HTML]{F1EBED}} 0.39 & {\cellcolor[HTML]{F1EEEF}} 1.01 & {\cellcolor[HTML]{EED6DD}} 0.93 & {\cellcolor[HTML]{EDD3DB}} 0.92 & {\cellcolor[HTML]{ECCAD4}} 0.89 & {\cellcolor[HTML]{E8B2C2}} 0.81 & {\cellcolor[HTML]{E7ACBD}} 0.79 & {\cellcolor[HTML]{F1EEEF}} 1.01 & {\cellcolor[HTML]{EDD0D9}} 0.91 & {\cellcolor[HTML]{EDD0D9}} 0.91 & {\cellcolor[HTML]{ECCDD6}} 0.10 & {\cellcolor[HTML]{E5A3B6}} 0.17 & {\cellcolor[HTML]{F1EBED}} 0.17 & {\cellcolor[HTML]{EFE0E4}} 0.23 & {\cellcolor[HTML]{F2EFF0}} 0.17 & {\cellcolor[HTML]{EFE0E5}} 0.39 & {\cellcolor[HTML]{F1ECEE}} 0.28 & {\cellcolor[HTML]{EAC1CD}} 0.25 & {\cellcolor[HTML]{EED7DE}} 1.62 & {\cellcolor[HTML]{F0E5E8}} 0.29 & {\cellcolor[HTML]{F1EEEF}} 0.40 \\
MINE & {\cellcolor[HTML]{F2F1F1}} 0.22 & {\cellcolor[HTML]{EFE2E6}} 0.38 & {\cellcolor[HTML]{EED7DD}} 0.20 & {\cellcolor[HTML]{EBC5D0}} 0.30 & {\cellcolor[HTML]{F0F1F1}} 0.42 & {\cellcolor[HTML]{F0F1F1}} 0.42 & {\cellcolor[HTML]{F0F1F1}} 0.42 & {\cellcolor[HTML]{EBC4CF}} 0.87 & {\cellcolor[HTML]{F0E8EB}} 0.99 & {\cellcolor[HTML]{EFDFE4}} 0.96 & {\cellcolor[HTML]{F1EEEF}} 1.01 & {\cellcolor[HTML]{F0F1F1}} 0.30 & {\cellcolor[HTML]{EDD0D9}} 0.91 & {\cellcolor[HTML]{F0E5E8}} 1.25 & {\cellcolor[HTML]{F2F2F2}} 1.02 & {\cellcolor[HTML]{F2F1F1}} 0.41 & {\cellcolor[HTML]{F0E5E8}} 0.98 & {\cellcolor[HTML]{F1EEEF}} 0.58 & {\cellcolor[HTML]{F1EBEC}} 1.60 & {\cellcolor[HTML]{F2F1F1}} 0.41 & {\cellcolor[HTML]{E8B5C4}} 0.82 & {\cellcolor[HTML]{EDD0D9}} 0.91 & {\cellcolor[HTML]{EBC7D2}} 0.88 & {\cellcolor[HTML]{E49DB2}} 0.74 & {\cellcolor[HTML]{E1859F}} 0.66 & {\cellcolor[HTML]{DF7996}} 0.62 & {\cellcolor[HTML]{E8B2C2}} 0.81 & {\cellcolor[HTML]{EBC7D2}} 0.88 & {\cellcolor[HTML]{EAC1CD}} 0.86 & {\cellcolor[HTML]{E8B4C3}} 0.02 & {\cellcolor[HTML]{DE7694}} 0.02 & {\cellcolor[HTML]{EDD6DD}} 0.10 & {\cellcolor[HTML]{E9BBC9}} 0.11 & {\cellcolor[HTML]{EFE0E5}} 0.12 & {\cellcolor[HTML]{E498AE}} 0.15 & {\cellcolor[HTML]{ECCBD5}} 0.17 & {\cellcolor[HTML]{F1EEEF}} 0.40 & {\cellcolor[HTML]{F0E3E7}} 1.66 & {\cellcolor[HTML]{F1EBED}} 0.31 & {\cellcolor[HTML]{F0F1F1}} 0.42 \\
InfoNCE & {\cellcolor[HTML]{F2F1F1}} 0.22 & {\cellcolor[HTML]{F0E8EB}} 0.40 & {\cellcolor[HTML]{EFDDE2}} 0.22 & {\cellcolor[HTML]{EDD4DC}} 0.35 & {\cellcolor[HTML]{F2F1F1}} 0.41 & {\cellcolor[HTML]{F0F1F1}} 0.42 & {\cellcolor[HTML]{F0F1F1}} 0.42 & {\cellcolor[HTML]{E8B5C4}} 0.82 & {\cellcolor[HTML]{F1EBED}} 1.00 & {\cellcolor[HTML]{F0E2E6}} 0.97 & {\cellcolor[HTML]{F2F2F2}} 1.02 & {\cellcolor[HTML]{F2F1F1}} 0.29 & {\cellcolor[HTML]{EBC7D2}} 0.88 & {\cellcolor[HTML]{F0E5E8}} 1.25 & {\cellcolor[HTML]{F1EEEF}} 1.01 & {\cellcolor[HTML]{F2F1F1}} 0.41 & {\cellcolor[HTML]{F0E8EB}} 0.99 & {\cellcolor[HTML]{F1EEEF}} 0.58 & {\cellcolor[HTML]{EFDFE3}} 1.56 & {\cellcolor[HTML]{F2F1F1}} 0.41 & {\cellcolor[HTML]{E6A6B9}} 0.77 & {\cellcolor[HTML]{EFDFE4}} 0.96 & {\cellcolor[HTML]{EDD0D9}} 0.91 & {\cellcolor[HTML]{E394AB}} 0.71 & {\cellcolor[HTML]{E5A0B4}} 0.75 & {\cellcolor[HTML]{E28EA6}} 0.69 & {\cellcolor[HTML]{E5A3B6}} 0.76 & {\cellcolor[HTML]{EED6DD}} 0.93 & {\cellcolor[HTML]{EBC7D2}} 0.88 & {\cellcolor[HTML]{ECCAD3}} 0.09 & {\cellcolor[HTML]{E9B8C6}} 0.24 & {\cellcolor[HTML]{F0E8EA}} 0.16 & {\cellcolor[HTML]{EFE0E4}} 0.23 & {\cellcolor[HTML]{F0E6E9}} 0.14 & {\cellcolor[HTML]{EEDAE0}} 0.37 & {\cellcolor[HTML]{EED7DE}} 0.21 & {\cellcolor[HTML]{F1EEEF}} 0.40 & {\cellcolor[HTML]{F1ECEE}} 1.69 & {\cellcolor[HTML]{F1EEEF}} 0.32 & {\cellcolor[HTML]{F0F1F1}} 0.42 \\
D-V & {\cellcolor[HTML]{F2F1F1}} 0.22 & {\cellcolor[HTML]{F0E5E8}} 0.39 & {\cellcolor[HTML]{EFE0E4}} 0.23 & {\cellcolor[HTML]{EDD4DC}} 0.35 & {\cellcolor[HTML]{F2F1F1}} 0.41 & {\cellcolor[HTML]{F0F1F1}} 0.42 & {\cellcolor[HTML]{F2F1F1}} 0.41 & {\cellcolor[HTML]{E9B8C6}} 0.83 & {\cellcolor[HTML]{F1EBED}} 1.00 & {\cellcolor[HTML]{F0E2E6}} 0.97 & {\cellcolor[HTML]{F2F2F2}} 1.02 & {\cellcolor[HTML]{F2F1F1}} 0.29 & {\cellcolor[HTML]{ECCAD4}} 0.89 & {\cellcolor[HTML]{F0E8EB}} 1.26 & {\cellcolor[HTML]{F1EEEF}} 1.01 & {\cellcolor[HTML]{F2F1F1}} 0.41 & {\cellcolor[HTML]{F0E5E8}} 0.98 & {\cellcolor[HTML]{F1EEEF}} 0.58 & {\cellcolor[HTML]{F0E5E8}} 1.58 & {\cellcolor[HTML]{F2F1F1}} 0.41 & {\cellcolor[HTML]{E7ACBD}} 0.79 & {\cellcolor[HTML]{EFDFE4}} 0.96 & {\cellcolor[HTML]{EDD0D9}} 0.91 & {\cellcolor[HTML]{E397AD}} 0.72 & {\cellcolor[HTML]{E49DB2}} 0.74 & {\cellcolor[HTML]{E188A2}} 0.67 & {\cellcolor[HTML]{E6A9BB}} 0.78 & {\cellcolor[HTML]{EDD3DB}} 0.92 & {\cellcolor[HTML]{EBC7D2}} 0.88 & {\cellcolor[HTML]{E7AEBF}} 0.00 & {\cellcolor[HTML]{DE7392}} 0.01 & {\cellcolor[HTML]{D53F6B}} -0.40 & {\cellcolor[HTML]{D53D69}} -0.91 & {\cellcolor[HTML]{ECCBD5}} 0.05 & {\cellcolor[HTML]{DF7D99}} 0.06 & {\cellcolor[HTML]{EBC5D0}} 0.15 & {\cellcolor[HTML]{F1EEEF}} 0.40 & {\cellcolor[HTML]{F1ECEE}} 1.69 & {\cellcolor[HTML]{F1EEEF}} 0.32 & {\cellcolor[HTML]{F2F1F1}} 0.41 \\
NWJ & {\cellcolor[HTML]{F2F1F1}} 0.22 & {\cellcolor[HTML]{F0E8EB}} 0.40 & {\cellcolor[HTML]{EEDAE0}} 0.21 & {\cellcolor[HTML]{EDD1D9}} 0.34 & {\cellcolor[HTML]{F2F1F1}} 0.41 & {\cellcolor[HTML]{F0F1F1}} 0.42 & {\cellcolor[HTML]{F2F1F1}} 0.41 & {\cellcolor[HTML]{EABECB}} 0.85 & {\cellcolor[HTML]{F1EBED}} 1.00 & {\cellcolor[HTML]{F0E2E6}} 0.97 & {\cellcolor[HTML]{F2F2F2}} 1.02 & {\cellcolor[HTML]{F2F1F1}} 0.29 & {\cellcolor[HTML]{EDD0D9}} 0.91 & {\cellcolor[HTML]{F0E8EB}} 1.26 & {\cellcolor[HTML]{F1EEEF}} 1.01 & {\cellcolor[HTML]{F2F1F1}} 0.41 & {\cellcolor[HTML]{F0E5E8}} 0.98 & {\cellcolor[HTML]{F1EEEF}} 0.58 & {\cellcolor[HTML]{EEDCE1}} 1.55 & {\cellcolor[HTML]{F2F1F1}} 0.41 & {\cellcolor[HTML]{E7AFBF}} 0.80 & {\cellcolor[HTML]{EFDFE4}} 0.96 & {\cellcolor[HTML]{EDD0D9}} 0.91 & {\cellcolor[HTML]{E397AD}} 0.72 & {\cellcolor[HTML]{E49DB2}} 0.74 & {\cellcolor[HTML]{E188A2}} 0.67 & {\cellcolor[HTML]{E7ACBD}} 0.79 & {\cellcolor[HTML]{EED6DD}} 0.93 & {\cellcolor[HTML]{EBC7D2}} 0.88 & {\cellcolor[HTML]{E8B1C1}} 0.01 & {\cellcolor[HTML]{DD708F}} 0.00 & {\cellcolor[HTML]{EAC1CD}} 0.03 & {\cellcolor[HTML]{D53D69}} -0.78 & {\cellcolor[HTML]{ECCED7}} 0.06 & {\cellcolor[HTML]{DF7D99}} 0.06 & {\cellcolor[HTML]{E9BCC9}} 0.12 & {\cellcolor[HTML]{F1EEEF}} 0.40 & {\cellcolor[HTML]{F1E9EB}} 1.68 & {\cellcolor[HTML]{F1EEEF}} 0.32 & {\cellcolor[HTML]{F2F1F1}} 0.41 \\
\midrule
\multicolumn{41}{l}{\emph{Classic estimators}} \\
KSG & {\cellcolor[HTML]{F0F1F1}} 0.23 & {\cellcolor[HTML]{EFE2E6}} 0.38 & {\cellcolor[HTML]{EDD1D9}} 0.18 & {\cellcolor[HTML]{E8B3C2}} 0.24 & {\cellcolor[HTML]{F2F1F1}} 0.41 & {\cellcolor[HTML]{F2F1F1}} 0.41 & {\cellcolor[HTML]{F2F1F1}} 0.41 & {\cellcolor[HTML]{D53D69}} 0.17 & {\cellcolor[HTML]{ECCAD4}} 0.89 & {\cellcolor[HTML]{E0829D}} 0.65 & {\cellcolor[HTML]{EDEEF0}} 1.04 & {\cellcolor[HTML]{F0F1F1}} 0.30 & {\cellcolor[HTML]{D53D69}} 0.19 & {\cellcolor[HTML]{E7AFBF}} 1.07 & {\cellcolor[HTML]{EFDFE4}} 0.96 & {\cellcolor[HTML]{F0F1F1}} 0.42 & {\cellcolor[HTML]{E49DB2}} 0.74 & {\cellcolor[HTML]{F1EBED}} 0.57 & {\cellcolor[HTML]{E28DA6}} 1.29 & {\cellcolor[HTML]{F2F1F1}} 0.41 & {\cellcolor[HTML]{D53D69}} 0.20 & {\cellcolor[HTML]{EDD3DB}} 0.92 & {\cellcolor[HTML]{E394AB}} 0.71 & {\cellcolor[HTML]{D53D69}} 0.17 & {\cellcolor[HTML]{E397AD}} 0.72 & {\cellcolor[HTML]{DB6486}} 0.55 & {\cellcolor[HTML]{D53D69}} 0.19 & {\cellcolor[HTML]{ECCDD6}} 0.90 & {\cellcolor[HTML]{E28EA6}} 0.69 & {\cellcolor[HTML]{EFE2E6}} 0.17 & {\cellcolor[HTML]{E8B2C2}} 0.22 & {\cellcolor[HTML]{EDD3DA}} 0.09 & {\cellcolor[HTML]{EABECB}} 0.12 & {\cellcolor[HTML]{EDD1D9}} 0.07 & {\cellcolor[HTML]{E6A7B9}} 0.20 & {\cellcolor[HTML]{EBC5D0}} 0.15 & {\cellcolor[HTML]{F0F1F1}} 0.42 & {\cellcolor[HTML]{F1E9EB}} 1.68 & {\cellcolor[HTML]{F1EEEF}} 0.32 & {\cellcolor[HTML]{F2F1F1}} 0.41 \\
LNN & {\cellcolor[HTML]{EBEDF0}} 0.25 & {\cellcolor[HTML]{6D96C8}} 0.89 & {\cellcolor[HTML]{4479BB}} 2.72 & {\cellcolor[HTML]{4479BB}} 6.66 & {\cellcolor[HTML]{F2F1F1}} 0.41 & {\cellcolor[HTML]{F0F1F1}} 0.42 & {\cellcolor[HTML]{F0F1F1}} 0.42 & {\cellcolor[HTML]{4479BB}} 17.34 & {\cellcolor[HTML]{4479BB}} 2.68 & {\cellcolor[HTML]{4479BB}} 6.45 & {\cellcolor[HTML]{AAC0DB}} 1.27 & {\cellcolor[HTML]{8BAAD1}} 0.65 & {\cellcolor[HTML]{4479BB}} 17.34 & {\cellcolor[HTML]{4479BB}} 17.34 & {\cellcolor[HTML]{4479BB}} 3.10 & {\cellcolor[HTML]{4479BB}} 2.48 & {\cellcolor[HTML]{4479BB}} 7.31 & {\cellcolor[HTML]{4479BB}} 6.77 & {\cellcolor[HTML]{4479BB}} 33.48 & {\cellcolor[HTML]{F1EBED}} 0.39 & {\cellcolor[HTML]{4479BB}} 17.34 & {\cellcolor[HTML]{4479BB}} 2.49 & {\cellcolor[HTML]{4479BB}} 7.27 & {\cellcolor[HTML]{4479BB}} 17.34 & {\cellcolor[HTML]{4479BB}} 3.10 & {\cellcolor[HTML]{4479BB}} 7.31 & {\cellcolor[HTML]{4479BB}} 17.34 & {\cellcolor[HTML]{4479BB}} 2.38 & {\cellcolor[HTML]{4479BB}} 7.24 & {\cellcolor[HTML]{D85178}} -0.31 & {\cellcolor[HTML]{D53D69}} -0.70 & {\cellcolor[HTML]{9CB6D7}} 0.49 & {\cellcolor[HTML]{4479BB}} 1.01 & {\cellcolor[HTML]{4479BB}} 2.11 & {\cellcolor[HTML]{4479BB}} 2.47 & {\cellcolor[HTML]{4479BB}} 5.49 & {\cellcolor[HTML]{EEDCE1}} 0.34 & {\cellcolor[HTML]{E7ADBE}} 1.48 & {\cellcolor[HTML]{F0E5E8}} 0.29 & {\cellcolor[HTML]{F0F1F1}} 0.42 \\
CCA & {\cellcolor[HTML]{E7AEBF}} 0.00 & {\cellcolor[HTML]{DD708F}} 0.00 & {\cellcolor[HTML]{E49AB0}} 0.00 & {\cellcolor[HTML]{DC6B8C}} 0.00 & {\cellcolor[HTML]{EEDCE1}} 0.34 & {\cellcolor[HTML]{F2F1F1}} 0.41 & {\cellcolor[HTML]{F0E8EA}} 0.38 & {\cellcolor[HTML]{F0E5E8}} 0.98 & {\cellcolor[HTML]{F0E2E6}} 0.97 & {\cellcolor[HTML]{EFDCE2}} 0.95 & {\cellcolor[HTML]{F0F0F1}} 1.03 & {\cellcolor[HTML]{F0F1F1}} 0.30 & {\cellcolor[HTML]{EDEEF0}} 1.04 & {\cellcolor[HTML]{EAECEF}} 1.32 & {\cellcolor[HTML]{F0F0F1}} 1.03 & {\cellcolor[HTML]{F0F1F1}} 0.42 & {\cellcolor[HTML]{F2F2F2}} 1.02 & {\cellcolor[HTML]{F2F1F1}} 0.59 & {\cellcolor[HTML]{CED9E6}} 1.75 & {\cellcolor[HTML]{F1EBED}} 0.39 & {\cellcolor[HTML]{F0E8EB}} 0.99 & {\cellcolor[HTML]{EFDFE4}} 0.96 & {\cellcolor[HTML]{EFDFE4}} 0.96 & {\cellcolor[HTML]{EABECB}} 0.85 & {\cellcolor[HTML]{D53D69}} 0.23 & {\cellcolor[HTML]{D53D69}} 0.39 & {\cellcolor[HTML]{F0E5E8}} 0.98 & {\cellcolor[HTML]{E9BBC9}} 0.84 & {\cellcolor[HTML]{ECCAD4}} 0.89 & {\cellcolor[HTML]{B9CBE0}} 0.42 & {\cellcolor[HTML]{4479BB}} 2.15 & {\cellcolor[HTML]{EEDCE1}} 0.12 & {\cellcolor[HTML]{D5DEE9}} 0.39 & {\cellcolor[HTML]{EABFCB}} 0.01 & {\cellcolor[HTML]{E9ECEF}} 0.48 & {\cellcolor[HTML]{E59EB2}} 0.02 & {\cellcolor[HTML]{DF7B98}} 0.02 & {\cellcolor[HTML]{EEDAE0}} 1.63 & {\cellcolor[HTML]{EBC4CF}} 0.18 & {\cellcolor[HTML]{F0E8EA}} 0.38 \\
\midrule
dist & \begin{sideways} Asinh @ \textit{St} 1 $\times$ 1 (dof=1) \end{sideways} & \begin{sideways} Asinh @ \textit{St} 2 $\times$ 2 (dof=1) \end{sideways} & \begin{sideways} Asinh @ \textit{St} 3 $\times$ 3 (dof=2) \end{sideways} & \begin{sideways} Asinh @ \textit{St} 5 $\times$ 5 (dof=2) \end{sideways} & \begin{sideways} Bimodal 1 $\times$ 1 \end{sideways} & \begin{sideways} Bivariate \textit{Nm} 1 $\times$ 1 \end{sideways} & \begin{sideways} \textit{Hc} @ Bivariate \textit{Nm} 1 $\times$ 1 \end{sideways} & \begin{sideways} \textit{Hc} @ \textit{Mn} 25 $\times$ 25 (2-pair) \end{sideways} & \begin{sideways} \textit{Hc} @ \textit{Mn} 3 $\times$ 3 (2-pair) \end{sideways} & \begin{sideways} \textit{Hc} @ \textit{Mn} 5 $\times$ 5 (2-pair) \end{sideways} & \begin{sideways} \textit{Mn} 2 $\times$ 2 (2-pair) \end{sideways} & \begin{sideways} \textit{Mn} 2 $\times$ 2 (dense) \end{sideways} & \begin{sideways} \textit{Mn} 25 $\times$ 25 (2-pair) \end{sideways} & \begin{sideways} \textit{Mn} 25 $\times$ 25 (dense) \end{sideways} & \begin{sideways} \textit{Mn} 3 $\times$ 3 (2-pair) \end{sideways} & \begin{sideways} \textit{Mn} 3 $\times$ 3 (dense) \end{sideways} & \begin{sideways} \textit{Mn} 5 $\times$ 5 (2-pair) \end{sideways} & \begin{sideways} \textit{Mn} 5 $\times$ 5 (dense) \end{sideways} & \begin{sideways} \textit{Mn} 50 $\times$ 50 (dense) \end{sideways} & \begin{sideways} \textit{Nm} CDF @ Bivariate \textit{Nm} 1 $\times$ 1 \end{sideways} & \begin{sideways} \textit{Nm} CDF @ \textit{Mn} 25 $\times$ 25 (2-pair) \end{sideways} & \begin{sideways} \textit{Nm} CDF @ \textit{Mn} 3 $\times$ 3 (2-pair) \end{sideways} & \begin{sideways} \textit{Nm} CDF @ \textit{Mn} 5 $\times$ 5 (2-pair) \end{sideways} & \begin{sideways} \textit{Sp} @ \textit{Mn} 25 $\times$ 25 (2-pair) \end{sideways} & \begin{sideways} \textit{Sp} @ \textit{Mn} 3 $\times$ 3 (2-pair) \end{sideways} & \begin{sideways} \textit{Sp} @ \textit{Mn} 5 $\times$ 5 (2-pair) \end{sideways} & \begin{sideways} \textit{Sp} @ \textit{Nm} CDF @ \textit{Mn} 25 $\times$ 25 (2-pair) \end{sideways} & \begin{sideways} \textit{Sp} @ \textit{Nm} CDF @ \textit{Mn} 3 $\times$ 3 (2-pair) \end{sideways} & \begin{sideways} \textit{Sp} @ \textit{Nm} CDF @ \textit{Mn} 5 $\times$ 5 (2-pair) \end{sideways} & \begin{sideways} \textit{St} 1 $\times$ 1 (dof=1) \end{sideways} & \begin{sideways} \textit{St} 2 $\times$ 2 (dof=1) \end{sideways} & \begin{sideways} \textit{St} 2 $\times$ 2 (dof=2) \end{sideways} & \begin{sideways} \textit{St} 3 $\times$ 3 (dof=2) \end{sideways} & \begin{sideways} \textit{St} 3 $\times$ 3 (dof=3) \end{sideways} & \begin{sideways} \textit{St} 5 $\times$ 5 (dof=2) \end{sideways} & \begin{sideways} \textit{St} 5 $\times$ 5 (dof=3) \end{sideways} & \begin{sideways} Swiss roll 2 $\times$ 1 \end{sideways} & \begin{sideways} Uniform 1 $\times$ 1 (additive noise=.1) \end{sideways} & \begin{sideways} Uniform 1 $\times$ 1 (additive noise=.75) \end{sideways} & \begin{sideways} Wiggly @ Bivariate \textit{Nm} 1 $\times$ 1 \end{sideways} \\
\bottomrule
\end{tabular}

  \end{table}

\end{landscape}

\section{Single-cell signaling responses: technical details}
\label{app:slemi-details}

This section complements \Cref{sec:exp:sc-signalling}: it provides additional details and results. 
The reference study for this section is \cite{jetka2019information}, 
for which there are no ground truth \gls{MI} estimates: as such, throughout this section, we compare against the biological conclusions of that study, and note that \gls{MI} estimates are essentially equivalent to our results.

\paragraph{Application domain.}
Cells sense extracellular cues through signaling pathways that convert ligand concentrations into effector activity and gene regulation. A canonical example is the NF-$\mathcal{K}$B pathway, which responds to the inflammatory cytokine TNF-$\alpha$ and regulates immune responses; although the underlying biochemistry is well characterized, how reliably individual cells infer stimulus strength from their response trajectories remains unclear~\citep{purvis2013encoding,lee2014fold,antebi2017operational}. Over the past two decades, cellular signaling has increasingly been formulated in terms of information theory~\citep{nurse2008life,waltermann2011information,brennan2012information,jetka2018information,petkova2019optimal}: an extracellular stimulus ($X$) is transmitted through a stochastic biochemical network to produce a cellular response ($Y$), so mutual information $\MI(X;Y)$ quantifies how much observing the response reduces uncertainty about the stimulus, while channel capacity measures the maximum information transmissible over input distributions. This perspective has enabled measurements of signaling fidelity in pathways such as TNF-$\alpha$--NF-$\mathcal{K}$B and has shown that time-resolved response trajectories can transmit more information than static measurements \citep{tostevin2009mutual,cheong2011information,selimkhanov2014accurate}.

\paragraph{Estimand and metrics.}
The input $X$ is an experimentally controlled stimulus taking one of $m$ values with input distribution $p(X)$, and the output $Y \in \mathbb{R}^{d}$ is a vector of single-cell measurements distributed according to the unknown conditionals $P(Y \mid X = x_i)$. Mutual information decomposes as $\MI(X;Y) = \sum_i p_i D_i$, where $D_i = \KL{P(Y \mid x_i)}{\bar{P}}$ is the divergence of each dose-conditional from the output mixture $\bar{P} = \sum_j p_j P(Y \mid x_j)$. We estimate each $D_i$ with the conditional variant of the estimand (\Cref{app:mi-variants:cond}).
Capacity $C = \max_p \MI(X;Y)$ is computed by the Blahut--Arimoto algorithm \citep{blahut1972computation,arimoto1972algorithm} run directly on the estimated per-dose divergences: since the conditional fields do not depend on $p$, they are cached once, and only the mixture field is re-estimated as the ascent updates $p_i \propto p_i e^{D_i}$. At each iteration the shuffled context of \Cref{app:mi-variants:cond} is redrawn with the current $p$: doses drawn from $p$ select the pool from which each response row is taken, and the dose column is overwritten by an independent draw from $p$, so the response marginal of the context is $\sum_i p_i P(Y \mid x_i)$ and every $D_i$ is measured against the mixture of the current iterate. The sampling noise of the redraw is held fixed across iterations by reseeding from one base seed, so the ascent is a deterministic function of $p$. 
For the pairwise probability of correct discrimination (PCD), which is the Bayes accuracy of deciding between doses $i$ and $j$ from a single cell under equal priors, we exploit the fact that the two-dose mutual information at $p = (\tfrac{1}{2}, \tfrac{1}{2})$ equals the Jensen--Shannon divergence $J_{ij}$, which brackets the Bayes accuracy as $\tfrac{1}{2}(1 + J_{ij}) \le \mathrm{PCD}_{ij} \le \tfrac{1}{2}\big(1 + \min(1, \sqrt{2 \ln 2 \cdot J_{ij}})\big)$ (with $J$ in bits). 

\paragraph{Results in full.}
\Cref{fig:nfkb-pcd} shows the six-panel version of \Cref{fig:nfkb}, with the pairwise discrimination matrices.
The precise numbers presented in \Cref{sec:exp:sc-signalling} are the following. The single-frame capacity (panel B) peaks at $1.12$ to $1.16$ bits at minutes $15$ to $21$, against about $1$ bit in \citet{jetka2019information}; it falls to $0.04$ bits at minute $66$, and the second rise reaches $0.40$ bits at minute $93$. The prefix capacity is $0.87$ bits with the first three frames, $1.29$ with the first five, $1.34$ with the first nine, and between $1.07$ and $1.34$ afterwards, so the prefix of frames exceeds the best single frame. The trajectory capacity is $C\approx1.05$ bits ($0.93$ to $1.20$ across seeds), against $1.3$ bits in the reference study. The PCD averages $0.74$ over the $55$ dose pairs for the single frame at minute $21$ and $0.84$ for the trajectory; over the $15$ pairs of doses at or above $0.5$ ng/ml, where the amplitude of the first peak saturates, the averages are $0.56$ and $0.67$, so the gain from dynamics is concentrated at high doses. 

\begin{figure}[t]
  \centering
  \includegraphics[width=\textwidth]{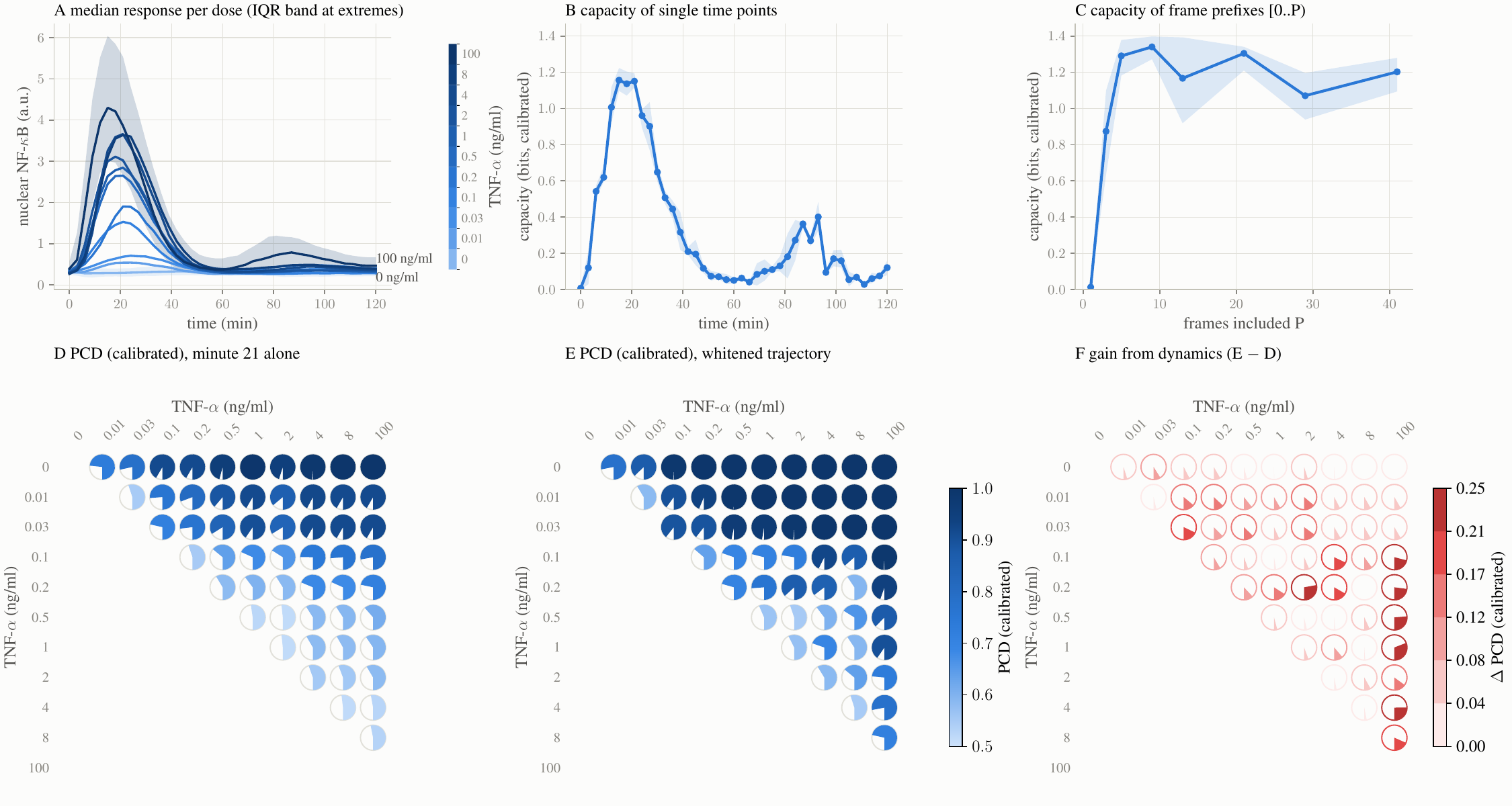}
  \caption{Six-panel analysis of the NF-$\mathcal{K}$B dose channel. \textbf{(A--C)}~As in \Cref{fig:nfkb}. \textbf{(D--F)}~Pairwise probability of correct discrimination between doses: from the single frame at minute $21$ (D), from the trajectory (E), and the gain from dynamics (F, E minus D). The filled fraction of each circle and its color both encode the value, from chance ($0.5$) to certain discrimination ($1$) in D and E, and from $0$ to $0.25$ in F.}
  \label{fig:nfkb-pcd}
\end{figure}

\section{Promoter identification: technical details}
\label{app:promoter-details}

This section expands on the application domain, the dataset, the encoding, and the diagnostics behind the results of \Cref{sec:exp:promoter}, and reports the numbers that the main text summarizes.

\paragraph{Application domain.}
Genomics relies on computational methods to find patterns in large datasets from basic and clinical research \citep{libbrecht2015machine, whalen2022navigating, teschendorff2025epigenetic}. DNA is a sequence of four bases, adenine (A), thymine (T), guanine (G), and cytosine (C), and the order of the bases determines the biological instructions that a strand of DNA carries. We follow the recent practice of treating DNA sequences as text \citep{dotan2024effect, qiao2024model, malusare2024understanding, eapen2025genomic}, with the simplest tokenization: each base is one token, so a sequence is a high-dimensional vector whose coordinates take four values. A central question in molecular biology is the regulation of gene expression: expression requires a stretch of regulatory DNA called a promoter, which contains motifs, that is, patterns whose presence shows a statistically significant dependence with expression levels. Computational methods based on \gls{MI} \citep{elemento2007universal, rao2007motif} search whole genomes for the key elements of transcription regulation by quantifying the dependence between the presence of a motif in a regulatory region and the expression of the corresponding gene; further motif properties, such as position bias, orientation preference, and functional interactions, can be studied with \gls{MI} as well \citep{elemento2007universal}. A minimal eukaryotic promoter contains a transcription start site (TSS) and a \textsc{tata-box} motif about $30$ base pairs upstream of the TSS; in \textit{Arabidopsis thaliana} the preferred position is between $-39$ and $-26$ relative to the TSS \citep{bernard2010tc}.

\paragraph{Dataset.}
\citet{umarov2017recognition} evaluate convolutional promoter-recognition models on sequences extracted from the \textsc{epd} database \citep{dreos2013epd}. We use their \textit{Arabidopsis thaliana} \textsc{tata}-promoter and non-promoter collection, $1{,}497$ and $2{,}879$ sequences respectively, each of $251$ bases; promoter sequences span positions $-200$ to $+50$ around the annotated TSS. We discard sequences containing ambiguous bases and subsample the non-promoter class to $1{,}497$ sequences, so the promoter label $X$ is uniform and the \gls{MI} of every window is bounded by $H(X)=\ln 2$ nats.

\paragraph{Relation to previous work on motif search.}
\citet{umarov2017recognition} localize functional elements by substituting a sliding region of the input with random bases and tracking the drop in classification accuracy. \citet{foresti2026infosedd} uses a recent \gls{MI} estimator based on discrete diffusion to reproduce the same protocol of \citet{umarov2017recognition}.
The \gls{MI} profile we obtain in \Cref{sec:exp:promoter} recasts this search using \model{}: windows on segments irrelevant to promoter status yield values near zero, and windows overlapping the \textsc{tata-box} motif yield high values. 

\paragraph{Encoding.}
The velocity fields of \Cref{eq:mi-velocity} are defined on $\bbR^d$, so discrete symbols are transformed by a fixed injective embedding: each symbol maps to a vector of $K$ real coordinates, where every coordinate holds an independent random permutation of equally spaced standard-normal quantiles, perturbed by a small uniform dither confined within each level. Injectivity preserves $\MI(X;Y)$ exactly for every $K$; the dither removes the ties that would otherwise collapse the Gaussian-copula normalization of the estimator, and each encoded coordinate is approximately standard normal, the scale on which \model{} is trained. A window of $L$ bases concatenates its per-base vectors, giving blocks of dimension $K$ for the label and $LK$ for the window; the unequal, length-dependent widths are handled natively by the variable-dimension capabilities of \model{}. The reported results use $K=1$, the minimal injective width.

\paragraph{Protocol and numbers.}
For each of the $246$ ($L=6$) or $248$ ($L=4$) window positions, \model{} conditions on a context of $1{,}024$ encoded label--window pairs, and the velocity differences are averaged over $1{,}024$ held-out pairs and $32$ time points. Both window lengths place the maximum at $30$ to $32$ bases upstream of the TSS, inside the documented \textsc{tata-box} band: the peak is $0.31$ nats at offset $-30$ for $L=4$ over $248$ windows, and $0.34$ nats at offset $-32$ for $L=6$ over $246$ windows. 

\paragraph{Discussion.}
Both window lengths place the top windows at TSS offsets $-32$ to $-30$, and both resolve the two core promoter elements the sequences carry: the \textsc{tata-box} at $-30$, whose top 6-mers (\texttt{TATATA}, \texttt{TATAAA}) each occur in about $4\%$ of promoters against about $0.4\%$ for the most frequent non-promoter 6-mer, and the initiator element straddling the TSS, which is a pyrimidine/purine pair (position $-1$ is C or T in $94\%$ of promoters, position $+1$ is A or G in $93\%$) and has no recurring k-mer. Since the promoter set is the \textsc{tata}-containing subset of \textsc{epd}, the $-30$ peak acts as a positive control for localization.

\section{Brain Region Activity Patterns: Details}
\label{app:soi-mice-details}

This Section provides additional details about the \gls{O-information} estimator used in \Cref{sec:exp:soi-mice}, the structure of the Visual Behavior Neuropixels data, the selection of flashes that produced the analyzed tables, the experimental protocol, and the full per-window results. 

\paragraph{Estimator.}
For $N$ blocks $X=(X_1,\dots,X_N)$ the total correlation and the dual total correlation are $\mathrm{TC}=\mathrm{KL}\big(p(x)\,\|\,\prod_i p(x_i)\big)$ and $\mathrm{DTC}=H(X)-\sum_i H(X_i\mid X_{\setminus i})$, where $X_{\setminus i}$ denotes all blocks except $X_i$, and the \acrshort{O-information} is $\Omega=\mathrm{TC}-\mathrm{DTC}$ \citep{Rosas2019QuantifyingHI}. \citet{bounoua2024} write both terms as time integrals of squared score differences, evaluated at the same noised coordinates: for $\mathrm{TC}$, between the joint score and the concatenation of the $N$ marginal scores; for $\mathrm{DTC}$, between the joint score and the concatenation of the $N$ scores of each block conditioned on the clean values of the other blocks. Under the interpolant $x_t=(1-t)x_0+t\,\varepsilon$ of \Cref{sec:method}, the score of a block and its velocity are related by $s=((1-t)v-x_t)/t$, so two fields that share the noised coordinate differ by $\Delta s=\tfrac{1-t}{t}\Delta v$. With the weight of \Cref{eq:estimator} this gives
\begin{align}
  \mathrm{TC}  &= \int_0^1 \frac{1-t}{t}\;\Ex\sum_{i=1}^N \Norm{\vmodel(x_t,t;\ctx)\big|_i-\vmodel\big(x_{i,t},t;\ctx_i\big)}^2\,\d t, \label{eq:tc-velocity}\\
  \mathrm{DTC} &= \int_0^1 \frac{1-t}{t}\;\Ex\sum_{i=1}^N \Norm{\vmodel(x_t,t;\ctx)\big|_i-\vmodel\big([x_{i,t},x_{0,\setminus i}],t,\indone_i;\ctx\big)\big|_i}^2\,\d t, \label{eq:dtc-velocity}
\end{align}
where $\ctx_i$ is the context restricted to the columns of block $i$, $\indone_i$ is the noising indicator that noises block $i$ and holds the other blocks at their clean values, and $|_i$ selects the coordinates of block $i$. The marginal field in \Cref{eq:tc-velocity} requires no dedicated mechanism: \model{} accepts any joint width, so the field conditioned on $\ctx_i$ is the field of the marginal law of $X_i$. The conditional field in \Cref{eq:dtc-velocity} is the masked field of \Cref{eq:estimator}, and for $N=2$ \Cref{eq:dtc-velocity} is the \gls{MI} estimator of \Cref{sec:method}. \Cref{eq:dtc-velocity} rests on the identity $\Ex\big[s_{i\mid\setminus i}(x_{i,t};x_{0,\setminus i})\,\big|\,x_t\big]=s(x_t)|_i$. Given the clean values of the other blocks, the noised block $i$ and the noised other blocks are independent, so the conditional score of block $i$ equals the block-$i$ score of $p_t(x_t\mid x_{0,\setminus i})$; averaging that score over $p(x_{0,\setminus i}\mid x_t)$ gives the joint score. All fields are evaluated under common random numbers: the same $(x_0,t,\varepsilon)$ draw is used for every term. Per Monte-Carlo row, the estimator evaluates $2N+1$ velocities: one joint, $N$ conditional, and $N$ marginal. Both integrals are invariant under any per-block bijection, so the copula normalization of \Cref{sec:method}, which is applied per coordinate, leaves them unchanged, and the sub-context fields see the normalized columns of the joint context.

\paragraph{Task and trial structure.}
One image is shown to a mouse for $250$ ms followed by $500$ ms of gray screen, so a new flash starts every $750$ ms. The image repeats over several flashes and then changes; the mouse earns water by licking after a change. A trial is one run of repeats together with the change flash that ends it, and the next trial repeats the image that the change introduced. The position of a flash is its rank inside its trial, counting from one. About $5\%$ of flashes are omitted by design (the screen stays gray). The active block of a session holds about $4{,}800$ flashes. Neuropixels probes record single units in the areas \textsc{VISp}, \textsc{VISl}, \textsc{VISal}, \textsc{VISrl}, \textsc{VISam}, and \textsc{VISpm}. We use the $72$ sessions of \citet{bounoua2024}: mice with a familiar-image and a novel-image session, recorded on consecutive days, with more than $20$ well-isolated units (signal-to-noise ratio above $1$ and fewer than one inter-spike-interval violation) in each of the six areas. The pairing of the two sessions of a mouse and their order (familiar first) were verified against the session table of the Allen Institute.

\paragraph{Selection of flashes.}
The preprocessing of \citet{bounoua2024} is designed to keep, for both flash types, only trials in which the mouse was rewarded, and to drop non-change flashes during which the mouse licked. We reproduced their tables exactly from the raw spike times (identical row counts and values in every session we compared) and found that neither filter has an effect in the released code: the reward filter tests a field that is always missing and is therefore always satisfied, and the lick exclusion is negated twice and is also always satisfied. The analyzed tables are therefore defined as follows. A change flash is any flash of the active block at which the image changed. A non-change flash is any non-omitted flash of the active block at positions $4$ to $10$ of its trial that still shows the image the trial started with. A session holds $145$ to $351$ change flashes and $1{,}154$ to $1{,}715$ non-change flashes. On the first session, for example, the $253$ change flashes comprise $203$ hits and $47$ misses, and the $1{,}476$ non-change flashes comprise $518$ flashes of hit trials, $154$ of miss trials, $546$ of aborted trials, and $247$ of catch trials, $111$ of them with a lick. We keep this selection so that our estimates and those of \citet{bounoua2024} describe the same rows; the outcome of every trial is stored with every flash, so the hit-only and lick-free selections need no new data. Positions $1$ to $3$ of a trial are excluded because the response to a repeated image decreases over the first repeats and levels off from the fourth.

\paragraph{Windows, step size, and dimension.}
For every flash and unit, spikes are counted in $250$ bins of $1$ ms after flash onset, averaged over the units of an area, cut into five windows of $50$ ms, and summed inside each window in steps of $s$ ms. The dimension of the variable of one area in one window is therefore $50/s$: one number at $s=50$ ms (used in \Cref{sec:exp:soi-mice}), $25$ numbers at $s=2$ ms (the main figure of \citet{bounoua2024}), and $10$ or $50$ at $s=5$ or $1$ ms (their appendix). The joint width is the number of areas times this dimension. The file distributed with \citet{bounoua2024} contains the tables at $s=50$ ms; the finer resolutions were rebuilt from the raw spike times. The rebuilt $50$ ms tables reproduce the distributed ones exactly: estimating the $3{,}375$ per-session values from the rebuilt tables returns the same numbers to machine precision.

\paragraph{Correlation between flashes.}
Rows of a session are flashes in temporal order, and consecutive flashes are correlated. \Cref{tab:soi-mice-autocorr} reports the lag autocorrelation of the six-area vector within a session at $100$ to $150$ ms. Change flashes occur once per trial and are close to independent; non-change flashes occur about six times per trial, $750$ ms apart, and are strongly correlated at short lags. The effective number of independent draws per session, from the truncated autocorrelation sum, is $185$ for change flashes (of $251$ nominal rows, median over sessions) and $146$ to $205$ for non-change flashes (of $1{,}447$). The two flash types therefore carry a similar amount of information despite a six-fold difference in row count. Two consequences follow for the protocol. Sample sizes are quoted as effective draws. The context and the evaluation rows of a session are split by contiguous runs, because a random split places repeats of one trial on both sides, and the evaluation points then have near copies in the context.

\begin{table}[t]
  \centering\small
  \caption{Within-session autocorrelation of the six-area vector at lag $k$ (in flashes), mean over areas and sessions, $100$ to $150$ ms window.}
  \label{tab:soi-mice-autocorr}
  \begin{tabular}{@{}lrrrrrrrr@{}}
    \toprule
    $k$ & 1 & 2 & 3 & 5 & 10 & 20 & 50 & 100 \\
    \midrule
    change     & $-0.02$ & $+0.07$ & $+0.07$ & $+0.07$ & $+0.06$ & $+0.05$ & $+0.01$ & $-0.03$ \\
    non-change & $+0.47$ & $+0.38$ & $+0.33$ & $+0.25$ & $+0.15$ & $+0.10$ & $+0.08$ & $+0.07$ \\
    \bottomrule
  \end{tabular}
\end{table}

\paragraph{Protocol details.}
Each session, window, and flash type is estimated from the rows of that session with \model{}-Base: we use a context of $128$ rows and an evaluation set of the remaining rows, capped at $512$, assigned by contiguous runs of $32$ rows, $64$ time draws per evaluation row, and five context draws. 
The context size is the largest one that keeps most change sessions: $192$ rows are required for $128$ context rows and $64$ evaluation rows, and sessions with fewer change flashes are excluded from the change condition, which leaves $31$ familiar and $32$ novel sessions for change flashes and all $36$ of each kind for non-change flashes. The same context size is applied to both flash types because the estimate depends on the context size. In a first run with the context of each flash type set by its own row count ($128$ to $150$ rows for change flashes and $1{,}024$ for non-change flashes), the two flash types differed already in the first window, before the visual response ($+0.12$ nats, $p=2\cdot10^{-8}$ over $63$ sessions); with matched contexts the first-window difference is $+0.01$ nats ($p=0.9$) and the peak difference is unchanged. Two further controls quantify the choices above. Assigning rows to the context at random, without the run structure, changes the estimates by $0.01$ nats at this dimension. Drawing the $128$ context rows from the other $71$ sessions raises the median estimate at the peak from $0.49$ to $0.75$ nats for non-change flashes and from $0.70$ to $0.83$ for change flashes, and removes most of the difference between familiar- and novel-image sessions (medians of $0.80$ and $0.87$ nats for change flashes, against $0.41$ and $1.08$ with own-session contexts); a context pooled across animals describes a mixture whose components share the animal-specific level of activity, and that shared component is attributed to redundancy. Between-session differences account for $27$ to $56\%$ of the variance of every column of the pooled table.

\paragraph{Per-window results.}
\Cref{fig:soi-mice-appendix} completes \Cref{fig:soi-mice} with the familiar-image sessions, the within-session difference between the two flash types for both image sets, and the difference between the two sessions of a mouse for non-change flashes. \Cref{tab:soi-mice-windows} lists the median over sessions of the per-session estimates behind both figures, and \Cref{tab:soi-mice-paired} the paired contrasts with signed-rank tests. In novel-image sessions the \acrshort{O-information} is positive in every window for every mouse, and the maximum over windows falls at $100$ to $150$ ms for $78\%$ of the mice for change flashes and for $47\%$ for non-change flashes. In familiar-image sessions the change-flash profile is flat, and the non-change profile rises late, with its maximum at $150$ to $200$ ms. The familiar session of a mouse precedes the novel one by one day, and the two sessions already differ before the visual response arrives, by $-0.16$ nats in the first window. \Cref{tab:soi-mice-variance} decomposes the variance of the per-session estimates by sequential sums of squares of the crossed factors; the seed share is the variability of the estimator across context draws. Across all $63$ sessions with both flash types, animal identity accounts for $38\%$ of the variance of the peak-window estimates, experience level for $23\%$, flash type for $8\%$, and the variability of the estimator across context draws for $10\%$.

\begin{figure}[t]
  \centering
  \includegraphics[width=\textwidth]{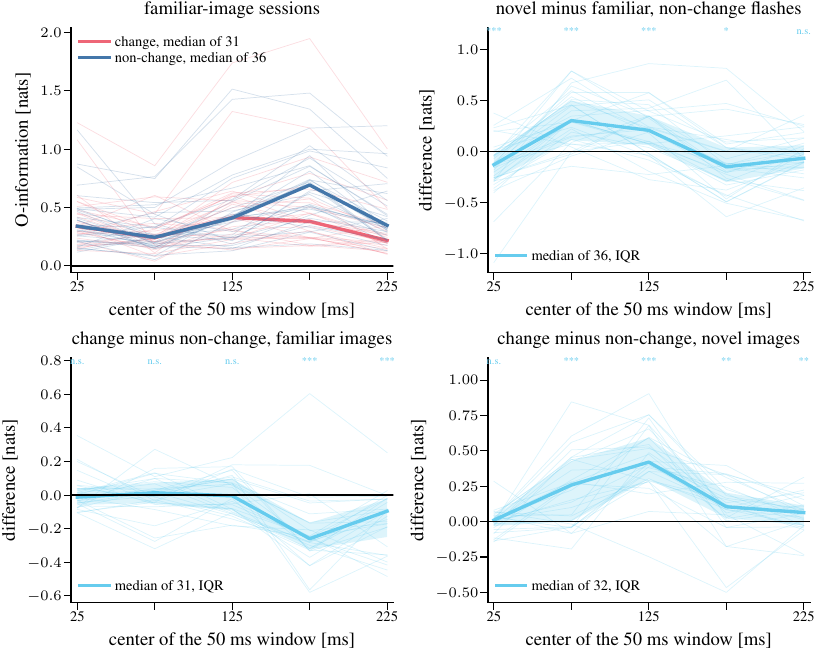}
  \caption{Complement of \Cref{fig:soi-mice}: \acrshort{O-information} of the six visual areas in the five $50$ ms windows after a flash, one estimate per session. Top left: familiar-image sessions, one thin line per mouse and flash type, group medians in bold. Top right: difference between the novel and the familiar session of each mouse for non-change flashes. Bottom: difference between change and non-change flashes within each session, for familiar-image (left) and novel-image (right) sessions. Difference panels show the median, the interquartile band, and a signed-rank test per window ($*$: $p<0.05$, $**$: $p<0.01$, $***$: $p<0.001$). The plotted values are stored with the figure.}
  \label{fig:soi-mice-appendix}
\end{figure}

\begin{table}[t]
  \centering\small
  \caption{Median over sessions of the per-session \acrshort{O-information} (nats) of the six areas, context of $128$ rows drawn from the same session, by window after flash onset. Interquartile range in brackets.}
  \label{tab:soi-mice-windows}
  \setlength{\tabcolsep}{4pt}
  \resizebox{\textwidth}{!}{%
  \begin{tabular}{@{}llccccc@{}}
    \toprule
    sessions & flash type & $0$--$50$ & $50$--$100$ & $100$--$150$ & $150$--$200$ & $200$--$250$ ms \\
    \midrule
    familiar ($n=31$) & change & $0.34$ [$0.25$, $0.47$] & $0.24$ [$0.18$, $0.32$] & $0.41$ [$0.29$, $0.59$] & $0.38$ [$0.24$, $0.56$] & $0.22$ [$0.19$, $0.33$] \\
    familiar ($n=36$) & non-change & $0.34$ [$0.22$, $0.48$] & $0.25$ [$0.18$, $0.37$] & $0.41$ [$0.25$, $0.51$] & $0.69$ [$0.50$, $0.85$] & $0.35$ [$0.28$, $0.53$] \\
    novel ($n=32$) & change & $0.16$ [$0.12$, $0.24$] & $0.83$ [$0.69$, $0.96$] & $1.08$ [$0.88$, $1.27$] & $0.66$ [$0.51$, $0.79$] & $0.37$ [$0.24$, $0.47$] \\
    novel ($n=36$) & non-change & $0.17$ [$0.11$, $0.26$] & $0.55$ [$0.45$, $0.73$] & $0.60$ [$0.48$, $0.75$] & $0.49$ [$0.39$, $0.73$] & $0.28$ [$0.22$, $0.41$] \\
    \bottomrule
  \end{tabular}}
\end{table}

\begin{table}[t]
  \centering\small
  \caption{Paired contrasts of the per-session estimates (nats): median difference, two-sided Wilcoxon signed-rank $p$-value, and fraction of positive differences, by window.}
  \label{tab:soi-mice-paired}
  \setlength{\tabcolsep}{3pt}
  \resizebox{\textwidth}{!}{%
  \begin{tabular}{@{}llccccc@{}}
    \toprule
    contrast & group & $0$--$50$ & $50$--$100$ & $100$--$150$ & $150$--$200$ & $200$--$250$ ms \\
    \midrule
    change $-$ non-change, within session & familiar ($n=31$) & $-0.01$ ($0.8$, $48\%$) & $+0.01$ ($0.8$, $52\%$) & $-0.00$ ($0.8$, $48\%$) & $-0.26$ ($8\cdot10^{-6}$, $10\%$) & $-0.10$ ($2\cdot10^{-6}$, $6\%$) \\
    change $-$ non-change, within session & novel ($n=32$) & $+0.01$ ($0.9$, $56\%$) & $+0.26$ ($3\cdot10^{-5}$, $75\%$) & $+0.42$ ($3\cdot10^{-9}$, $97\%$) & $+0.10$ ($3\cdot10^{-3}$, $81\%$) & $+0.06$ ($5\cdot10^{-3}$, $72\%$) \\
    novel $-$ familiar, within mouse & change ($n=27$) & $-0.16$ ($8\cdot10^{-4}$, $15\%$) & $+0.51$ ($3\cdot10^{-8}$, $96\%$) & $+0.56$ ($7\cdot10^{-8}$, $93\%$) & $+0.22$ ($7\cdot10^{-4}$, $81\%$) & $+0.11$ ($1\cdot10^{-3}$, $81\%$) \\
    novel $-$ familiar, within mouse & non-change ($n=36$) & $-0.13$ ($8\cdot10^{-4}$, $22\%$) & $+0.30$ ($2\cdot10^{-8}$, $86\%$) & $+0.21$ ($2\cdot10^{-5}$, $78\%$) & $-0.15$ ($0.01$, $31\%$) & $-0.07$ ($0.09$, $33\%$) \\
    \bottomrule
  \end{tabular}}
\end{table}

\begin{table}[t]
  \centering\small
  \caption{Variance shares of the per-session estimates (all $63$ sessions with both flash types, five context draws each) by window: sequential sums of squares of mouse identity, experience level, and flash type; the seed share is the variability of the estimator across context draws; the remainder holds interactions.}
  \label{tab:soi-mice-variance}
  \begin{tabular}{@{}lrrrrr@{}}
    \toprule
    window (ms) & mouse & experience & flash type & seed & remainder \\
    \midrule
    $0$--$50$    & $41\%$ & $10\%$ & $0\%$ & $20\%$ & $29\%$ \\
    $50$--$100$  & $21\%$ & $40\%$ & $3\%$ & $13\%$ & $24\%$ \\
    $100$--$150$ & $38\%$ & $23\%$ & $8\%$ & $10\%$ & $22\%$ \\
    $150$--$200$ & $54\%$ & $0\%$  & $1\%$ & $14\%$ & $31\%$ \\
    $200$--$250$ & $55\%$ & $0\%$  & $1\%$ & $14\%$ & $29\%$ \\
    \bottomrule
  \end{tabular}
\end{table}

\end{document}